\documentclass{article}

\usepackage{arxiv}

\usepackage[utf8]{inputenc} 
\usepackage[T1]{fontenc}    
\usepackage{hyperref}       
\usepackage{url}            
\usepackage{booktabs}       
\usepackage{amsfonts}       
\usepackage{nicefrac}       
\usepackage{microtype}      
\usepackage{lipsum}

\usepackage{epsfig}
\usepackage{graphicx}
\usepackage{subfigure}
\usepackage{epstopdf}
\usepackage{amsmath} 
\usepackage{amssymb}  
\usepackage{amsthm}
\usepackage{multirow}
\usepackage{algpseudocode}
\usepackage{booktabs}
\usepackage{amssymb}
\usepackage{amsthm}
\usepackage{tikz}
\usepackage{wrapfig}
\usepackage{algorithm}
\usepackage{multirow}
\usepackage{graphicx}
\usepackage{caption}
\usepackage{epstopdf}
\usepackage{color}
\usepackage{cleveref}
\usepackage{latexsym}

\usepackage{enumerate}
\usepackage[shortlabels]{enumitem}
\usepackage{mathrsfs}

\newtheorem{theorem}{Theorem}[section]

\newtheorem{lemma}[theorem]{Lemma}

\allowdisplaybreaks
\usepackage{xspace}
\usepackage[draft]{fixme}
\fxsetup{layout=inline,theme=color}
\definecolor{fxnote}{rgb}{0.8000,0.0000,0.0000}
\colorlet{fxnotebg}{yellow}

\DeclareRobustCommand\onedot{\futurelet\@let@token\@onedot}
\def\@onedot{\ifx\@let@token.\else.\null\fi\xspace}
\def\eg{\emph{e.g}\onedot} 
\def\ie{\emph{i.e}\onedot}

\def\etal{\emph{et al}\onedot}

\title{Self-Weighted Robust LDA for Multiclass Classification with Edge Classes \thanks{Corresponding authors: Xiaoqin Zhang and Zhihui Li.}}

\author{
  Caixia Yan \\
  School of Electronic and Information Engineering \\
  Xi'an Jiaotong University \\
  \texttt{yancaixia@stu.xjtu.edu.cn} \\
   \And
  Xiaojun Chang \\
  Faculty of Information Technology \\
  Monash University \\
  \texttt{xiaojun.chang@monash.edu} \\
  \AND
  Minnan Luo \\
  School of Electronic and Information Engineering \\
  Xi'an Jiaotong University \\
  \texttt{minnluo@mail.xjtu.edu.cn} \\
  \And
  Qinghua Zheng \\
  School of Electronic and Information Engineering \\
  Xi'an Jiaotong University \\
  \texttt{qhzheng@mail.xjtu.edu.cn} \\
  \And
  Xiaoqin Zhang* \\
  College of Computer Science and Artificial Intelligence \\
  Wenzhou University \\
  \texttt{zhangxiaoqinnan@gmail.com}
  \And
  Zhihui Li* \\
  Qilu University of Technology \\
  Shandong Academy of Sciences \\
  \texttt{zhihuilics@gmail.com} \\
  \And
  Feiping Nie \\
  Center for Optical Image Analysis and Learning \\
  Northwestern Polytechnical University \\
  \texttt{feipingnie@gmail.com}
}

\begin{document}
\maketitle

\begin{abstract}
Linear discriminant analysis (LDA) is a popular technique to learn the most discriminative features for multi-class classification. A vast majority of existing LDA algorithms are prone to be dominated by the class with very large deviation from the others, \emph{i.e.}, edge class, which occurs frequently in multi-class classification. First, the existence of edge classes often makes the total mean biased in the calculation of between-class scatter matrix. Second, the exploitation of $\ell_2$-norm based between-class distance criterion magnifies the extremely large distance corresponding to edge class. In this regard, a novel self-weighted robust LDA with $\ell_{2,1}$-norm based pairwise between-class distance criterion, called SWRLDA, is proposed for multi-class classification especially with edge classes. SWRLDA can automatically avoid the optimal mean calculation and simultaneously learn adaptive weights for each class pair without setting any additional parameter. An efficient re-weighted algorithm is exploited to derive the global optimum of the challenging $\ell_{2,1}$-norm maximization problem. The proposed SWRLDA is easy to implement, and converges fast in practice. Extensive experiments demonstrate that SWRLDA performs favorably against other compared methods on both synthetic and real-world datasets, while presenting superior computational efficiency in comparison with other techniques.
\end{abstract}

\keywords{Robust linear discriminant analysis \and dimension reduction \and multi-class classification \and edge class}

\section{Introduction}
As one of the most fundamental problems in data mining, multi-class classification has attracted a surge of research interests \cite{yan2018top,zhao2018multiclass}.
Various models for multi-class classification have been proposed in literature, such as decision trees \cite{holmes2002multiclass}, $k$-Nearest Neighbor \cite{zhang2005extending}, Naive Bayes \cite{farid2014hybrid}, and Support Vector Machines \cite{chang2011libsvm}. However, the efficiency and effectiveness of these methods usually drop exponentially as the feature dimensionality increases due to the ``curse of dimensionality" problem \cite{zhao2017unsupervised,xie2017joint}. 
The original data in real-world applications usually possess very large dimension and contain redundant or noisy features which are useless and even harmful to the separation of multiple categories \cite{wen2018robust}.
Linear Discriminant Analysis (LDA) is thus developed to address this problem by selecting and extracting the most discriminative features for multi-class classification in a supervised way \cite{Ye2007Least,xiong2017daehr,yuan2017multilinear,ChangY17,WangCLSC16}. The basic idea of LDA is to learn an optimal projection matrix which minimizes the variability within each class and simultaneously maximizes the discrepancy between different classes in the embedding space. In such a way, the original high-dimensional data can be transformed into a low-dimensional subspace with high class separability.

Despite their efficacy, most existing LDA algorithms can be easily dominated by the classes with very large deviation from the others (\ie, edge classes), which are more likely to appear as the class number increases \cite{wang2018generalized,zhang2012confused,HanYZCL18,ZhuHCSS17,ChangNMYZ15}.
First, the between-class scatter matrix completely relies on the calculation of total data mean. The average of the whole dataset is generally regarded as the total mean, while edge classes can easily make this mean calculation with zero breakdown point biased \cite{yahaya2016robust}. Second, the exploitation of $\ell_2$-norm based between-class distance criterion is known to be dominated by class pairs with extreme values induced by edge classes. 
As a result, the discrimination between edge and non-edge class can easily dominate the whole classification, while the remaining multiple non-edge classes undergo a large overlap, termed as class separation problem \cite{zhang2012confused,ChengCLHGZ17,zhang2020top}. 
 As the illustration in \Cref{introduction} indicates, the left four classes are located in close proximity to each other, while the ``Black footed Albatross'' class is quite far from them, which can be regarded as an edge class. An ideal projection should distinguish the five classes exactly by making a good trade-off between edge class and non-edge class as shown in \Cref{introduction}.
 However, a worse projection direction is usually learned by existing LDA methods due to the domination of edge class, such that the four non-edge classes overlap with each other seriously, leading to an overall low and suboptimal classification performance. 

Several earlier LDA variants have been proposed to address this issue by taking the estimation of optimal mean into consideration, including M-estimator \cite{martin2006robust}, S-estimators \cite{he2000high}, Minimum Covariance Determinant (MCD) \cite{hubert2004fast} and Minimum Volume Ellipsoid (MVE) \cite{chork1992integrating}. However, these methods usually suffer from intractable computational complexity and still cannot guarantee a better solution to the class separation problem. Other researchers have approached the problem by maximizing the pairwise distance between every two class means. As a result, the original $c$-class fisher criterion is decomposed into $\frac{1}{2}c(c-1)$ pairwise between-class distances, corresponding to the worst case between-class separation criterion \cite{zhang2010worst}. 
However, the class pairs with extreme large distance can still dominant the whole classification due to the exploitation of $\ell_{2}$-norm distance criterion. To focus more on the closer ones, many weighting methods have been proposed to learn different weights for each class pair. Loog \etal \cite{loog2001multiclass} proposed an approximate pairwise accuracy criterion such that the importance of each class pair depends on the approximation of Bayes error rate. This method can be easily solved by eigenvalue decomposition as traditional LDA; however, the learned approximate pairwise weights may not be the optimal ones because it is calculated in the original high-dimensional space. Tao \etal \cite{tao2009geometric} assumed that all the classes are sampled from homoscedastic Gaussians with identical covariance matrices and developed three new criteria to maximize the geometric mean of Kullback-Leibler (KL) divergences between different pairs of classes. Bian \etal \cite{bian2008harmonic} proposed to maximize the harmonic mean of all pairs of symmetric KL divergences under the homoscedastic Gaussian assumption. However, gradient method is adopted to solve the proposed challenging problems in \cite{tao2009geometric,bian2008harmonic,LuoCNYHZ18,LuoNCYHZ18,ChengGCSHZ18}, which converge very slowly in many cases. Bian \etal \cite{bian2011max} also presented a max-min distance analysis to guarantee the separation of all class pairs by maximizing the minimum between-class distance. However, this method needs to solve a challenging non-smooth min-max problem with orthonormal constraints, which can only obtain an approximate solution by using a sequential convex relaxation algorithm.

\begin{figure}[t] 	
	\centering 		
	\includegraphics[width=0.6\linewidth]{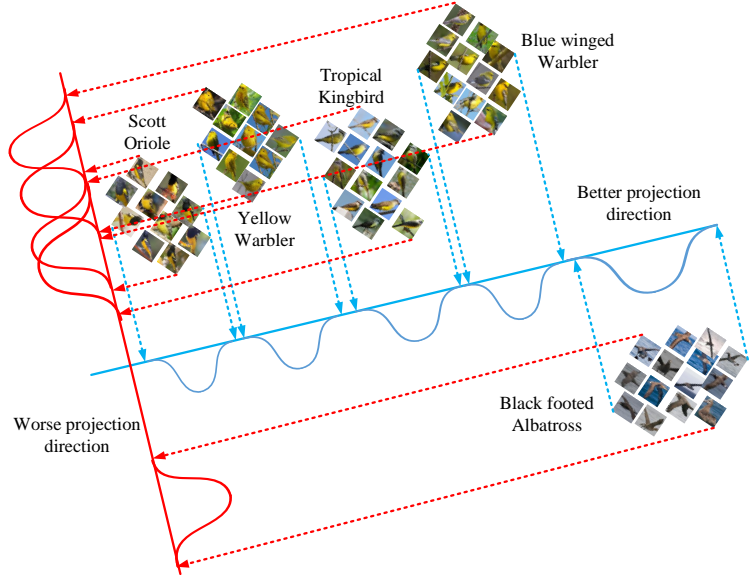}
	\caption{Bird images from five categories. The solid and dashed line represent the projection direction and the class boundary projection respectively.}	
	\label{introduction}	
\end{figure}
Taking all the discussed challenges into consideration, we propose a novel formulation of self-weighted robust LDA for multi-class classification with edge classes, termed as SWRLDA. The optimal mean estimation is automatically avoided by maximizing the sum of distances between every two class centers, rather than that between each class center and the total mean.
Considering the fact that $\ell_{2}$-norm is prone to be dominated by large between-class distances corresponding to edge class, $\ell_{2,1}$-norm is exploited to further reduce the effect of edge classes and simultaneously enhance the robustness. 
An efficient re-weighted algorithm is adopted to optimize all the projection directions simultaneously without increasing the time complexity in contrast to the state-of-the-art LDA methods \cite{zhao2018new,wen2018robust,LuoNCYHZ16,iosifidis2013optimal,NieZYCLS17,ChangHSLYH18}. By optimizing the objective function of SWRLDA, the self-adaptive weights of each class pair can be learned implicitly in the embedding space. In this way, SWRLDA is capable of investigating the contribution of each class pair to classification automatically, instead of tuning additional parameters. Besides, the optimal weight corresponding to each class pair is learned based on their distances in the desired subspace such that it is more robust when the original data contain large noise. Extensive experiments on both synthetic and real-world datasets demonstrate the superior classification performance and high computational efficiency of the proposed method.

\section{Related Work}
\subsection{Linear Discriminant Analysis}
In previous literature, various extensions of naive LDA have been developed for further enhancement. For example, orthogonal LDA (OLDA) \cite{ye2006null}, uncorrelated LDA (ULDA) \cite{ye2006feature} were proposed to address the small sample size problem of naive LDA. To capture the geometric structure of data, discriminant locality preserving projection ($\ell_{2,1}$-DLPP) \cite{liu2018l_}, discriminative locality alignment (DLA) \cite{zhang2008discriminative}, and manifold partition discriminant analysis (MPDA) \cite{zhou2016manifold} were proposed. To classify high-dimensional data without feature selection, Peng \etal \cite{peng2016supervised,peng2019discriminative} introduced LDA based discriminant ridge models to classification.
However, all those methods maximize the sum of distances between each class mean and the total data mean \cite{rao1948utilization,ChangSWLL14,cai2008srda,LiNCNZY18,ChangYYX16,XueNWCSY17}, instead of explicitly maximizing the distance between every two class pairs. In such a way, it is inevitable that several classes are totally overlapped in the learned subspace, especially in multi-class classification with edge classes. To achieve the separability of every binary classes, FLDA is proposed to maximize the pairwise distance between every two class means, but not just the average distances \cite{fukunaga2013introduction}.
Specifically, the between-class scatter matrix of FLDA is usually rewritten as $\mathbf{S}_b=\sum_{i=1}^{c-1}\sum_{j=i+1}^{c}\frac{n_i n_j}{n^2}(\mathbf{\overline{x}}_i-\mathbf{\overline{x}}_j)(\mathbf{\overline{x}}_i-\mathbf{\overline{x}}_j)^\top$, which is measured by distances between each class pair instead of that between each class center and the total mean. Based on this reformulation, the $c$-class fisher criterion can be decomposed to $\frac{1}{2}c(c-1)$ two-class fisher criterion. FLDA is typically solved by a two-step algorithm, which first learns a whitened space to obtain the optimal distance metric, and then conducts dimension reduction in the learned whitened space. However, FLDA still cannot fully address the class separation problem. First, there is no guarantee that FLDA can prevent the incompatibility and information loss between the two independent stages, making the final result unreliable. Second, the quadratic distances between all the class pairs are maximized with equal weights, which is easily dominated by the classes located remotely from the others. 
\subsection{Weighted Pairwise Distance Criterion}
To attach more importance to the closer class pairs, many weighting schemes have been proposed to learn different weights for different class pairs. For example, Loog \etal \cite{loog2001multiclass} developed a weighting function based on the Mahanalobis distance between the $i$th and $j$th class in the original space, which can be efficiently solved by eigen-decomposition as in traditional LDA. However, the weights are simply calculated according to the distances in the original space, but not calculated according to the distances in the optimal subspace. Therefore, the calculated weights might not the optimal weights, especially when the data distribution in the optimal subspace changes largely from the original space or when the data contain large noise.
GMSS \cite{tao2009geometric} and HMSS \cite{bian2008harmonic} are also weighting methods to reduce the class separation problem. GMSS \cite{tao2009geometric} was proposed to maximize the weighted geometric mean of between-class distances under the homoscedastic Gaussian assumption, \ie, $\max_{\mathbf{W}^\top\mathbf{S}_w\mathbf{W}=\mathbf{I}}\sum_{i\neq j}p_ip_j \log\Delta_{ij}$.
Bian and Tao \etal \cite{bian2008harmonic} presented HMSS to further reduce the class separation problem by maximizing the weighted harmonic mean, \ie, $\max_{\mathbf{W}^\top\mathbf{S}_w\mathbf{W}=\mathbf{I}}-\sum_{i\neq j}p_ip_j {\Delta_{ij}}^{-1}$, where ${\Delta_{ij}}$ is the distance between the $i$th and $j$th class in the transformed subspace. However, since both HMSS and GMSS method don't have closed-form solution, gradient optimization is adopted to solve them, which converges very slow in some cases. Recently, many authors have proposed to seek such a transformation matrix $\mathbf{W}$ that the class pairs with smallest projected distance would be as far as possible by optimizing the following problem, \ie,  
$\max_{\mathbf{W}^\top\mathbf{S}_w\mathbf{W}=\mathbf{I}}\min_{i\neq j}p_ip_j {\Delta_{ij}}$ \cite{bian2011max,xu2010dimensionality,yu2011distance,zhang2010worst,GongTCY19}.
To guarantee the separation of all the class pairs, Abou-Moustafa \etal \cite{abou2010pareto} further proposed to maximize all the pairwise distances simultaneously. Although these methods
have reported improved performance compared with FLDA, they are all based on some complex iterative optimization procedures to solve the models, which makes them not scalable for large-scale dimension reduction problems.

\section{The Proposed Methodology} 
In this section, we will elaborate the formulation of the proposed self-weighted robust LDA model for multi-class classification. The notations used in this paper are introduced as follows. Following the standard notation, we denote vector and matrix by bold lowercase letters $(\eg, \mathbf{a})$ and bold uppercase letters $(\eg, \mathbf{A})$ respectively. The $\ell_{2}$-norm of vector $\mathbf{a}\in\mathbb{R}^n$ is written as ${\|\mathbf{a}\|}_2^2=\sum_{i=1}^{n}{a_i}^2$ and the $\ell_{2,1}$-norm of matrix $\mathbf{A}\in\mathbb{R}^{m\times n}$ is defined as ${\|\mathbf{A}\|}_{2,1}=\sum_{i=1}^{m}\sqrt{\sum_{j=1}^{n}{\mathbf{A}_{i,j}}^2}$, where $a_i$ refers to the $i$-th element of vector $\mathbf{a}$ and $\mathbf{A}_{i,j}$ denotes the element corresponding to the $i$th row and $j$th column of matrix $\mathbf{A}$.
\subsection{Problem Formulation}
Given training data $\mathbf{X}=[\mathbf{x}_1,\mathbf{x}_2,\dots,\mathbf{x}_n]\in\mathbb{R}^{d\times n}$ belonging to $c$ classes for dimension reduction, where $n$ refers to the number of data points and each data point is represented by a $d$-dimensional feature vector. The goal of LDA is to learn a linear projection matrix $\mathbf{W}\in\mathbb{R}^{d\times m}(m\ll d)$, which can transform original $d$-dimensional data to a $m$-dimensional embedding subspace with higher class separability. This is achieved by maximizing the between-class difference and meanwhile minimizing the within-class variability in the projected space, which can be formulated as the maximization of Fisher's criterion:
\begin{align}
\label{fisher}
\max_{\mathbf{W}}\mathbf{Tr}(\frac{(\mathbf{W}^\top\mathbf{S}_b\mathbf{W})}{(\mathbf{W}^\top\mathbf{S}_w\mathbf{W})}).
\end{align}
Specifically, $\mathbf{S}_w\in\mathbb{R}^{d\times d}$ and $\mathbf{S}_b\in\mathbb{R}^{d\times d}$ refer to the within-class and between-class scatter matrix respectively, which are defined as
\begin{align}
\label{sw}
&\mathbf{S}_w=\sum_{i=1}^{c}\sum_{j=1}^{n_i}(\mathbf{x}_j-\mathbf{\overline{x}}_i)(\mathbf{x}_j-\mathbf{\overline{x}}_i)^\top,\\
\label{sb}
&\mathbf{S}_b=\sum_{i=1}^c\frac{n_i}{n}(\mathbf{\overline{x}}_i-\mathbf{\overline{x}})(\mathbf{\overline{x}}_i-\mathbf{\overline{x}})^\top,
\end{align}
where $n_i$ refers to the number of data points belonging to the $i$th class. Moreover, $\mathbf{\overline{x}}=\frac{1}{n}\sum_{i=1}^{c}n_i\mathbf{\overline{x}}_i$ is the total mean of $\mathbf{X}$ and $\mathbf{\overline{x}}_i=\frac{1}{n_i}\sum_{j=1}^{n_i} \mathbf{x}_j$ is the mean of the $i$th class.
To guarantee the uniqueness of optimal solution, optimization problem (\ref{fisher}) can be reformulated by replacing the denominator with an equality constraint as: 
\begin{align}
\label{fisher2}
\max_{\mathbf{W}} \sum_{i=1}^{c}\frac{n_i}{n}\|\mathbf{W}^\top(\mathbf{\overline{x}}_i-\mathbf{\overline{x}})\|_2^2,\ \ \  s.t.\ \ \ \mathbf{W}^\top\mathbf{S}_w\mathbf{W}=\mathbf{I},
\end{align}
where $\mathbf{I}\in\mathbb{R}^{m\times m}$ is an identity matrix. In Eq. (\ref{fisher2}), the between-class distance is measured by subtracting the total mean from each class center, while the average of data is directly regarded as the optimal mean. In this way, the between-class scatter matrix calculation can be  easily dominated by edge classes, which decreases the robustness of LDA model. To this end, we equivalently reformulate the between-class distance criterion in traditional LDA to avoid the optimal mean calculation by introducing the following lemma and theorem.  

\begin{lemma}
	\label{lemma1}	
	If $\mathbf{p}\in\mathbb{R}^{c}\geq0$, $\mathbf{p}^\top\mathbf{1}=1$, $\mathbf{\overline{u}}=\sum_{j=1}^{c}p_j\mathbf{u}_j\in\mathbb{R}^{d}$, then the following equation holds:
	\begin{align}
	\label{lemma}
	\sum_{i=1}^{c}p_i(\mathbf{u}_i-\mathbf{\overline{u}})^\top(\mathbf{u}_i-\mathbf{\overline{u}})=\sum_{i,j=1}^{c}\frac{p_ip_j}{2}(\mathbf{u}_i-\mathbf{u}_j)^\top(\mathbf{u}_i-\mathbf{u}_j),
	\end{align}
	where $p_i$ and $p_j$ refer to the $i$th and $j$th element of $\mathbf{p}$.
\end{lemma} 
\begin{proof}
	Based on the fact that $\mathbf{\overline{u}}=\sum_{j=1}^{c}p_j\mathbf{u}_j$, the left side of Eq. (\ref{lemma}) can be derived as $\sum_{i=1}^{c}p_i(\mathbf{u}_i-\sum_{j=1}^cp_j\mathbf{u}_j)^\top(\mathbf{u}_i-\sum_{j=1}^cp_j\mathbf{u}_j)$, which can be further decomposed into the sum of three items, \ie, 
	\begin{align}
	\label{left}
	\sum_{i=1}^cp_i\mathbf{u}_i^\top\mathbf{u}_i-2\sum_{i=1}^cp_i\mathbf{u}_i^\top\sum_{j=1}^cp_j\mathbf{u}_j+\sum_{j=1}^{c}p_j\mathbf{u}_j^\top\sum_{j=1}^{c}p_j\mathbf{u}_j.
	\end{align}
	By combining the second and third terms of Eq. (\ref{left}), we can arrive at	
	\begin{align}
	\label{l3}
	\sum_{i=1}^{c}p_i\mathbf{u}_i^\top\mathbf{u}_i-\sum_{i=1}^cp_i\mathbf{u}_i^\top\sum_{j=1}^{c}p_j\mathbf{u}_j.
	\end{align}
	Subsequently, the right side of Eq. (\ref{lemma}) can be reformulated as
	\begin{align}
	\frac{1}{2}\sum_{i,j=1}^{c}(p_ip_j\mathbf{u}_i^\top\mathbf{u}_i+p_ip_j\mathbf{u}_j^\top\mathbf{u}_j-2p_ip_j\mathbf{u}_i^\top\mathbf{u}_j).
	\end{align}
	By combining the first two items, it reduces to Eq. (\ref{l3}).
	The proof is completed.
\end{proof}
\begin{theorem}
	\label{theorem2}
	The optimization problem (\ref{fisher2}) is equivalent to 
	\begin{align}
	\label{fisher4}
	\max_{\mathbf{W}} \sum_{i=1}^{c}\sum_{j=1}^{c}\frac{n_in_j}{2n^2}\|\mathbf{W}^\top(\mathbf{\overline{x}}_i-\mathbf{\overline{x}}_j)\|_2^2,\ \ \  s.t.\ \ \ \mathbf{W}^\top\mathbf{S}_w\mathbf{W}=\mathbf{I}.
	\end{align}
\end{theorem}
\begin{proof}
	Let $\mathbf{p}$, $\mathbf{u}_i$ and $\mathbf{\overline{u}}$ in \Cref{lemma1} denote the vector  $[\frac{n_1}{n},\frac{n_2}{n},\dots,\frac{n_c}{n}]\in\mathbb{R}^c$, $\mathbf{W}^\top\mathbf{\overline{x}}_i\in\mathbb{R}^m(\forall i)$ and 
	$\mathbf{W}^\top\mathbf{\overline{x}}\in\mathbb{R}^m$ respectively. Based on this denotation, the optimization problem (\ref{fisher2}) is equivalent to the left side of Eq. (\ref{lemma}), while the optimization problem (\ref{fisher4}) is equivalent to the right side of Eq. (\ref{lemma}).
	It's evident that the optimization problem (\ref{fisher2}) and (\ref{fisher4}) equal to each other according to \Cref{lemma1}. The proof is completed.
\end{proof}
Different from traditional LDA formulation based on $c$-class fisher criterion in (\ref{fisher2}), the optimal mean calculation can be automatically avoided in (\ref{fisher4}), where the between-class variability is measured by the sum of distances between every two class centers, instead of that between each class center and the total mean. 
However, with the exploitation of $\ell_2$-norm based distance criterion, the learned projection directions are still easily dominated by the class pairs with very large deviation. For this issue, we further reformulate problem (\ref{fisher4}) via $\ell_{2,1}$-norm maximization as
\begin{align}
\label{fisher5}
\max_{\mathbf{W}} \sum_{i=1}^{c}\sum_{j=1}^{c}\frac{n_in_j}{2n^2}\|\mathbf{W}^\top(\mathbf{\overline{x}}_i-\mathbf{\overline{x}}_j)\|_2,\ \ \  s.t.\ \ \ \mathbf{W}^\top\mathbf{S}_w\mathbf{W}=\mathbf{I},
\end{align} 
where $\sum_{i=1}^{c}\sum_{j=1}^{c}\|\mathbf{W}^\top(\mathbf{\overline{x}}_i-\mathbf{\overline{x}}_j)\|_2$ is essentially the sum of multiple matrix $\ell_{2,1}$-norms. Intuitively, the distance criterion $\|\mathbf{W}^\top(\mathbf{\overline{x}}_i-\mathbf{\overline{x}}_j)\|_2$ is not quadratic and thus edge classes would have less importance to it than the squared distance criterion in (\ref{fisher4}).
Whenever edge class occurs in a classification problem, the projection direction of traditional LDA would be completely destroyed because its between-class distance criterion depends entirely on the biased mean calculation and at the same time overemphasizes the effect of class pairs with large deviation. 
In terms of SWRLDA, only the distance calculation corresponding to the edge class will be affected with the  exploitation of ``avoiding optimal mean'' strategy and this effect is further reduced via $\ell_{2,1}$-norm based distance criterion. Note that no extra weight
factor for each class pair is explicitly included in this objective function. However, by solving problem (\ref{fisher5}) with a re-weighted optimization algorithm, we will demonstrate that this formulation of self-weighted robust LDA can indeed adaptively learn an optimal weight for each class pair without additional parameters.

\subsection{Optimization Procedure}
Considering the complexity of directly solving the proposed non-smooth problem, we adopt an efficient iterative re-weighted algorithm to reformulate the proposed optimization problem (\ref{fisher5}) as 
\begin{align}
\label{opt4}
\max_{\mathbf{W}} \sum_{i=1}^{c}\sum_{j=1}^{c}\frac{n_in_j}{2n^2}\mathbf{s}_{ij}^\top\mathbf{W}^\top(\mathbf{\overline{x}}_i-\mathbf{\overline{x}}_j),\ \ \  s.t.\ \ \ \mathbf{W}^\top\mathbf{S}_w\mathbf{W}=\mathbf{I},
\end{align} 
where $\mathbf{s}_{ij}$ is set to be stationary, \ie, 
\begin{align}
\label{opt5}
\mathbf{s}_{ij}=
\left\{
\begin{array}{cl}
\frac{\mathbf{W}^\top(\mathbf{\overline{x}}_i-\mathbf{\overline{x}}_j)}{\|\mathbf{W}^\top(\mathbf{\overline{x}}_i-\mathbf{\overline{x}}_j)\|_2}, & \hbox{if $\|\mathbf{W}^\top(\mathbf{\overline{x}}_i-\mathbf{\overline{x}}_j)\|_2\neq0$;} \\
\mathbf{0}, & \hbox{if $\|\mathbf{W}^\top(\mathbf{\overline{x}}_i-\mathbf{\overline{x}}_j)\|_2=0$.}
\end{array}
\right.
\end{align} 
Note that an unknown variable $\mathbf{s}_{ij}$ depending on $\mathbf{W}$ is introduced into the optimization problem. Thus, alternating optimization scheme is exploited to update the two variables iteratively. In each iteration, $\mathbf{s}_{ij}$ is updated with the current solution of $\mathbf{W}$, and then $\mathbf{W}$ is recalculated with the updated $\mathbf{s}_{ij}$. This iterative procedure is repeated until the objective function converges to a certain value. Note that $\frac{1}{\|\mathbf{W}^\top(\mathbf{\overline{x}}_i-\mathbf{\overline{x}}_j)\|_2}$ can be naturally treated as the weight corresponding to the distance between the $i$th and $j$th class pair, \ie, $\|\mathbf{W}^\top(\mathbf{\overline{x}}_i-\mathbf{\overline{x}}_j)\|_2^2$, which constitutes the symmetric weighted matrix. It's noteworthy that the smaller distance corresponds to a larger weight, which indicates that more attention will be paid to the class pairs with high similarities in the learned subspace. Besides, the weight is self-adaptive and can be derived automatically without additional parameters.
For notation simplicity, we introduce a matrix  $\mathbf{M}=\sum_{i,j=1}^{c}\frac{n_in_j}{2n^2}(\mathbf{\overline{x}}_i-\mathbf{\overline{x}}_j)\mathbf{s}_{ij}^\top$ to transform the problem (\ref{opt4}) as:
\begin{align}
\label{opt6}
\max_\mathbf{W} \mathbf{Tr}(\mathbf{W}^\top\mathbf{M}),\ \ \ s.t.\ \ \ {\mathbf{W}^\top\mathbf{S}_w\mathbf{W}=\mathbf{I}}.
\end{align} 
To determine the closed-form solution of optimal projection matrix $\mathbf{W}$ in problem (\ref{opt6}), we introduce the following theorem.
\begin{theorem}
	\label{theorem1}	
	The SVD of matrix $\mathbf{A}=\mathbf{S}_w^{-\frac{1}{2}}\mathbf{M}\in\mathbb{R}^{d\times m}$ is $\mathbf{U}[\mathbf{\Lambda};\mathbf{0}]\mathbf{V}^\top$, where $\mathbf{U}\in\mathbb{R}^{d\times d}$ and $\mathbf{V}\in\mathbb{R}^{m\times m}$ are both orthonormal matrices and $\mathbf{\Lambda}=\mathbf{diag}(\lambda_{11},\lambda_{22},\dots,\lambda_{mm})\in\mathbb{R}^{m\times m}$ is a diagonal matrix with $\lambda_{kk}\geq0(\forall k)$. Then the optimal solution of the optimization problem $\max_{\mathbf{W}^\top\mathbf{S}_w\mathbf{W}=\mathbf{I}}\mathbf{Tr}(\mathbf{W}^\top\mathbf{M})$ can be derived as $\mathbf{S}_w^{-\frac{1}{2}}\mathbf{U}[\mathbf{I};\mathbf{0}]\mathbf{V}^\top$.
\end{theorem} 
\begin{proof}
	Since the matrix $\mathbf{A}$ equals to $\mathbf{S}_w^{-\frac{1}{2}}\mathbf{M}$, the objective function in (\ref{opt6}) can be rewritten as $\mathbf{Tr}(\mathbf{W}^\top\mathbf{S}_w^{\frac{1}{2}}\mathbf{A})$.
	Based on the SVD result of $\mathbf{A}$, we can further transform it into
	\begin{align*}
	\begin{split}
	\mathbf{Tr}([\mathbf{\Lambda};\mathbf{0}]\mathbf{V}^\top\mathbf{W}^\top\mathbf{S}_w^{\frac{1}{2}}\mathbf{U})
	=\mathbf{Tr}([\mathbf{\Lambda};\mathbf{0}]\mathbf{Z})=\sum_{i=1}^m \lambda_{ii}z_{ii},
	\end{split}
	\end{align*}
	where $\mathbf{Z}=\mathbf{V}^\top\mathbf{W}^\top\mathbf{S}_w^{\frac{1}{2}}\mathbf{U}$, $\lambda_{ii}$ and $z_{ii}$ are the $(i,i)$-th element of matrix $\mathbf{\Lambda}$ and $\mathbf{Z}$ respectively.
	Note that $\mathbf{Z}$ is an orthonormal matrix, \ie, $\mathbf{Z}^\top\mathbf{Z}=\mathbf{I}$, thus $z_{ii}\leq1$ holds for each $i$. Since $\lambda_{ii}$ is the singular value of $\mathbf{A}$, $\lambda_{ii}\geq0$ satisfies for each $i$, and thus we can derive the following inequality:
	\begin{align}
	\label{inequality}
	\mathbf{Tr}(\mathbf{W}^\top\mathbf{M})=\sum_{i=1}^m \lambda_{ii}z_{ii}\leq\sum_{i=1}^m \lambda_{ii}.
	\end{align}
	When $z_{ii}=1(\forall i)$, \ie, $\mathbf{Z}=\mathbf{I}$, the equality in (\ref{inequality}) holds and at the same time $\mathbf{Tr}(\mathbf{W}^\top\mathbf{M})$ reaches its maximum. Recall that $\mathbf{Z}=\mathbf{V}^\top\mathbf{W}^\top\mathbf{S}_w^{\frac{1}{2}}\mathbf{U}$, the optimal solution of problem $\max_{\mathbf{W}^\top\mathbf{S}_w\mathbf{W}=\mathbf{I}}\mathbf{Tr}(\mathbf{W}^\top\mathbf{M})$ is $\mathbf{S}_w^{-\frac{1}{2}}\mathbf{U}[\mathbf{I};\mathbf{0}]\mathbf{V}^\top$. The proof is completed.	
\end{proof}
The key steps of SWRLDA are summarized in \Cref{alg:wmgc}. We will theoretically analyze its computational complexity and convergence in the following parts.
\begin{algorithm}[t]
	\caption{SWRLDA with re-weighted optimization  \label{alg:wmgc}}
	\begin{algorithmic}[1]
		\Require
		Training data $\mathbf{X}=[\mathbf{x}_1,\mathbf{x}_2,\dots,\mathbf{x}_n]\in\mathbb{R}^{d\times n}$ belonging to $c$ classes.
		\State Initialize projection matrix $\mathbf{W}_{(t)}$; $t=0$.
		\State Compute the within class scatter matrix $\mathbf{S}_w$ according to Eq. (\ref{sw}) and then denote $\mathbf{S}_w^{'}=\mathbf{S}_w^{-\frac{1}{2}}$;
		\While{not converge}
		\State Update ${\mathbf{s}_{ij}}_{(t)}$($\forall i,j$) according to Eq. (\ref{opt5});
		\State Compute the matrix $\mathbf{M}_{(t)}=\sum_{i,j=1}^{c}\frac{n_in_j}{2n^2}(\mathbf{\overline{x}}_i-\mathbf{\overline{x}}_j){\mathbf{s}_{ij}}_{(t)}^\top$ with the updated ${\mathbf{s}_{ij}}_{(t)}$;
		\State Calculate matrix $\mathbf{A}=\mathbf{S}_w^{'}\mathbf{M}_{(t)}$ with the updated $\mathbf{M}_{(t)}$ and then obtain the SVD result of matrix $\mathbf{A}$ as $\mathbf{U}[\mathbf{\Lambda};\mathbf{0}]\mathbf{V}^\top$;
		\State Update projection matrix $\mathbf{W}_{(t)}=\mathbf{S}_w^{'}\mathbf{U}[\mathbf{I};\mathbf{0}]\mathbf{V}^\top$;
		\State $t=t+1$;
		\EndWhile
		\Ensure Optimal projection matrix $\mathbf{W}\in\mathbb{R}^{d\times m}$.
	\end{algorithmic}
\end{algorithm}


\subsection{Computational Complexity Reduction Analysis}
Despite the fact that SWRLDA traverses the distance between all the class pairs, its computational complexity is $O(cmdt)$ which doesn't increase in contrast to the state-of-the-art LDA methods \cite{zhao2018new,wen2018robust,iosifidis2013optimal}.

The multiplication of matrix $\mathbf{W}^\top$ and $\mathbf{\overline{X}}$, \ie, $\mathbf{W}^\top\mathbf{\overline{X}}$, is computed as $\mathbf{P}=[\mathbf{p}_1,\mathbf{p}_2,\dots,\mathbf{p}_n]\in\mathbb{R}^{m\times c}$ with computational complexity $O(cdm)$, where $\mathbf{\overline{X}}=[\mathbf{\overline{x}}_1,\mathbf{\overline{x}}_2,\dots,\mathbf{\overline{x}}_c]\in\mathbb{R}^{d\times c}$ collects all the class mean vectors. According to Eq. (\ref{opt5}), $\mathbf{s}_{ij}(i,j=1,2,\dots,c)$ can be updated with the entire time complexity $O(mc^2)$, which is further exploited to calculate the matrix $\mathbf{M}$. It seems that the computational cost of $\mathbf{M}$ is $O(c^2dm)$. Actually, it can be accelerated by considering the computation of $\mathbf{s}_{ij}$ and $\mathbf{M}$ in a comprehensive way. The computation of $\mathbf{M}=\sum_{i,j=1}^{c}\frac{n_in_j}{2n^2}(\mathbf{\overline{x}}_i-\mathbf{\overline{x}}_j)\mathbf{s}_{ij}^\top$ can be reformulated as $\frac{1}{2n^2}\sum_{i=1}^{c}n_i\mathbf{\overline{x}}_i\sum_{j=1}^{c}n_j\mathbf{s}_{ij}^\top-\frac{1}{2n^2}\sum_{j=1}^{c}n_j\mathbf{\overline{x}}_j\sum_{i=1}^{c}n_i\mathbf{s}_{ij}^\top$.
Under the condition that $\mathbf{s}_{i.}^\top=\sum_{j=1}^{c}n_j\mathbf{s}_{ij}^\top$ and $\mathbf{s}_{.j}^\top=\sum_{i=1}^{c}n_i\mathbf{s}_{ij}^\top$, it reduces to $\frac{1}{2n^2}\sum_{i=1}^{c}n_i\mathbf{\overline{x}}_i(\mathbf{s}_{i.}^\top-\mathbf{s}_{.i}^\top)$.
Thus, when $\mathbf{s}_{i.}$ and $\mathbf{s}_{.i}$ are given, the matrix $\mathbf{M}$ can be computed with time complexity $O(cdm)$. Based on Eq. (\ref{opt5}), we can derive $\mathbf{s}_{i.}=\sum_{j=1}^{c}n_j\mathbf{s}_{ij}=\sum_{j=1}^{c}n_j\frac{\mathbf{p}_i-\mathbf{p}_j}{\|\mathbf{p}_i-\mathbf{p}_j\|_2}$ and $\mathbf{s}_{.j}=\sum_{i=1}^{c}n_i\mathbf{s}_{ij}=\sum_{i=1}^{c}n_i\frac{\mathbf{p}_i-\mathbf{p}_j}{\|\mathbf{p}_i-\mathbf{p}_j\|_2}$, where $\|\mathbf{p}_i-\mathbf{p}_j\|_2\neq0$. Therefore, $\mathbf{s}_{i.}$ and $\mathbf{s}_{.i}(\forall i)$ can be calculated with time complexity $O(c^2m)$. Next, the computational cost of the following steps mainly lies in the SVD of matrix $\mathbf{A}\in\mathbb{R}^{d\times m}=\mathbf{U\Sigma V}$. Note that the left singular vectors of $\mathbf{A}$, \ie, the column vectors of $\mathbf{U}$, are the eigenvectors of $\mathbf{AA}^\top$ and the right singular vectors of $\mathbf{A}$, \ie, the column vectors of $\mathbf{V}$,  are the eigenvectors of $\mathbf{A}^\top\mathbf{A}$ \cite{stewart2001eigensystem}. If $\mathbf{U}$ or $\mathbf{V}$ is given, the other one can be recovered by the equation $\mathbf{AV=U\Sigma}$ and $\mathbf{U}^\top\mathbf{A=\Sigma V}^\top$. Due to the fact that $d\geq m$, we only need to compute the eigenvectors of $\mathbf{A}^\top\mathbf{A}$ as $\mathbf{V}$, and then recover $\mathbf{U}$ from $\mathbf{V}$, whose total computational complexity is $O(dm^2+m^3)$. 

In conclusion, the whole time complexity of the proposed method is $O((cd+c^2+md+m^2)mt)$, where $t$ is the iteration number. Since the feature dimensionality $d$ is usually much larger than the number of classes $c$ and the reduced dimensionality $m$, the computational complexity of \Cref{alg:wmgc} can be simplified as $O(cmdt)$. 

\subsection{Convergence Analysis}
In this section, we will analyze the convergence of the proposed SWRLDA by introducing the following lemma and theorem.
\begin{lemma}
	\label{lemma2}
	The objective function in optimization problem (\ref{fisher5}) is upper bounded.
\end{lemma}	
\begin{proof}
	For a given $i$, the subproblem of (\ref{fisher5}) can reduce to $\sum_{j=1}^{c}\|\mathbf{W}^\top(\mathbf{\overline{x}}_i-\mathbf{\overline{x}}_j)\|_2$ by not considering the constants in $\frac{n_in_j}{2n^2}$. It can be further derived as $\sum_{j=1}^{c}\|\mathbf{B}^\top\mathbf{d}_{ij}\|_2$, where $\mathbf{B}=\mathbf{S}_w^{\frac{1}{2}}\mathbf{W}$ with $\mathbf{B}^\top\mathbf{B}=\mathbf{I}$ and $\mathbf{d}_{ij}=\mathbf{S}_w^{-\frac{1}{2}}(\mathbf{\overline{x}}_i-\mathbf{\overline{x}}_j)$. According to Cauchy-Schwarz inequality, we have the following derivations:
	\begin{align*}
	\begin{split}
	\sum_{j=1}^{c}\|\mathbf{B}^\top\mathbf{d}_{ij}\|_2&\leq\sum_{j=1}^{c}\|\mathbf{B}^\top\mathbf{d}_{ij}\|_1
	=\sum_{k=1}^{m}\sum_{j=1}^{c}\|\mathbf{b}_k^\top\mathbf{d}_{ij}\|_1\\&\leq\sum_{k=1}^{m}\sum_{j=1}^{c}\|\mathbf{b}_k^\top\|_2\|\mathbf{d}_{ij}\|_2=\sum_{j=1}^{c}\lambda\|\mathbf{d}_{ij}\|_2,
	\end{split}
	\end{align*}
	where $\mathbf{b}_k$ is the $k$th column of matrix $\mathbf{B}$. For a given dataset, $\sum_{j=1}^{c}\lambda\|\mathbf{d}_{ij}\|_2$ is a constant, which indicates the objective function in optimization problem (\ref{fisher5}) has an upper bound. The proof is completed.
\end{proof}
\begin{theorem}
	\label{theorem4} 
	{\hspace{0em}}
	The objective in optimization problem (\ref{fisher5}) will monotonically increase and converge to its maximum with \Cref{alg:wmgc}.
\end{theorem} 
\begin{proof} 
	For simplicity, let $\mathbf{v}_{ij}$ denote the vector $\mathbf{\overline{x}}_i-\mathbf{\overline{x}}_j$.
	With the fixed ${\mathbf{s}_{ij}}_{(t)}$ in the $t$th iteration, the optimal $\mathbf{W}_{(t)}$ can be obtained when the objective function in (\ref{opt4}) reaches its maximum, such that 
	\begin{align}
	\label{eq1}
	\sum_{i,j}\frac{n_in_j}{2n^2}{\mathbf{s}_{ij}}_{(t)}^\top\mathbf{W}_{(t)}^\top\mathbf{v}_{ij}\geq \sum_{i,j}\frac{n_in_j}{2n^2}{\mathbf{s}_{ij}}_{(t)}^\top\mathbf{W}_{(t-1)}^\top\mathbf{v}_{ij}.
	\end{align}
	According to Cauchy-Schwarz inequality, it's easily to get $	\|\mathbf{W}_{(t)}^\top\mathbf{v}_{ij}\|_2\|\mathbf{W}_{(t-1)}^\top\mathbf{v}_{ij}\|_2\geq \langle \mathbf{W}_{(t)}^\top\mathbf{v}_{ij},\mathbf{W}_{(t-1)}^\top\mathbf{v}_{ij} \rangle$.
	Based on this inequality and Eq. (\ref{opt5}), we have $\|\mathbf{W}_{(t)}^\top\mathbf{v}_{ij}\|_2- {\mathbf{s}_{ij}}_{(t)}^\top\mathbf{W}_{(t)}^\top\mathbf{v}_{ij}\geq 0$, and Eq. (\ref{opt5}) can be reformulated as $\|\mathbf{W}_{(t-1)}^\top\mathbf{v}_{ij}\|_2-{\mathbf{s}_{ij}}_{(t)}^\top\mathbf{W}_{(t-1)}^\top\mathbf{v}_{ij}=0$.
	Combining the above two equations, we can arrive at $\|\mathbf{W}_{(t)}^\top\mathbf{v}_{ij}\|_2- {\mathbf{s}_{ij}}_{(t)}^\top\mathbf{W}_{(t)}^\top\mathbf{v}_{ij}
	\geq \|\mathbf{W}_{(t-1)}^\top\mathbf{v}_{ij}\|_2-{\mathbf{s}_{ij}}_{(t)}^\top\mathbf{W}_{(t-1)}^\top\mathbf{v}_{ij}$.
	Since the above inequality holds for each $i$ and $j$, we can derive that
	\begin{align}
	\label{eq5}
	\begin{split}
	&\sum_{i,j}\frac{n_in_j}{2n^2}(\|\mathbf{W}_{(t)}^\top\mathbf{v}_{ij}\|_2-{\mathbf{s}_{ij}}_{(t)}^\top\mathbf{W}_{(t)}^\top\mathbf{v}_{ij})\\
	\geq& \sum_{i,j}\frac{n_in_j}{2n^2}(\|\mathbf{W}_{(t-1)}^\top\mathbf{v}_{ij}\|_2-{\mathbf{s}_{ij}}_{(t)}^\top\mathbf{W}_{(t-1)}^\top\mathbf{v}_{ij}).
	\end{split}
	\end{align}
	Summing the two inequalities in (\ref{eq1}) and (\ref{eq5}) on both sides, we obtain
	\begin{align}
	\label{eq6}
	\begin{split}
	\sum_{i,j}\frac{n_in_j}{2n^2}\|\mathbf{W}_{(t)}^\top\mathbf{v}_{ij}\|_2
	\geq \sum_{i,j}\frac{n_in_j}{2n^2}\|\mathbf{W}_{(t-1)}^\top\mathbf{v}_{ij}\|_2.
	\end{split}
	\end{align}
	Thus, the objective value in problem (\ref{fisher5}) monotonically increases with \Cref{alg:wmgc} and is bounded above by a supremum based on \Cref{lemma2}, then it will converge to its maximum. The proof is completed.
\end{proof} 

\section{Experiments}
In this section, extensive experiments on both synthetic and real-world datasets are conducted to evaluate the effectiveness of the proposed method. The proposed SWRLDA is developed in the Matlab environment. All experiments are performed on the Windows-10 operating system (Intel Core i5-6200U CPU @ 2.40 GHz, 16 GB RAM).
In addition, seven state-of-the-art algorithms are selected for comparison, which are briefly introduced as follows:
\begin{enumerate}[$\bullet$,leftmargin=*,topsep=0pt,noitemsep]
	\item $c$-class fisher criterion based LDA models rely on the total mean calculation, whose between-class distance is calculated by subtracting the total mean from each class center. Specifically, \textbf{LDA} \cite{cai2008srda} corresponds to the traditional LDA model and \textbf{RSLDA} \cite{wen2018robust} is a robust version of LDA by introducing the $\ell_{2,1}$-norm minimization of projection matrix.
	\item Weighted pairwise fisher criterion based methods decompose the original $c$-class fisher criterion into $\frac{1}{2}c(c-1)$ two-class fisher criterion, and then attach certain weight to each class pair to characterize its importance. \textbf{aPAC} \cite{loog2001multiclass} assigned different weights according to their Mahalanobis distance in the original space. \textbf{GMSS} \cite{tao2009geometric} and \textbf{HMSS} \cite{bian2008harmonic} learned weights by maximizing the weighted geometric mean and weighted harmonic mean of KL divergences of class pairs respectively.
	\textbf{STRDA} \cite{li2017beyond} was proposed to maximize the weighted harmonic mean of pairwise trace ratios.
	\textbf{MMDA} \cite{bian2011max} maximized the minimal pairwise distance by introducing a local SDP relaxation.
	
\end{enumerate}
\subsection{Experiments on Synthetic Datasets}
\subsubsection{Experimental Setup}
\begin{figure}[t] 	
	\centering 	
	\subfigure[\emph{syn1} without edge class]{\label{2a}		
		\includegraphics[width=0.35\linewidth]{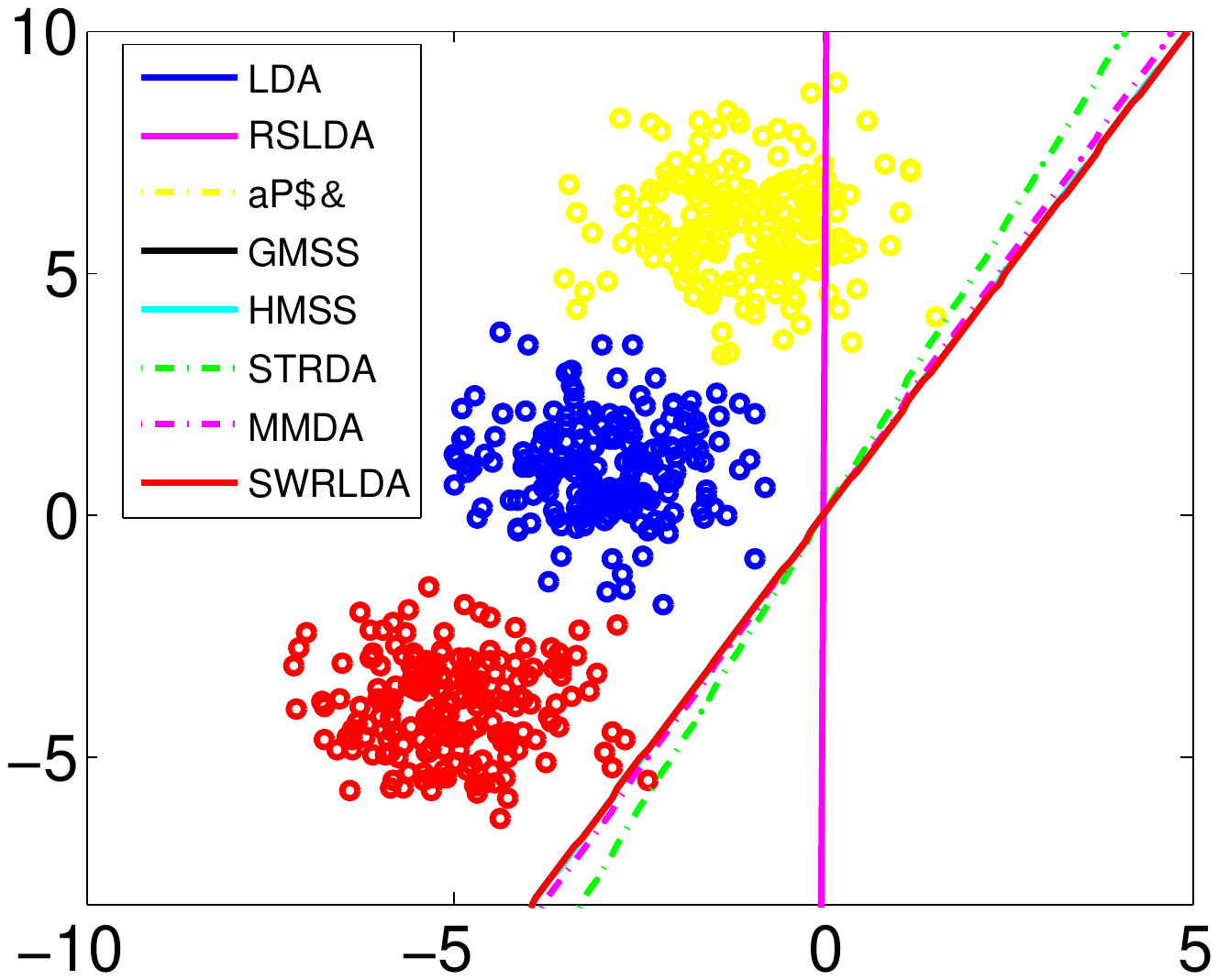}}
	\subfigure[\emph{syn2} with edge class]{\label{2b}				
		\includegraphics[width=0.35\linewidth]{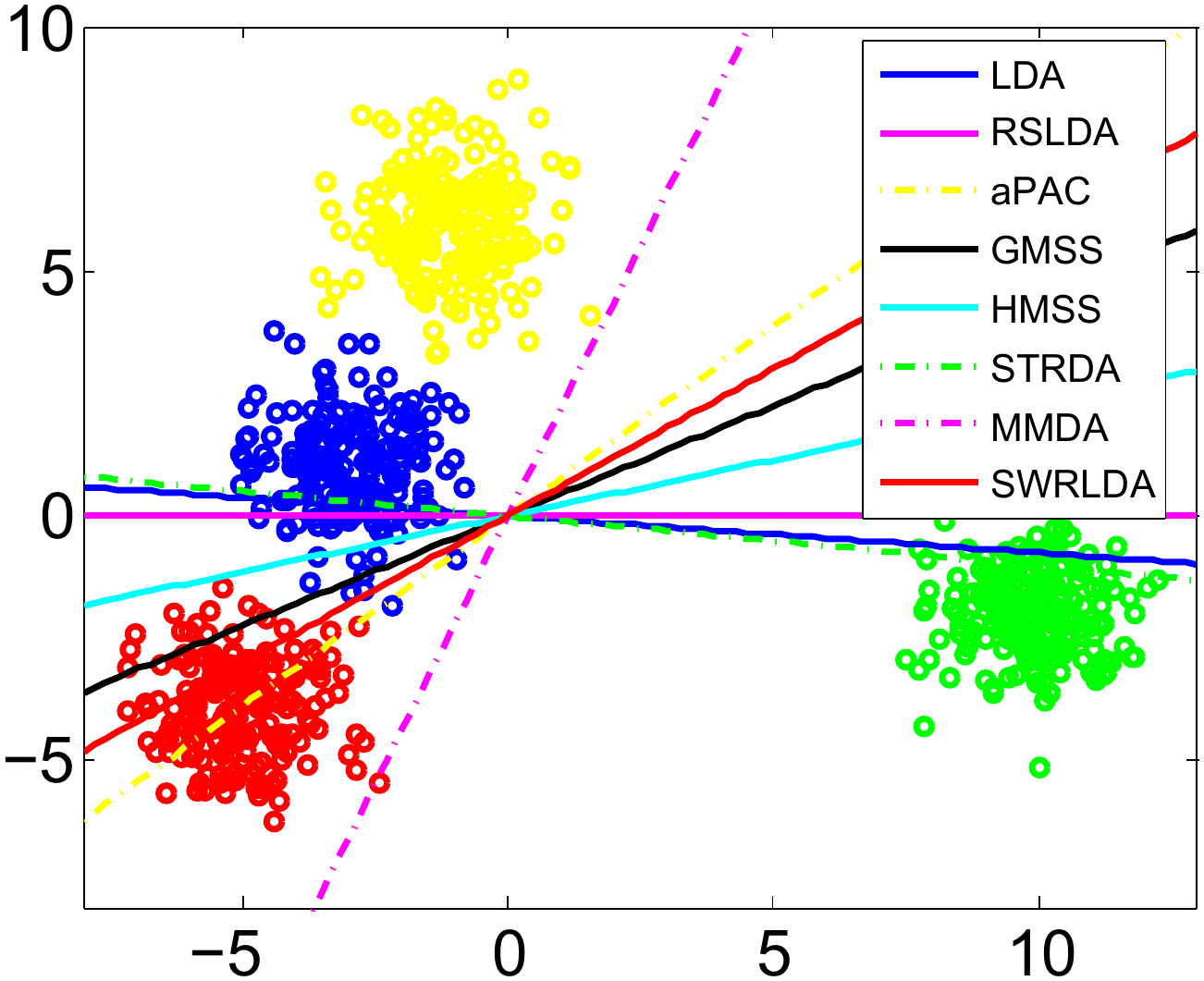}}	
	\caption{Projection directions learned by different methods on 2D synthetic data.}	
	\label{synthetic}	
\end{figure}
\begin{figure*}[t] 	
	\centering 	
	\subfigure[LDA]{\label{3a}		
		\includegraphics[width=0.23\linewidth]{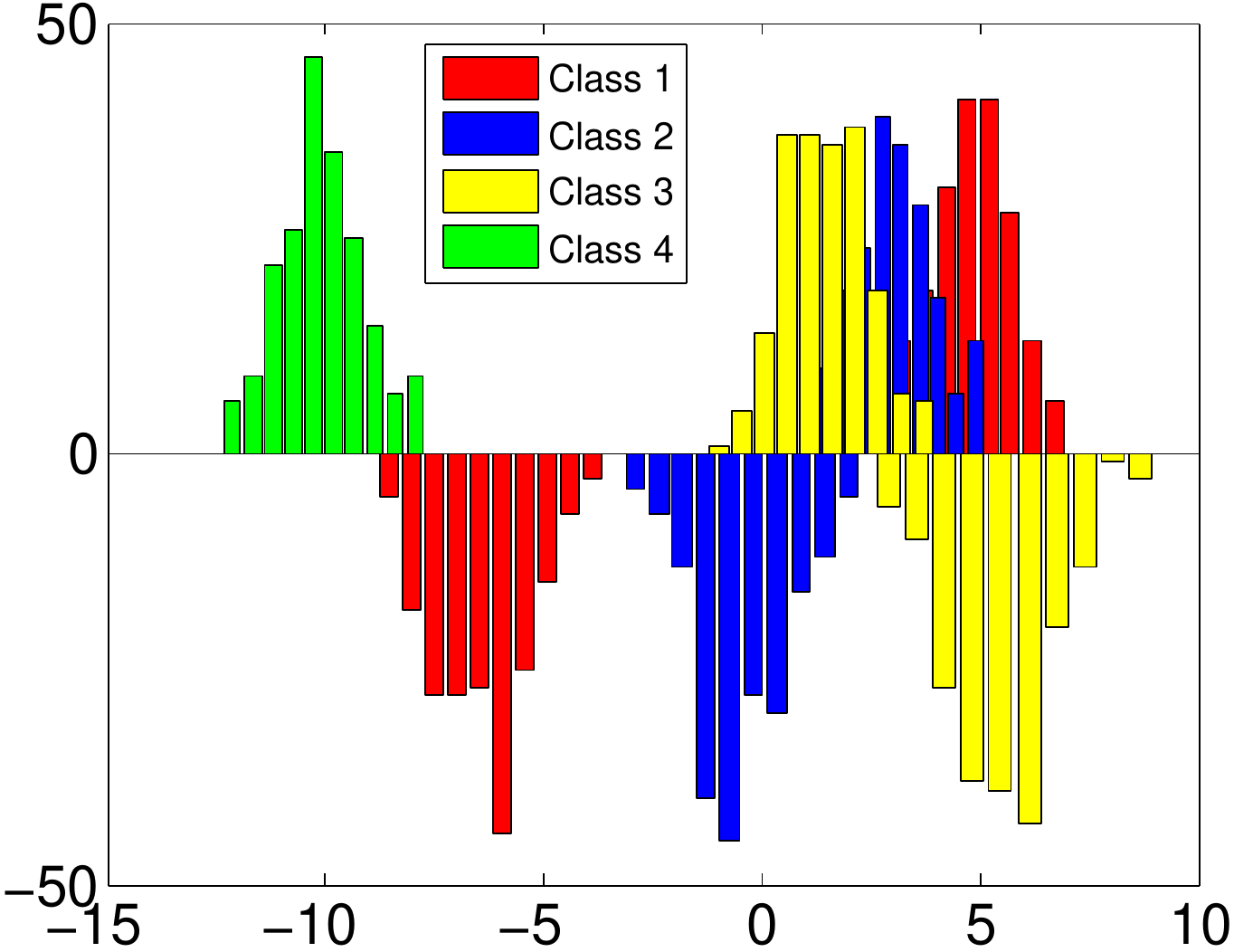}}
	\subfigure[RSLDA]{\label{3b}				
		\includegraphics[width=0.23\linewidth]{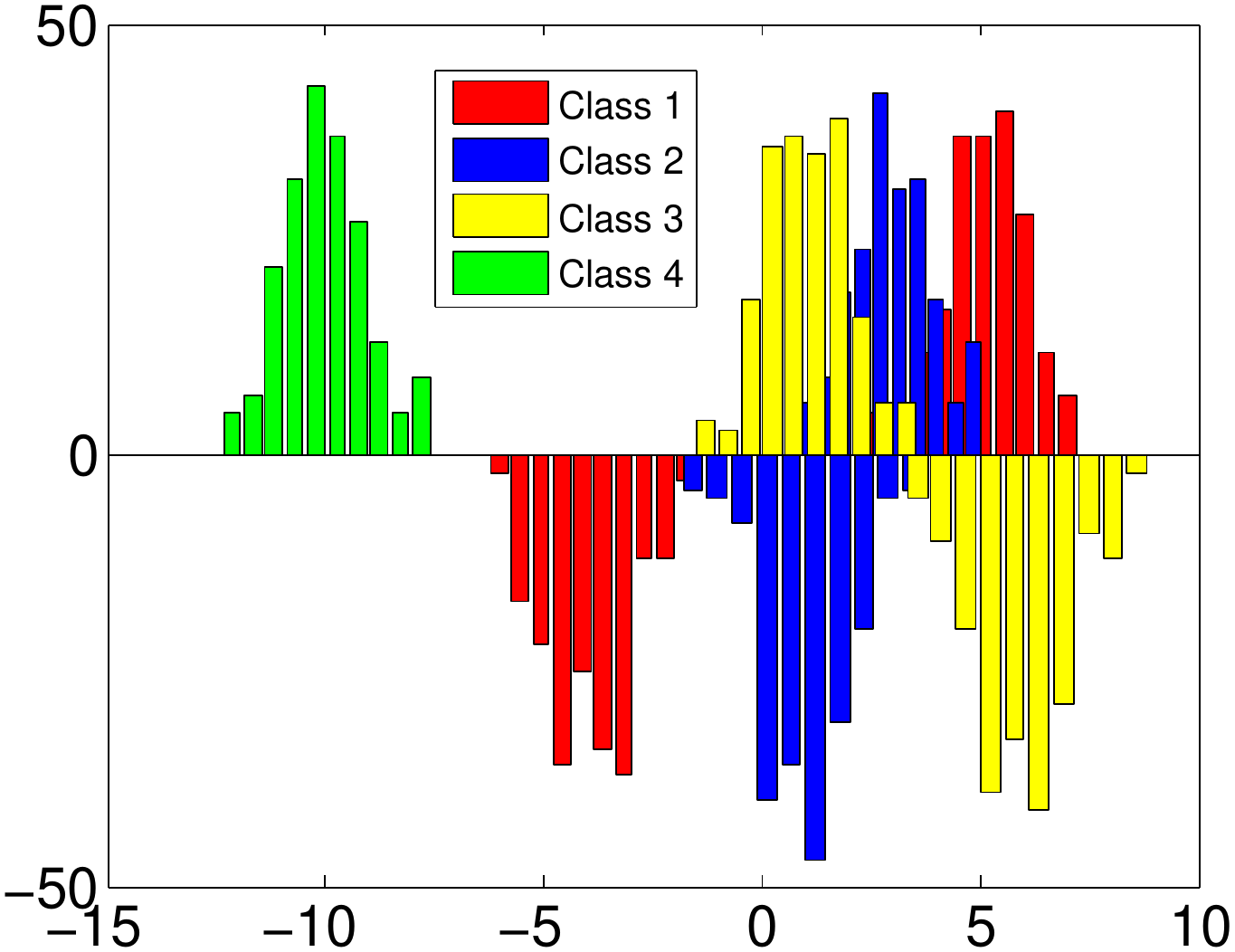}}	
	\subfigure[aPAC]{\label{3c}		
		\includegraphics[width=0.23\linewidth]{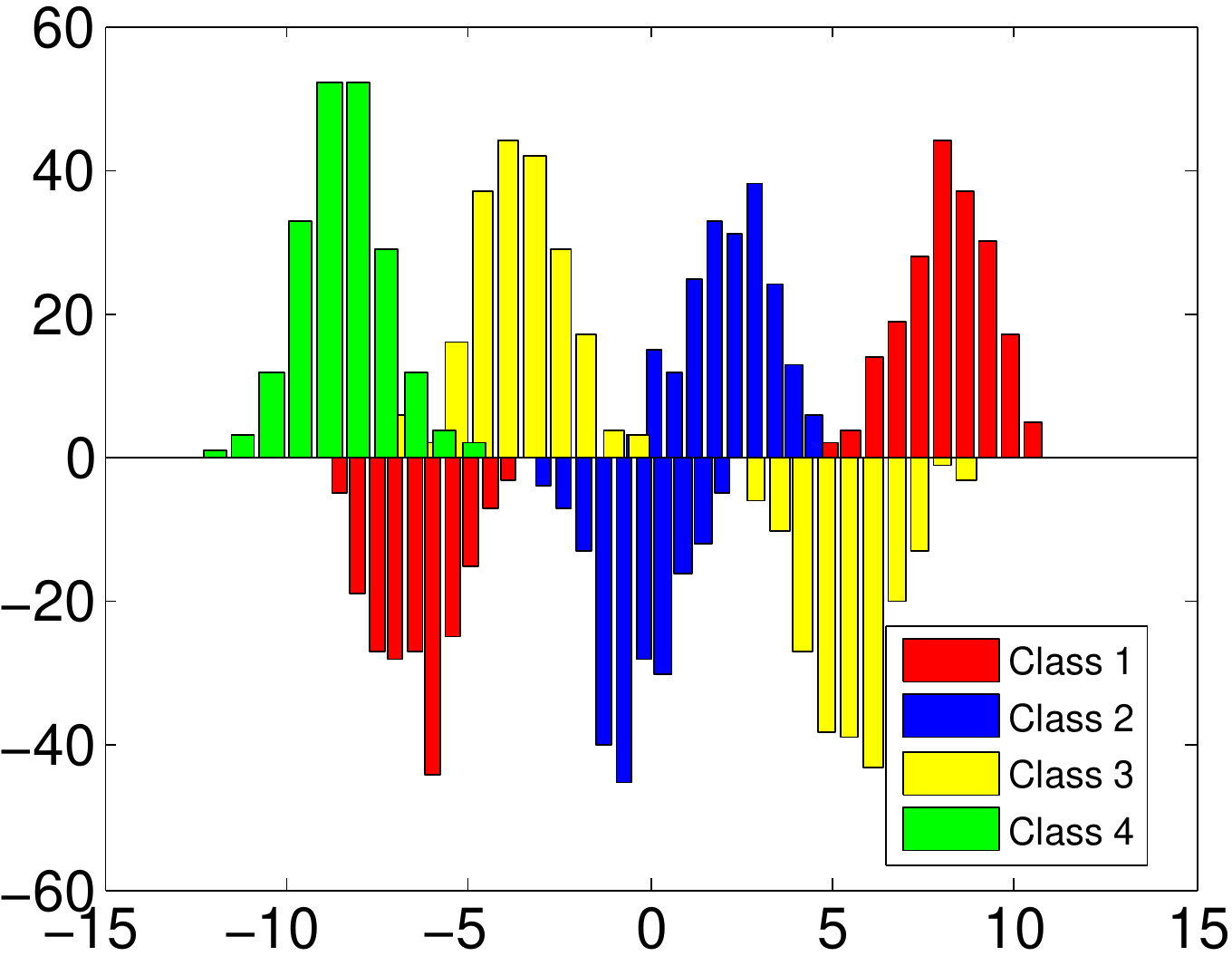}}		
	\subfigure[GMSS]{\label{3d}		
		\includegraphics[width=0.23\linewidth]{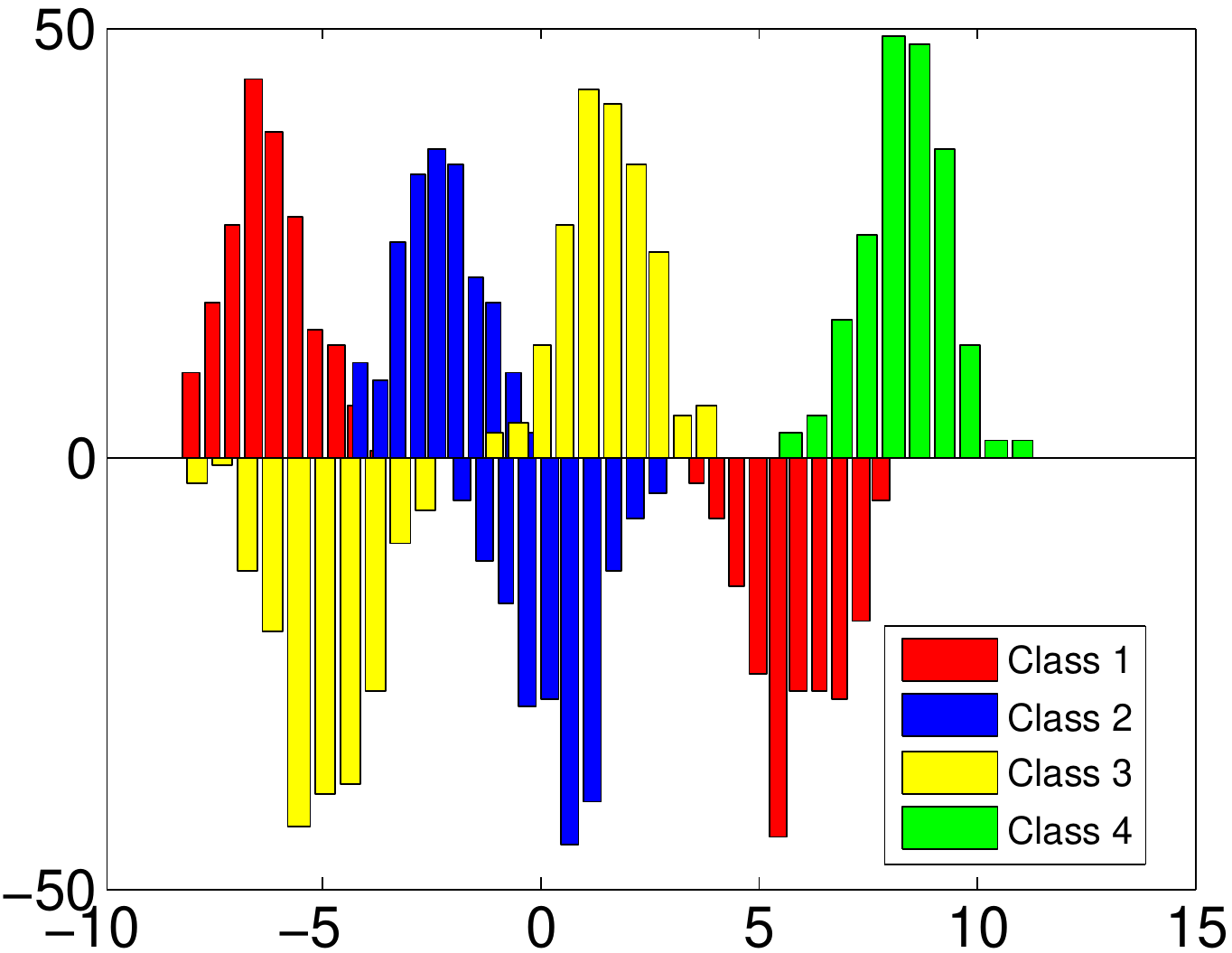}}
	\vfill	
	\subfigure[HMSS]{\label{3e}		
		\includegraphics[width=0.23\linewidth]{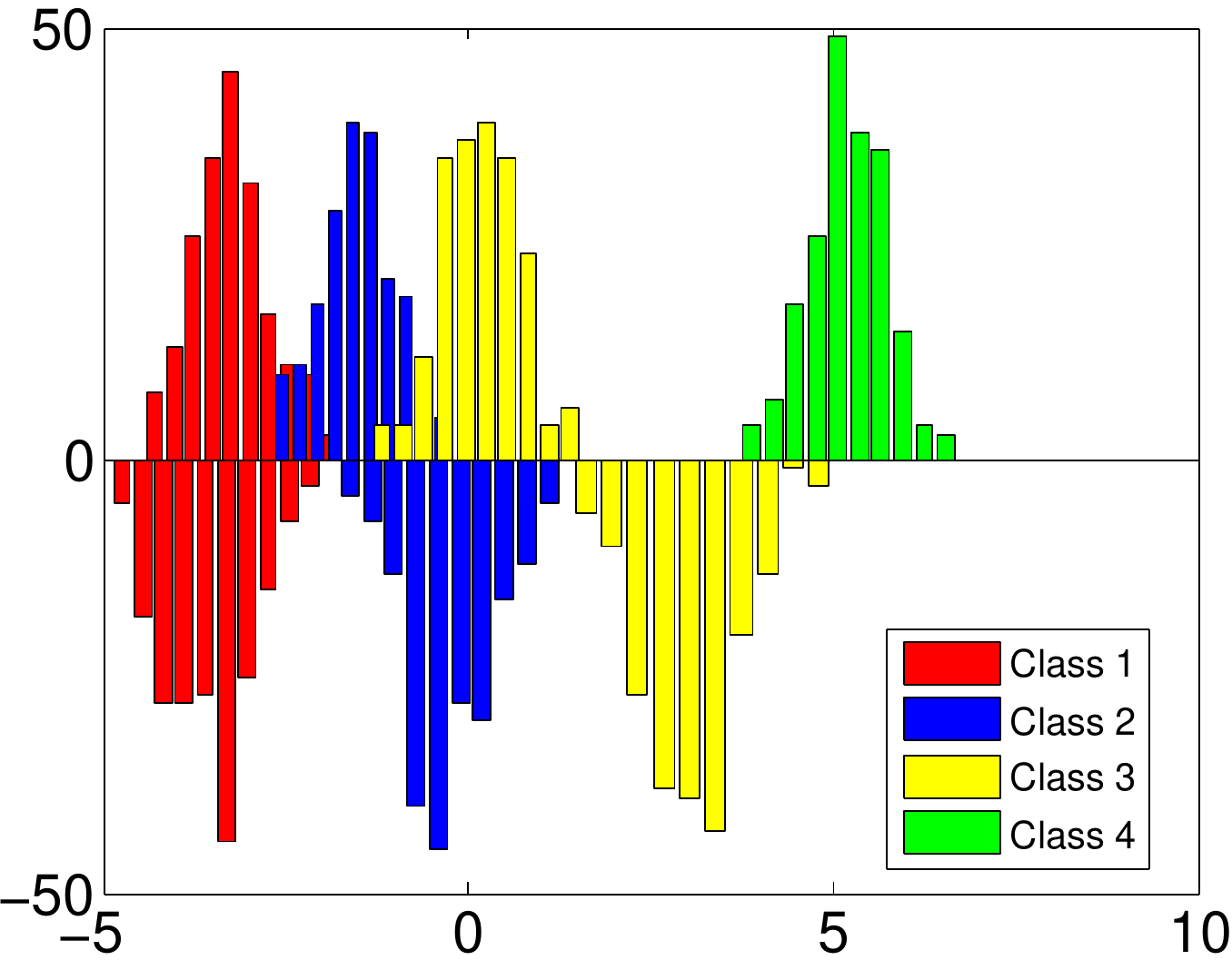}}
	\subfigure[STRDA]{\label{3f}				
		\includegraphics[width=0.23\linewidth]{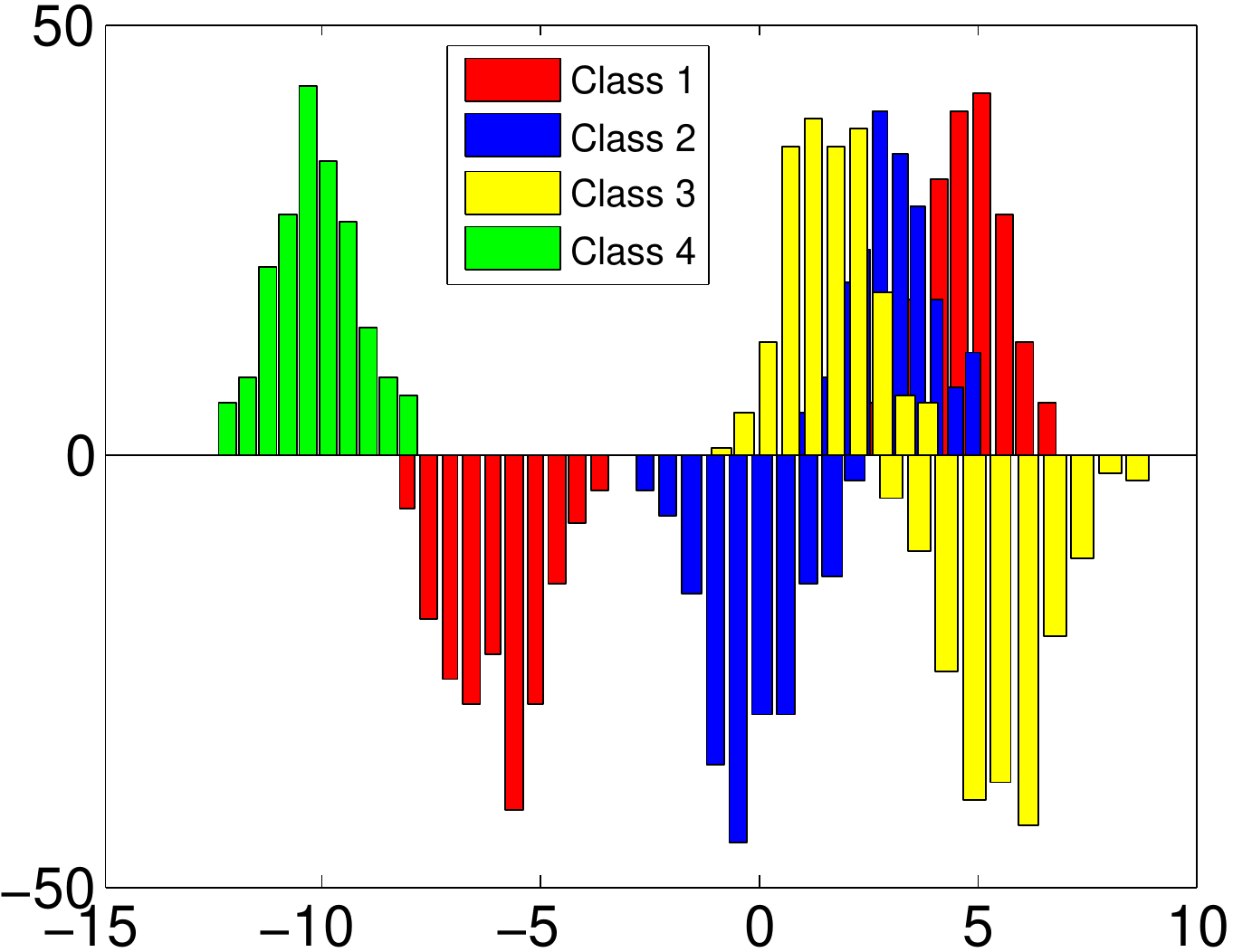}}	
	\subfigure[MMDA]{\label{3g}		
		\includegraphics[width=0.23\linewidth]{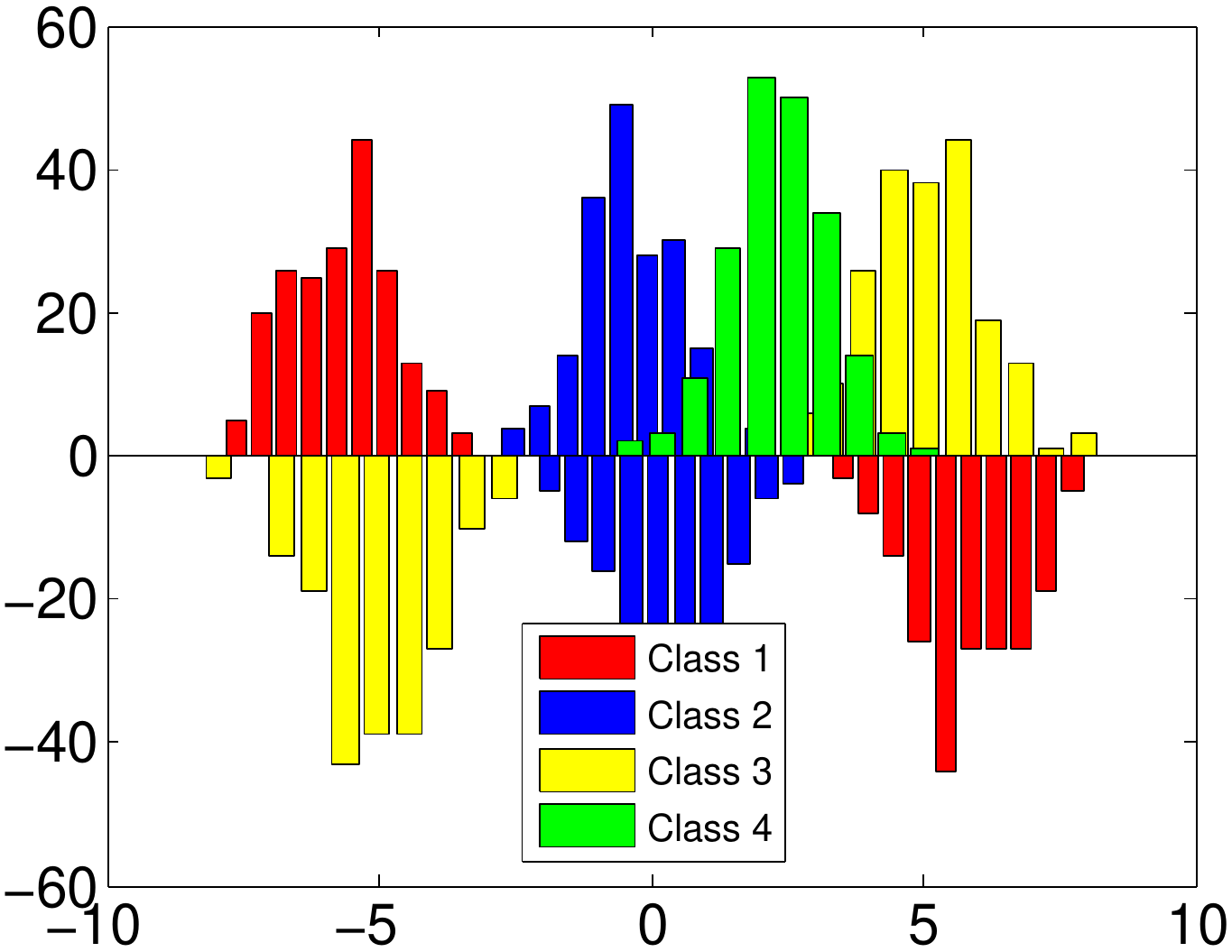}}		
	\subfigure[SWRLDA]{\label{3h}		
		\includegraphics[width=0.23\linewidth]{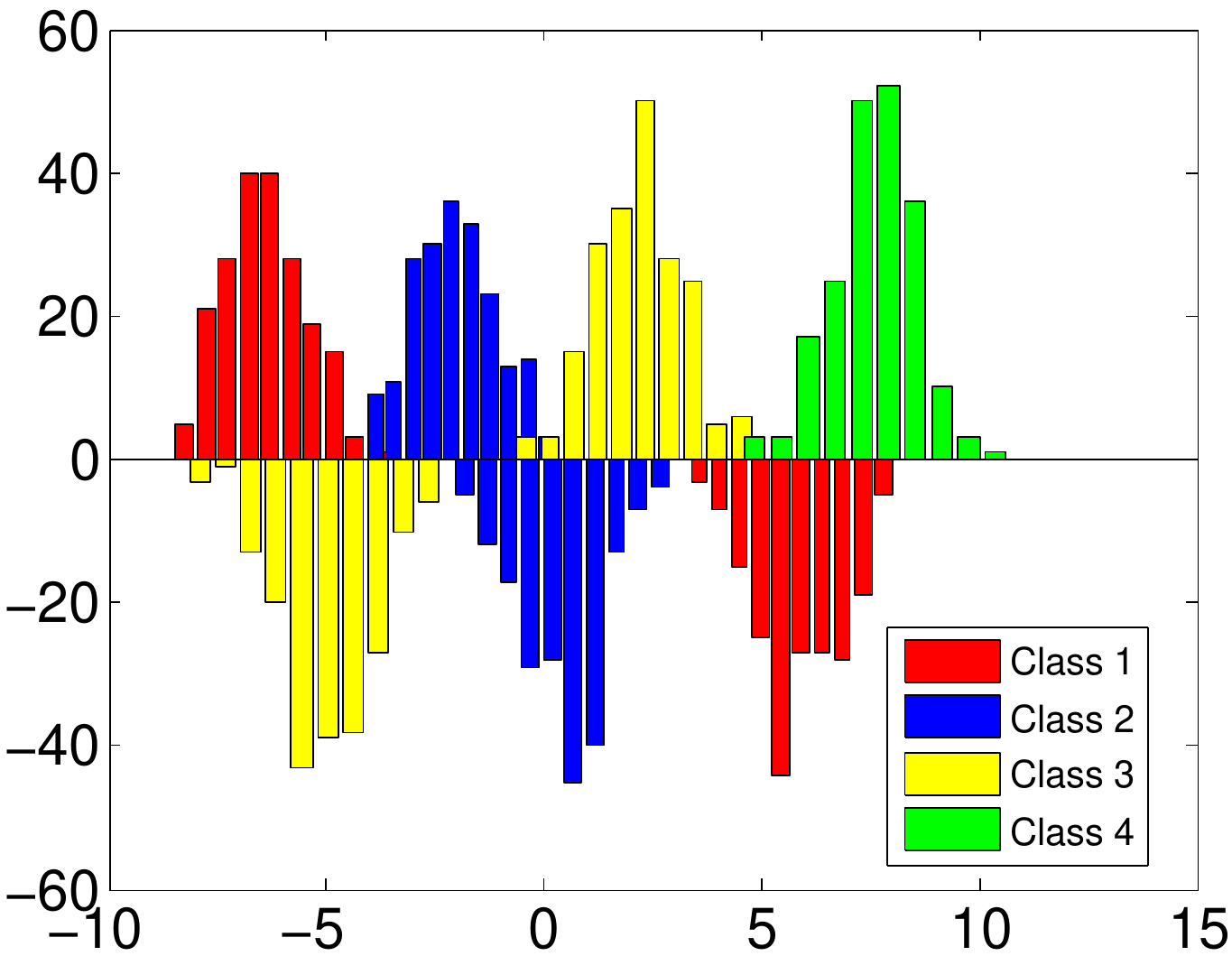}}
	\caption{The histograms of two synthetic datasets projected onto the optimal direction  learned by different methods. Specifically, the histogram below and above the zero ordinate corresponds to \emph{syn1} and \emph{syn2} respectively.}	
	\label{synthetic1}	
\end{figure*}
Considering that the projection directions may be different in the presence and absence of edge class, we design an experiment to evaluate each method on two $2$D synthetic datasets, denoted as \emph{syn1} and \emph{syn2} respectively. The first dataset \emph{syn1} contains three classes, \ie, class $1$, $2$ and $3$, which are located very closely to each other. The second dataset \emph{syn2} is constructed by adding an extra class far away from class $1$, $2$ and $3$, denoted as class $4$, which can be regarded as an edge class. Each class is composed of $200$ data points generated from Gaussian distribution with the same standard deviation [$1,0; 0,1$]. The data means corresponding to class $1$, $2$, $3$ and $4$ are [$-5,-4$], [$-3,1$], [$-1,6$] and [$10,-2$] respectively. The scatter plots of \emph{syn1} and \emph{syn2} are illustrated in \Cref{2a} and \Cref{2b}.

\subsubsection{Performance Comparison}
\Cref{synthetic} demonstrates the projection directions learned by different methods over \emph{syn1} and \emph{syn2}. Under each direction, the original data points can be projected into a new embedding space as shown in \Cref{3a} to \Cref{3h}, corresponding to different methods respectively. For \emph{syn1} without edge class, each comparative method can learn a satisfying projection direction to classify them correctly. However, the four classes of \emph{syn2} with edge class overlap with each other in varying degrees under the projection directions learned by most comparisons. According to \Cref{2b}, all the methods can be arranged by sorting their slopes in descending order as MMDA, aPAC, SWRLDA, GMSS, HMSS, RSLDA, LDA and STRDA. 
Based on their sensitivity to edge class, we roughly divide them into three categories, \ie, over-sensitive, middle and under-sensitive. From the projected histograms of \emph{syn2}, we can observe that all of LDA, RSLDA, HMSS, GMSS and STRDA fall into the over-sensitive category. Consequently, class $1$, $2$ and $3$ overlap with each other in varying degrees and class $4$ is far away from them under the directions learned by these methods. This phenomenon indicates that the existence of edge class can indeed dominate the projection direction learned by most LDA methods. The classification of the edge and non-edge class is over-emphasized while the distinction between multiple non-edge classes is neglected, which fails to obtain desirable classification results of the whole data. On the contrary, MMDA belongs to the under-sensitive category, since it only maximizes the class pair with minimum distance while completely ignoring the edge class. Although class $1$, $2$ and $3$ can be separated well, class $4$ suffers from a large overlap with others. Compared with these methods, aPAC and SWRLDA can much better results on the dataset with edge class. 
As illustrated in Figure 3(c) and 3(h), class 3 and 4 still suffer from a relatively large overlap on the projection direction learned by aPAC, while SWRLDA can find a projection that provides a better overall trade-off between all the classes. This phenomenon indicates the superiority of our method in classification and discrimination between classes, even though the training set contains edge class.

\subsection{Experiments on Real-World Datasets}
\begin{table}[t]\small
	\renewcommand{\arraystretch}{1.2}
	\tabcolsep= 6pt	
	\begin{center}
		\caption{Summary of the benchmark datasets used for classification in our experiments.}\label{table:dataset}
		\begin{tabular}{lcccc}
			\toprule[1pt]
			Dataset & Samples($n$) & Dimension($d$) & Classes($c$)& Type\\
			\hline
			COIL20 & 1440 & 1024 & 20 &Image\\
			COIL100& 7200 & 1024 & 100&Image\\
			Isolet & 7797 & 617 & 26 &Speech\\
			Wine& 178 & 13 & 3&UCI\\
			CMU PIE& 11554 & 1024 & 68&Image\\
			YaleB& 2414 & 1024 & 38&Image\\
			20Newsgroups& 18846 & 26214 & 20&Document\\
			\bottomrule[1pt]
		\end{tabular}	
	\end{center}
\end{table}
We adopt seven real-world datasets widely used in the previous research to evaluate the performance of the proposed SWRLDA, including COIL20 \cite{nene1996columbia}, COIL100 \cite{nayar1996columbia}, Isolet \cite{fanty1991spoken}, 20Newsgroups \cite{lang1995newsweeder}, Wine \cite{blake1998uci} \cite{debole2005analysis}, CMU PIE \cite{sim2002cmu} and YaleB \cite{georghiades2001few}. Among them, \textbf{COIL100} is an image library of $100$ objects viewed from varying angles and \textbf{COIL20} is a subset of COIL100 with $20$ classes. \textbf{Isolet} is a spoken letter database with $150$ subjects who spoke the name of each letter of the alphabet twice. 
 \textbf{20Newsgroups} is a large-scale dataset collected and originally used for document classification by Lang \cite{lang1995newsweeder} with $18,846$ documents, evenly distributed across $20$ classes.
The \textbf{CMU PIE} contains face images under $5$ poses(C$05$, C$07$, C$09$, C$27$, C$29$), $43$ illumination conditions, and $4$ expressions for each person.
The \textbf{YaleB} consists of $2414$ frontal-face images of $38$ individuals captured under various lighting conditions. \textbf{Wine} is a dataset from UCI machine learning repository \cite{blake1998uci}, which consists of 178 samples with 13 attributes from 3 classes.
The statistics for each dataset are summarized in \Cref{table:dataset}.

\subsubsection{Experimental Setup}
According to the standard $5$-fold cross validation, each dataset is randomly partitioned into $5$ sets with $4$ parts for training and the remaining part for testing in each round. During the training period, the optimal projection matrix can be obtained when the relative change of objective function is below $10^{-6}$. 
Once the optimal projection matrix is obtained, we use $k$-Nearest Neighbor ($k$-NN) method ($k = 1$ is used in this work) to classify the data points in the projected space. In $k$-NN, we use the most widely used Euclidean distance as the distance metric.
For a comprehensive comparison, the following five performance evaluation metrics are considered to evaluate the quality of different LDA methods for classification: (a) average classification accuracy with standard deviation; (b) minimum pairwise distance in the projected low-dimensional subspace; (c) two-dimensional data visualization; (d) average running time; (e) robustness of models.

\subsubsection{Classification Accuracy}
\begin{table*}[t]\small
	\renewcommand{\arraystretch}{1.2}
	\tabcolsep=10 pt	
	\begin{center}
		\caption{Performance comparison in terms of ACC $\pm$ std\% for different methods over six benchmark datasets.}\label{table:accuracy}
		\begin{tabular}{lcccccc}
			\toprule[1pt]
			Method & COIL20 & COIL100 & Isolet & 20Newsgroups & CMU PIE&YaleB\\
			\hline
			LDA&93.19$\pm$0.121&87.62$\pm$0.085 &93.91$\pm$0.066 &- &93.44$\pm$0.052&99.13$\pm$0.027  \\
			RSLDA&97.35$\pm$0.167&89.08$\pm$0.049 &85.62$\pm$0.025 &- &94.89$\pm$0.040&95.38$\pm$0.034\\
			aPAC &94.93$\pm$0.176&88.36$\pm$0.061 &93.63$\pm$0.085 &- &94.34$\pm$0.039&99.58$\pm$0.035\\
			GMSS &95.14$\pm$0.097&87.76$\pm$0.099 &93.76$\pm$0.032 &- &94.72$\pm$0.036&99.63$\pm$0.051\\
			HMSS &95.35$\pm$0.089&87.95$\pm$0.143&93.49$\pm$0.094 &- &94.72$\pm$0.030&99.12$\pm$0.020\\
			STRDA&90.69$\pm$0.108&87.45$\pm$0.065 &92.68$\pm$0.070 &- &90.37$\pm$0.078&99.46$\pm$0.028\\
			MMDA&95.69$\pm$0.061&88.36$\pm$0.101&94.07$\pm$0.029 &- &94.95$\pm$0.064&99.21$\pm$0.079\\
			\hline
			SWRLDA&\textbf{99.65}$\pm$0.042&\textbf{95.01}$\pm$0.045 &\textbf{96.64}$\pm$0.033 &\textbf{89.81}$\pm$0.058 &\textbf{97.02}$\pm$0.041&\textbf{99.75}$\pm$0.027\\	
			\bottomrule[1pt]				
		\end{tabular}	
	\end{center}
\end{table*}
\begin{figure*}[t] 	
	\centering 	
	\subfigure[COIL20]{\label{4a}		
		\includegraphics[width=0.32\linewidth]{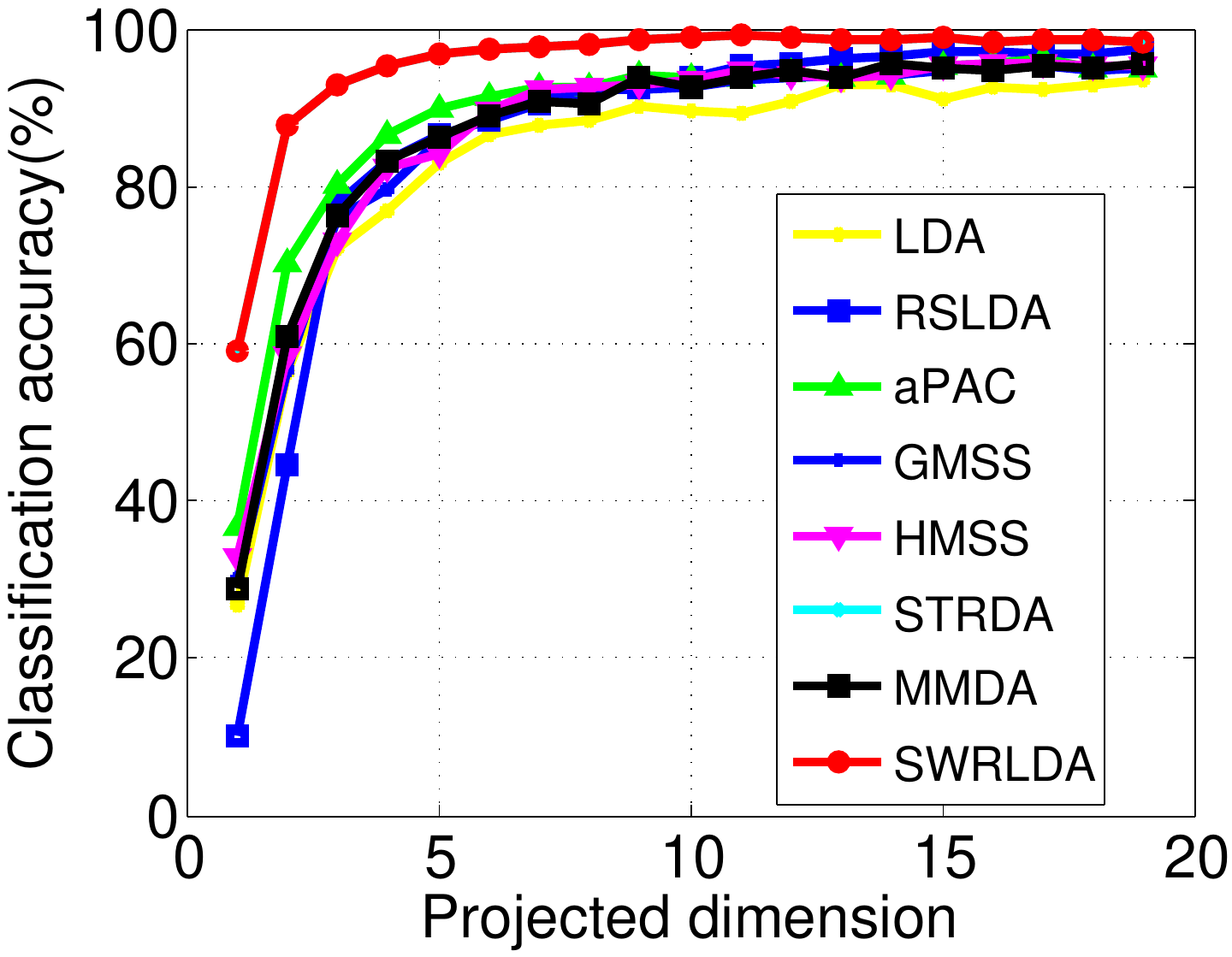}}
	\subfigure[COIL100]{\label{4b}				
		\includegraphics[width=0.32\linewidth]{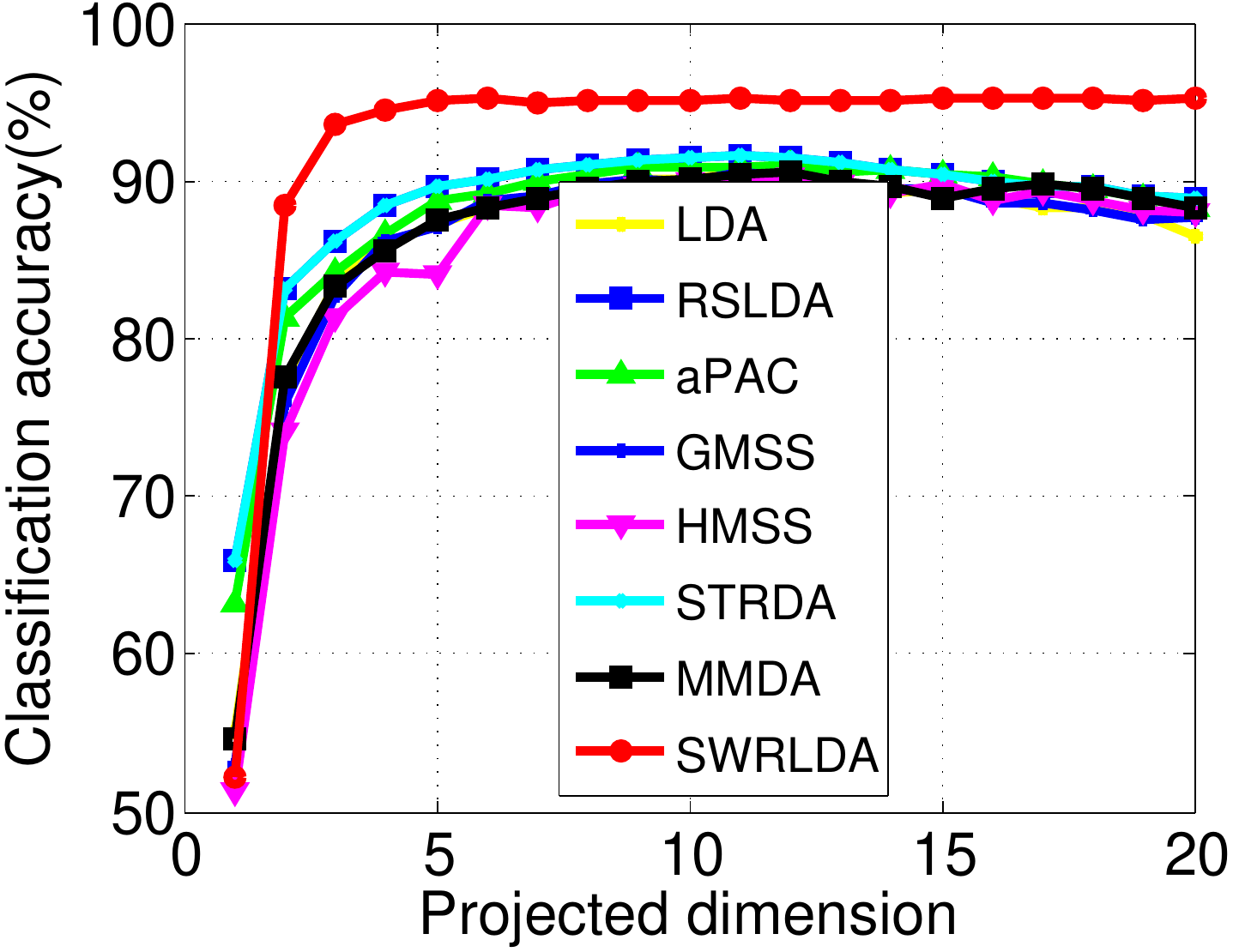}}
	\subfigure[Isolet]{\label{4c}		
		\includegraphics[width=0.32\linewidth]{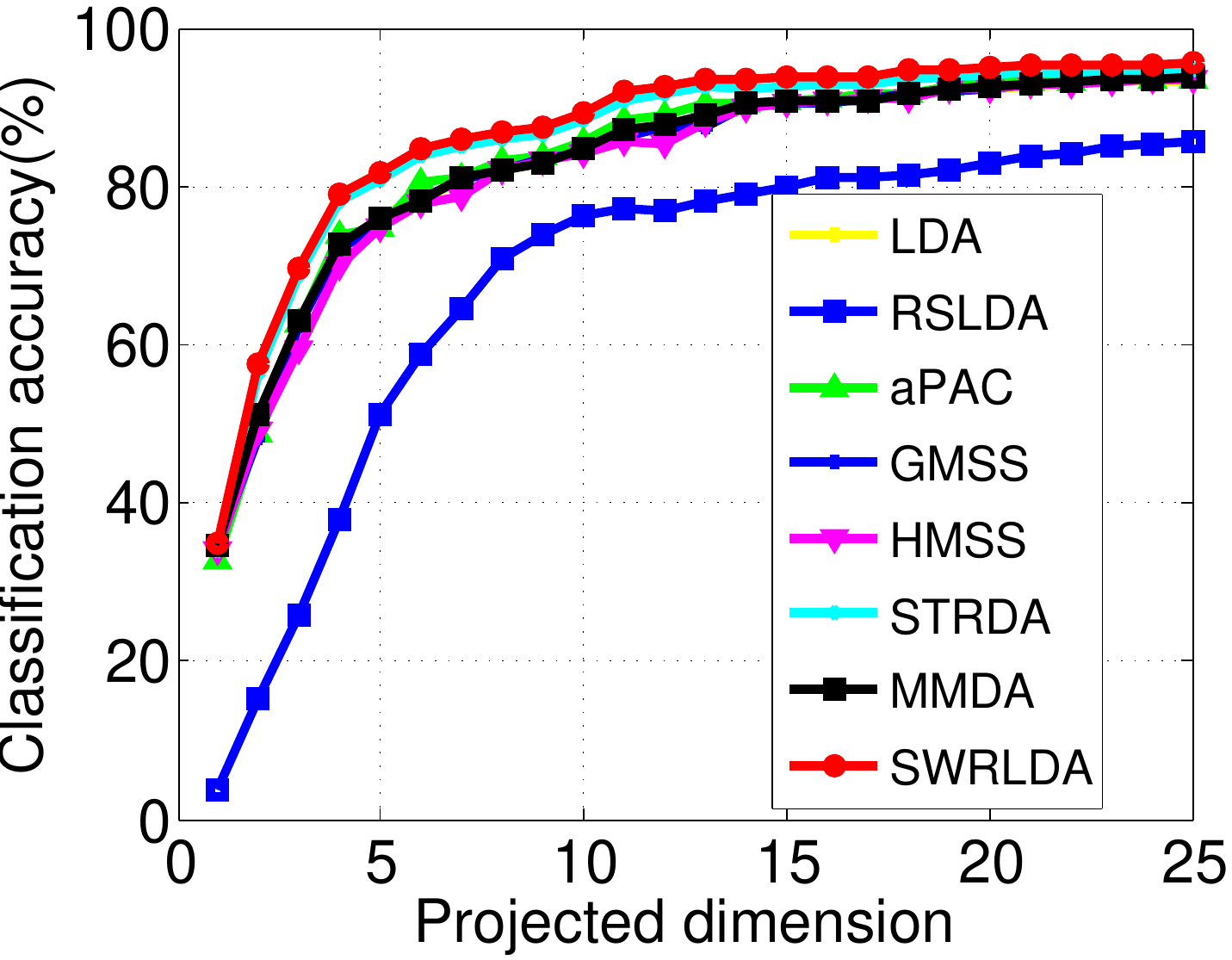}}	
	\caption{Classification accuracy with respect to different projected dimensions achieved by different methods.}	
	\label{dimension}	
\end{figure*}
The classification accuracy and standard deviation are the most widely used evaluation metrics for discriminant dimension reduction methods.
For fair comparison, the mean classification accuracy (ACC) and standard deviation (std) of $5$-fold cross validation with the same projection dimension $c-1$ corresponding to each method are reported in Table \ref{table:accuracy}. Notably, we only report the results on 20Newsgroups achieved by our method, because all the comparison methods fail to deal with this large-scale dataset with affordable cost ($>$24h).
\Cref{dimension} demonstrates the classification accuracy with respect to various reduced dimensions over three datasets.
From these results, we can make the following observations:
\begin{enumerate}[$\bullet$,leftmargin=*,topsep=0pt,noitemsep]
	\item The proposed SWRLDA achieves the highest classification accuracy among all the comparisons over all the datasets. Specifically, the classification accuracy improves $2.4\%$ on COIL20, $6.7\%$ on COIL100, $2.7\%$ on Isolet, $2.2\%$ on CMU PIE and $0.1\%$ on YaleB respectively. We can attribute this improvement to the combination of ``avoiding optimal mean'' strategy and self-weighted between-class distance criterion based on $\ell_{2,1}$-norm. 
	\item As shown in \Cref{dimension}, SWRLDA obtains the best performance in most dimensions over all the datasets. With the decrease of projection dimension $m$, the classification accuracy of comparative methods decreases significantly, while the proposed approach is more stable except when $m$ is too small (\eg, $m=1$). This suggests that SWRLDA is more robust to the selection of projected feature dimensions. 
\end{enumerate} 

\subsubsection{Average Minimum Pairwise Distance}
\begin{figure}[t] 	
	\centering 	
	\subfigure[COIL20]{\label{5a}		
		\includegraphics[width=0.36\linewidth]{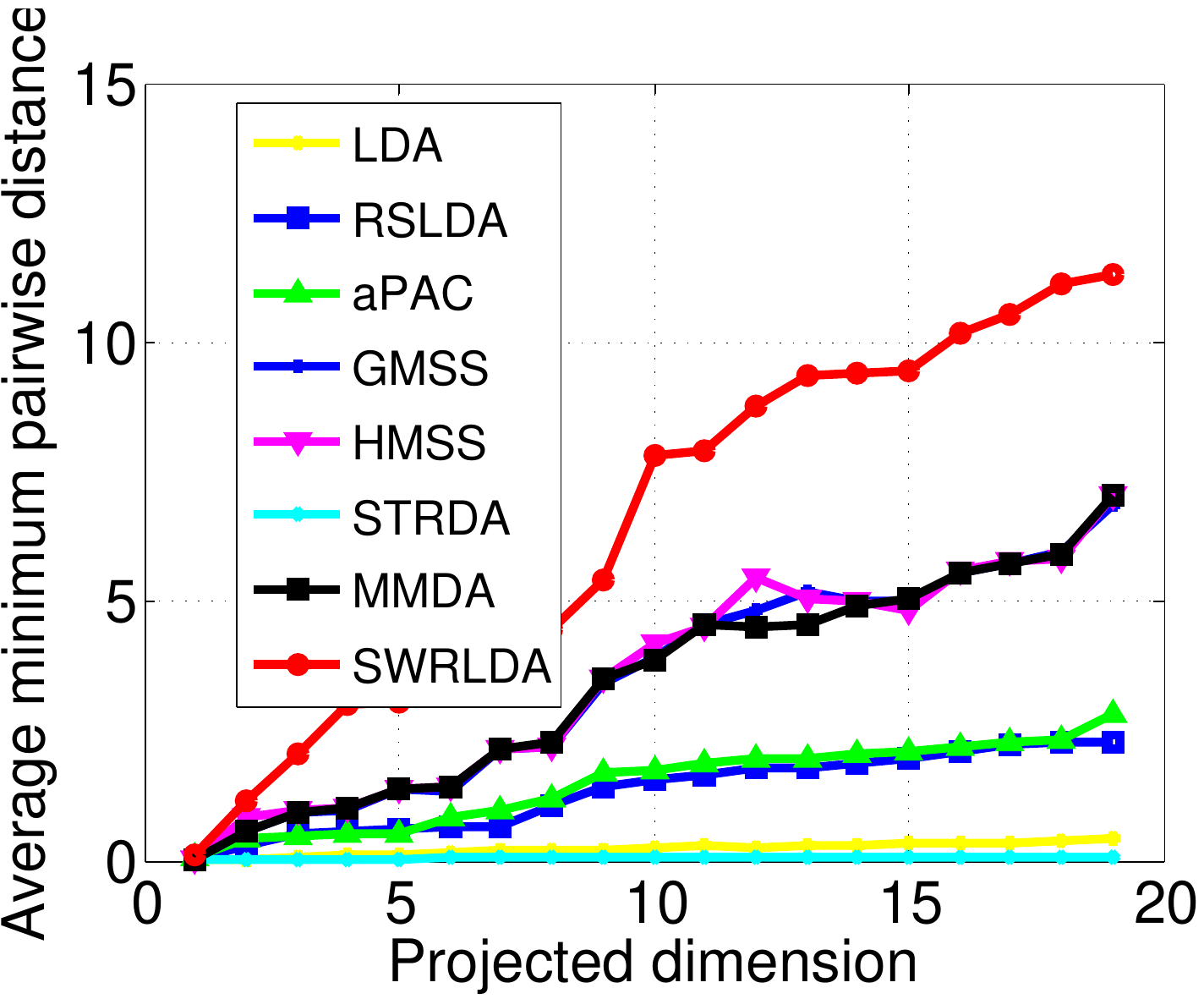}}
	\subfigure[Isolet]{\label{5b}				
		\includegraphics[width=0.35\linewidth]{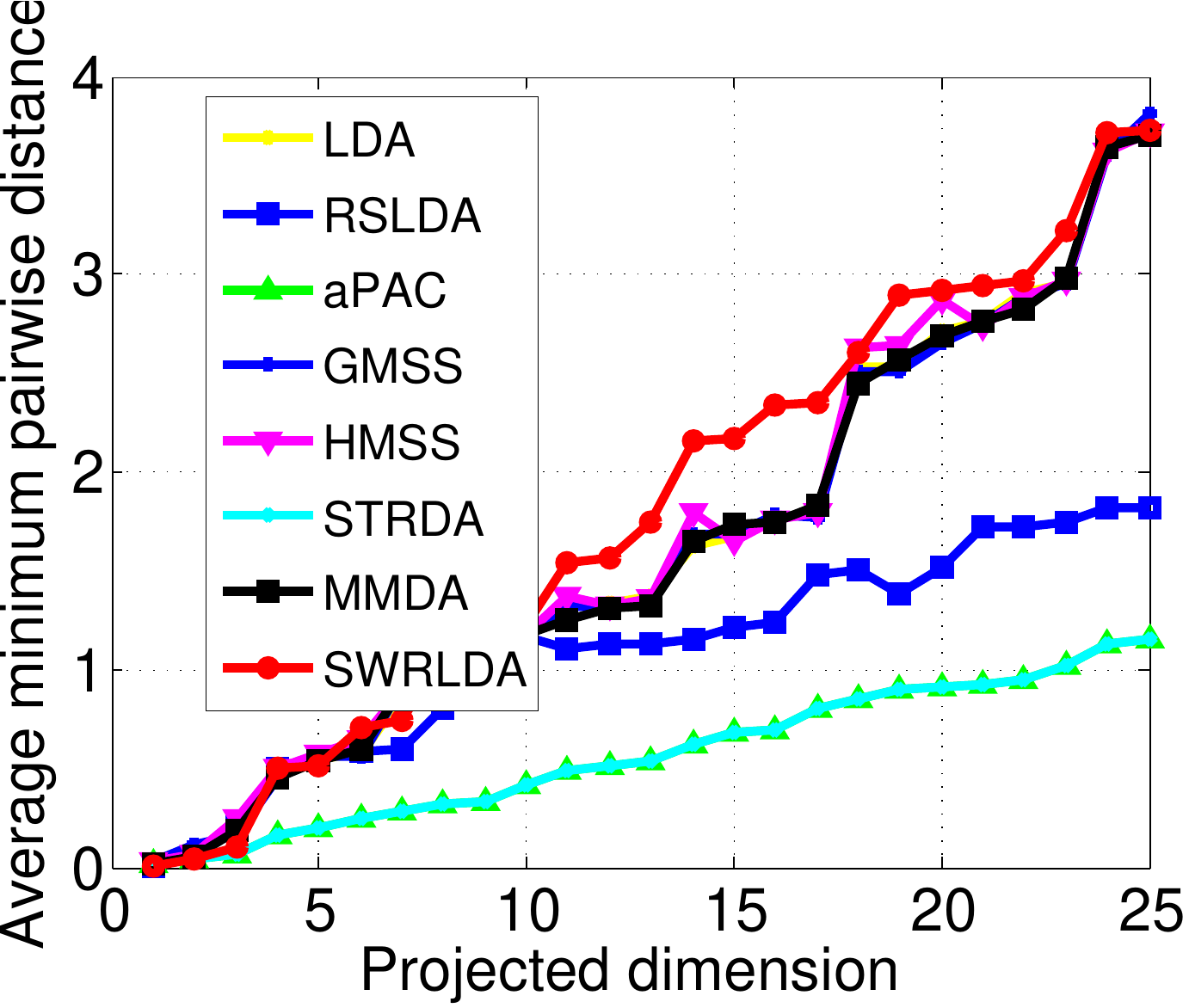}}	
	\caption{Average minimum pairwise distance with respect to different projected dimensions.}	
	\label{distance}	
\end{figure}
In general, a larger minimum pairwise distance indicates a better separation of all class pairs in the low-dimensional subspace. For each projected dimension, we calculate the minimum pairwise distance in the projected subspace by averaging $100$ independent runnings. 
The experimental results on COIL20 and Isolet dataset are demonstrated in \Cref{5a} and \Cref{5b}, respectively. As shown in \Cref{5a} and \Cref{5b}, the average minimum pairwise distance of each method increases continuously with the increasing of projected dimensions. Specifically, when projecting to $1$-dimensional subspace, almost all the compared algorithms obtain the same result. However, as the dimension increases, the proposed algorithm gains more significant performance improvements compared with other methods.
\subsubsection{Two-Dimensional Visualization of Wine Dataset}
For qualitative analysis, the two-dimensional scatter plot in the embedding subspace is exploited to visualize the separability of each model over Wine dataset. Specifically, a two-dimensional subspace is first learned with all the samples of Wine dataset, and then all of them are projected onto the two-dimensional subspace. From the graphical representations shown in \Cref{wine}, we can see that LDA, RSLDA and aPAC achieve the worst separation results, since all the three classes overlap with each other seriously.
Compared with them, the samples are almost separable in the subspace generated by GMSS, HMSS, STRDA, MMDA and SWRLDA. While there are still slight overlaps among the results of these methods, the proposed SWRLDA is able to learn a separation of no overlapped samples with the largest minimum between-class distance. This indicates that SWRLDA can still make a good trade-off among all the classes by learning appropriate weights even though the training data do not contain edge class.
\begin{figure*}[t] 	
	\centering 	
	\subfigure[LDA]{\label{6a}		
		\includegraphics[width=0.233\linewidth]{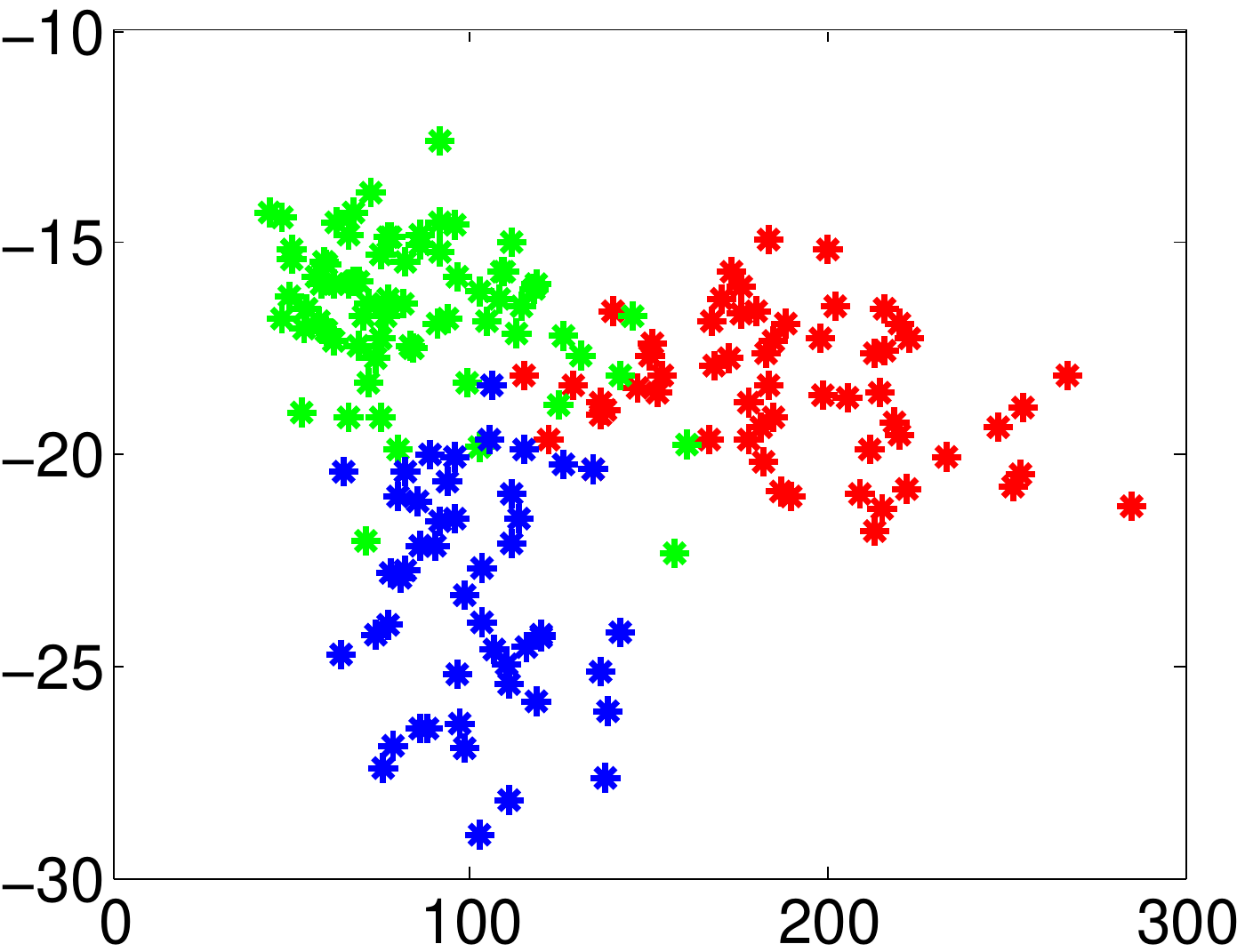}}
	\subfigure[RSLDA]{\label{6b}				
		\includegraphics[width=0.23\linewidth]{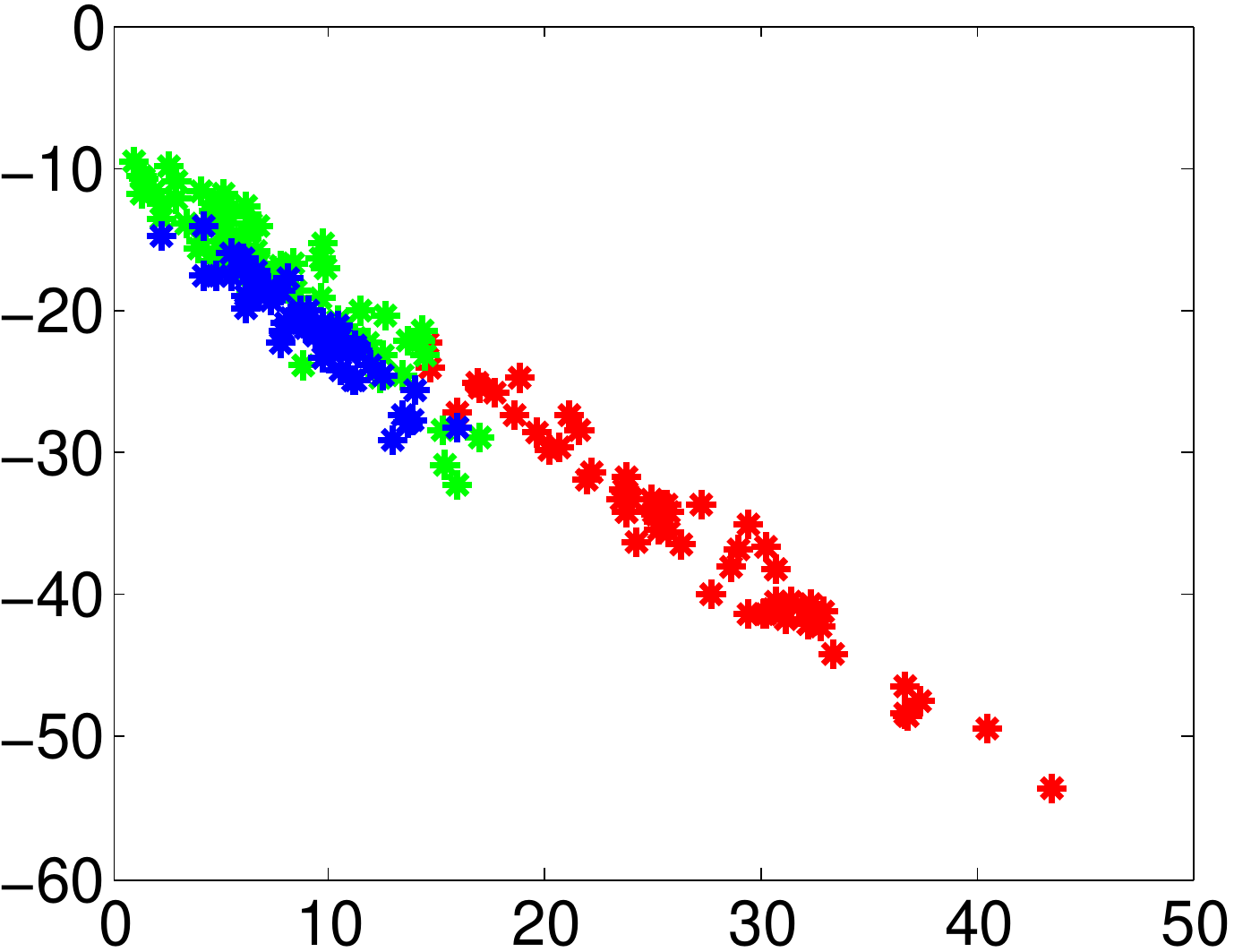}}	
	\subfigure[aPAC]{\label{6c}		
		\includegraphics[width=0.227\linewidth]{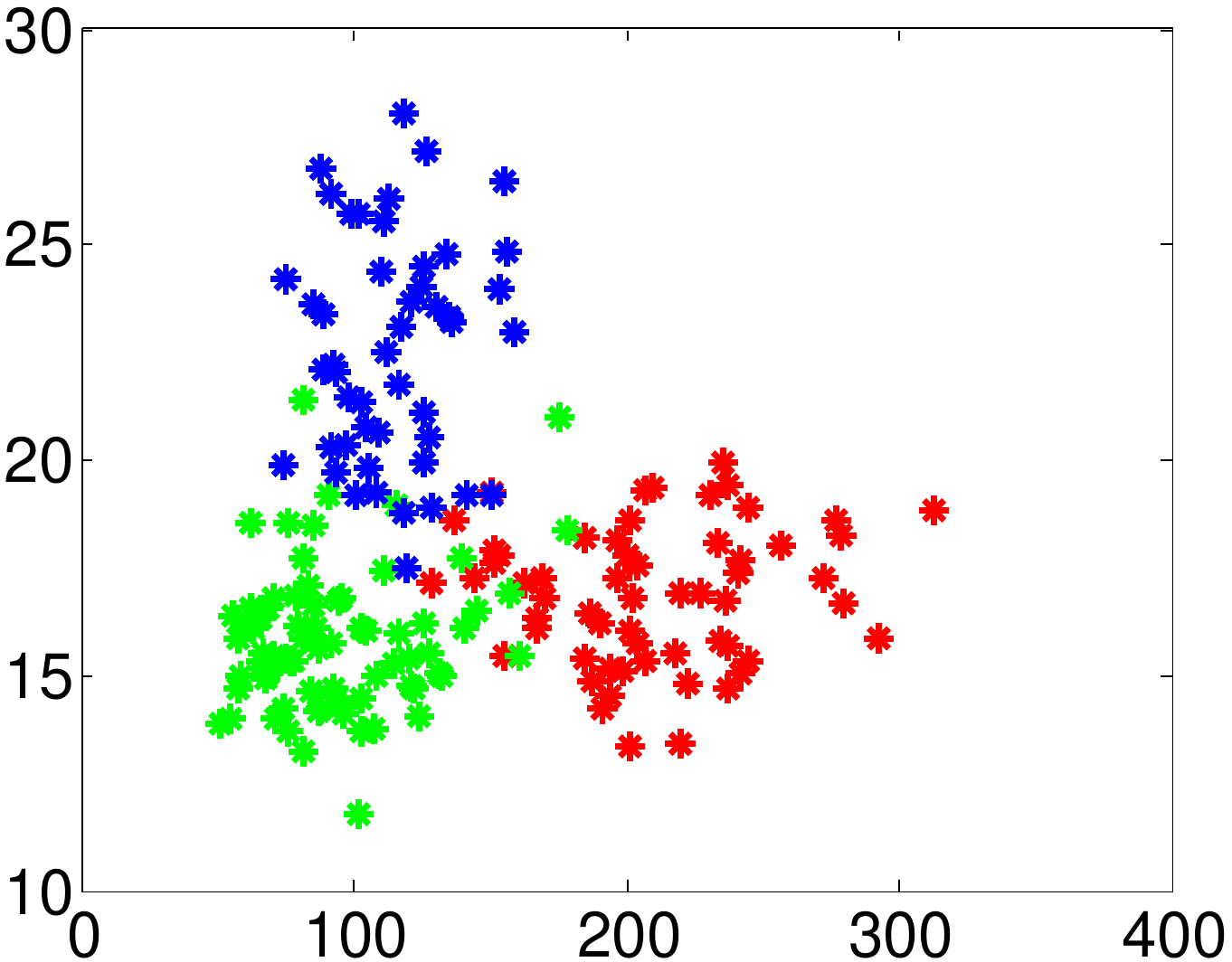}}		
	\subfigure[GMSS]{\label{6d}		
		\includegraphics[width=0.225\linewidth]{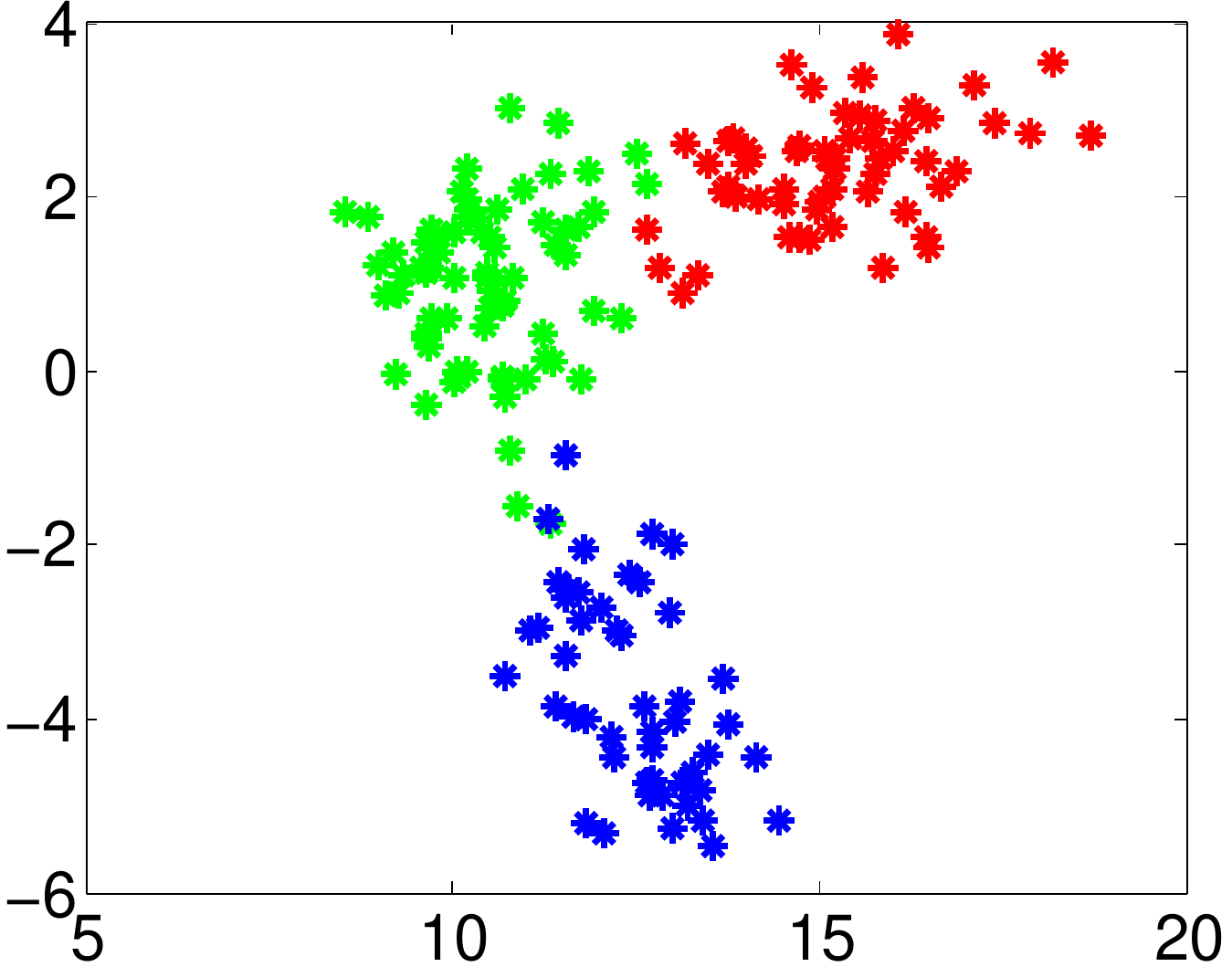}}
	\vfill	
	\subfigure[HMSS]{\label{6e}		
		\includegraphics[width=0.227\linewidth]{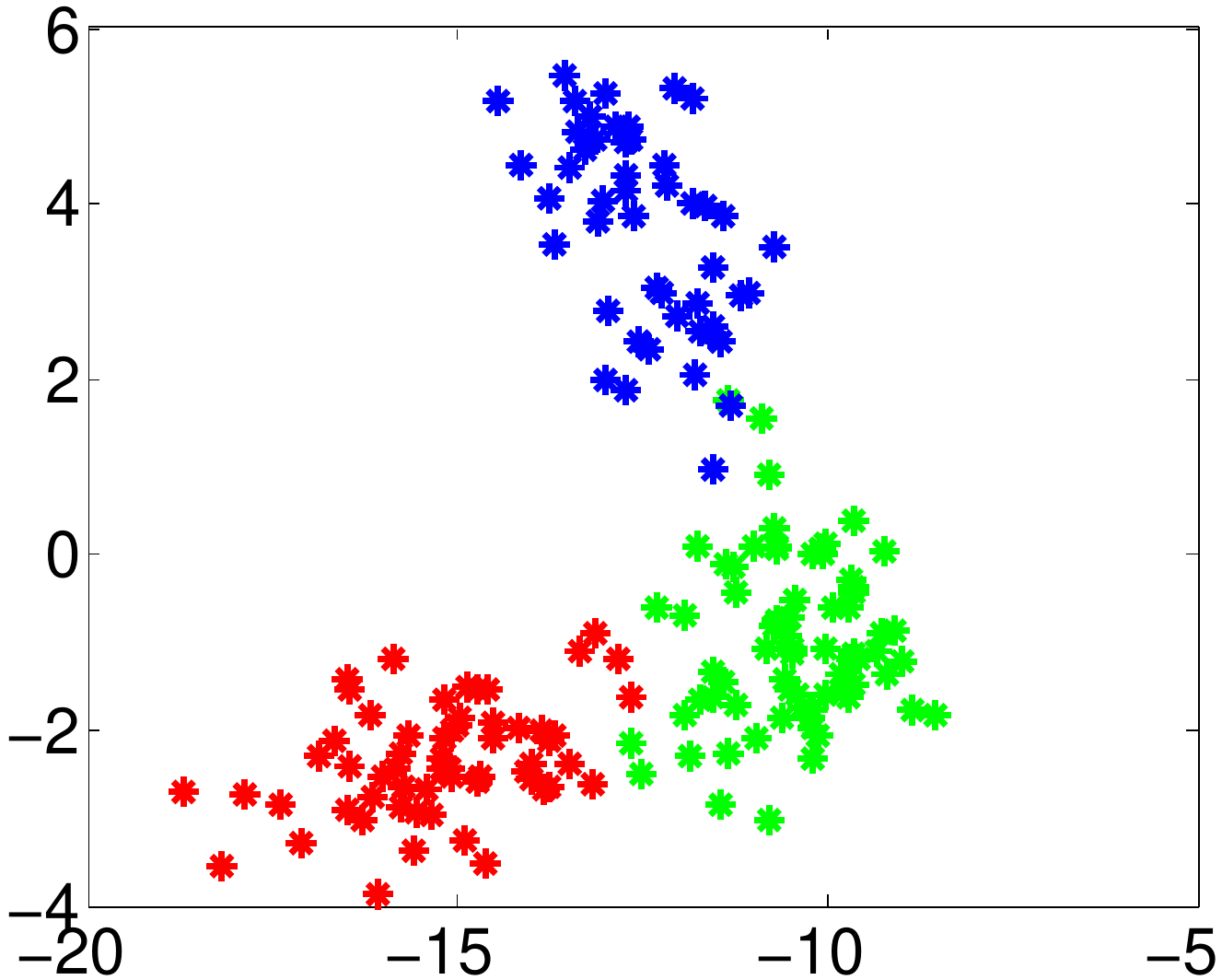}}
	\subfigure[STRDA]{\label{6f}				
		\includegraphics[width=0.23\linewidth]{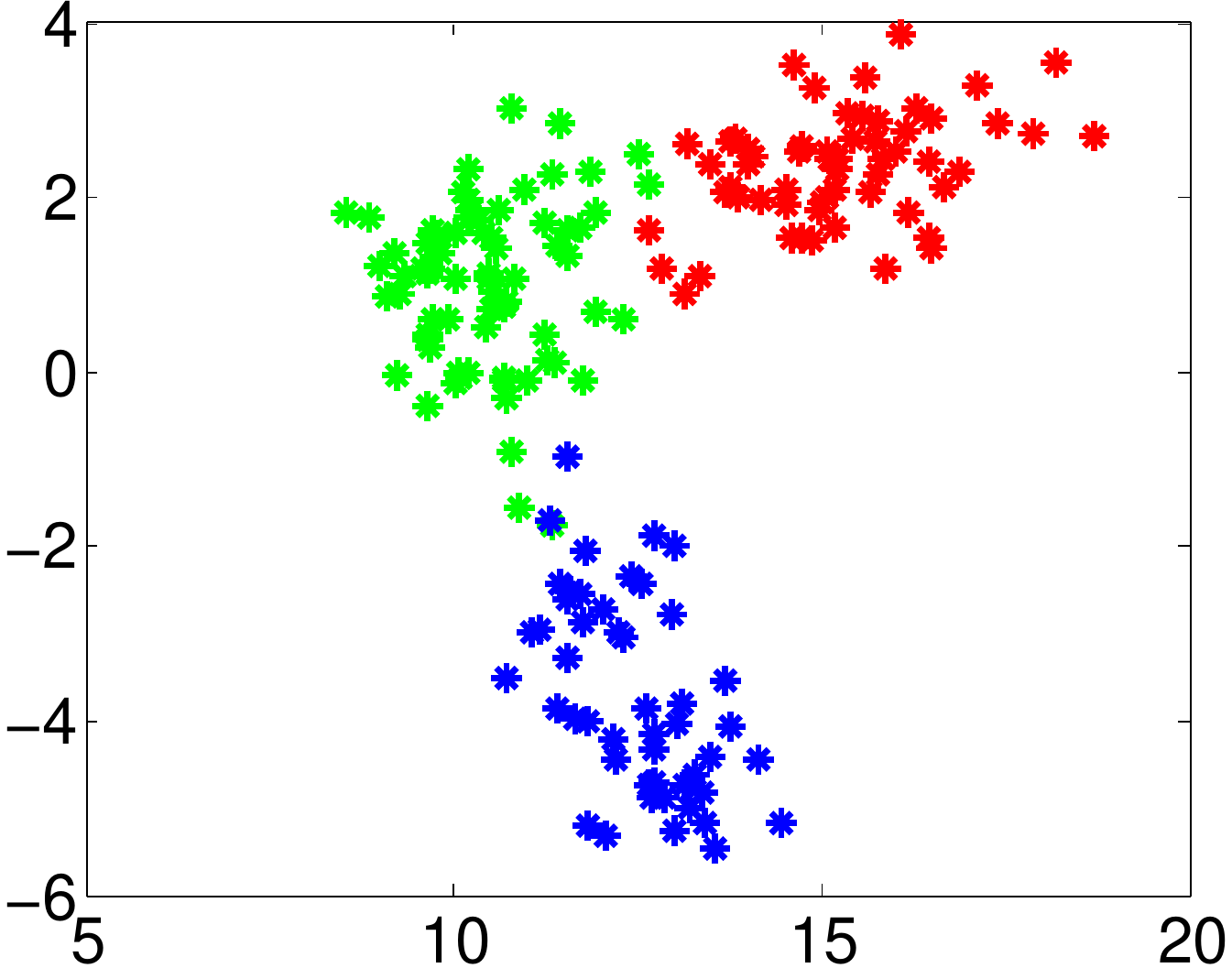}}	
	\subfigure[MMDA]{\label{6g}		
		\includegraphics[width=0.23\linewidth]{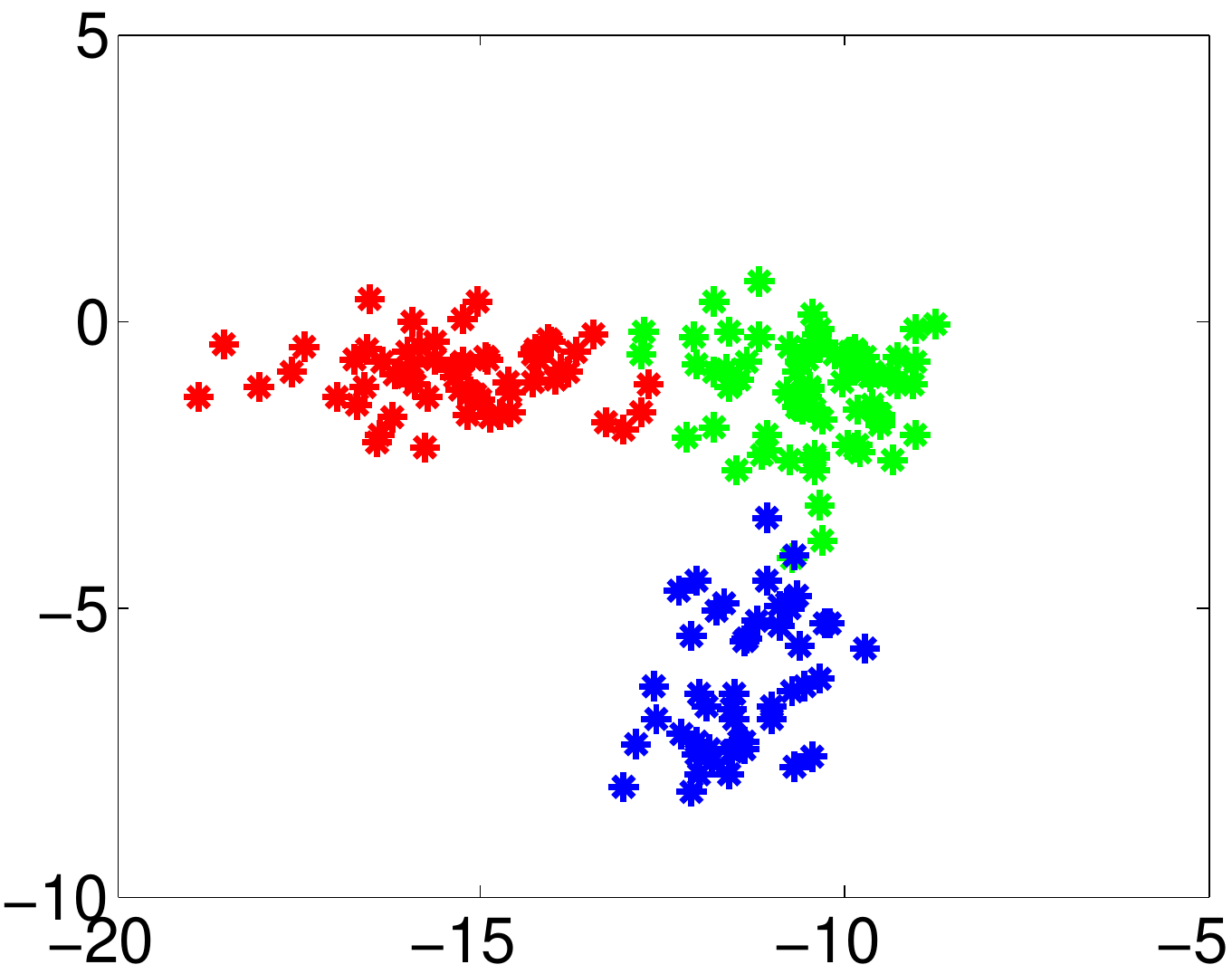}}		
	\subfigure[SWRLDA]{\label{6h}		
		\includegraphics[width=0.225\linewidth]{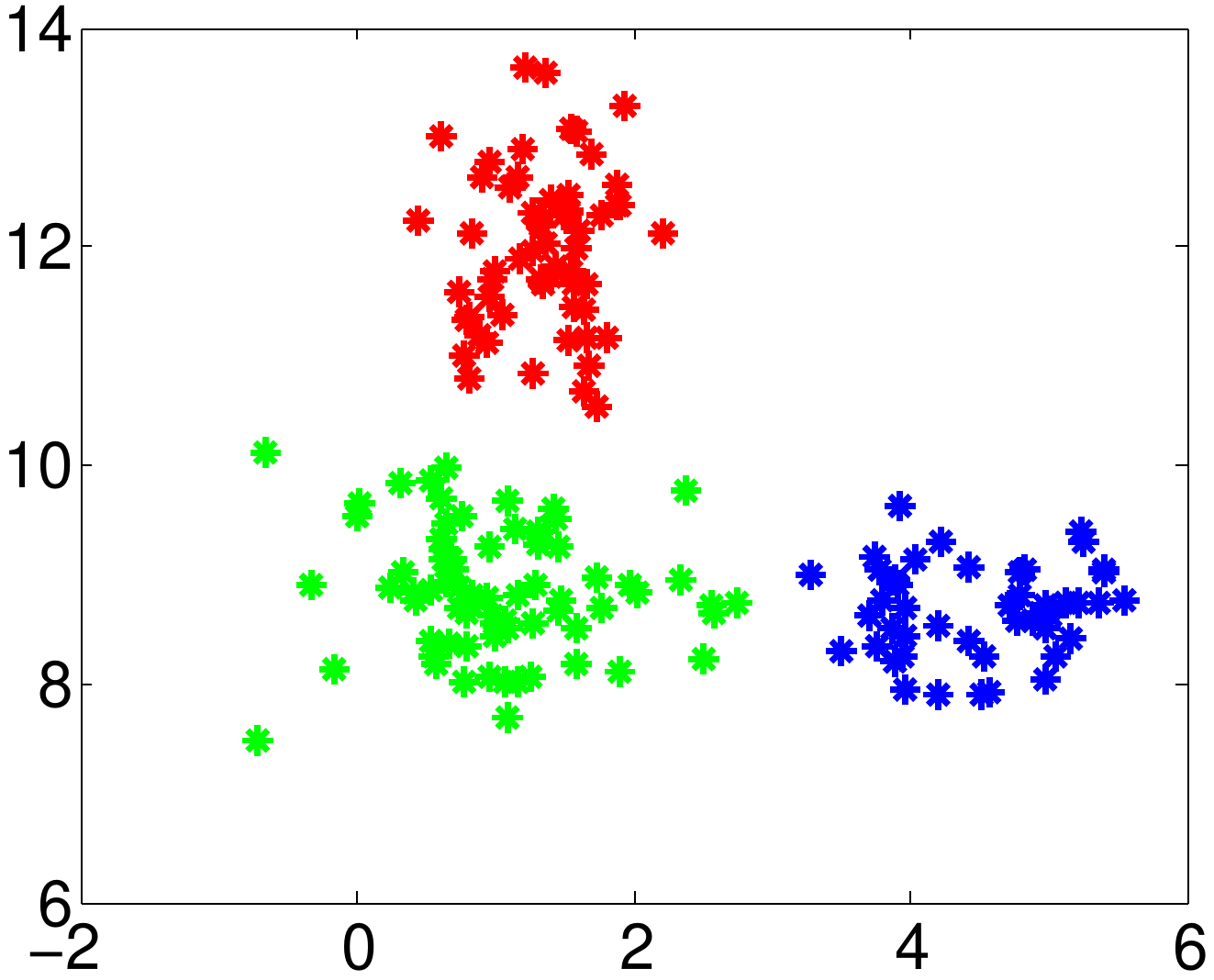}}
	\caption{Visualization of Wine dataset projected onto a two-dimensional subspace learned by different methods.}	
	\label{wine}	
\end{figure*}
\subsubsection{Average Running Time}
\begin{table}[t]\small
	\renewcommand{\arraystretch}{1.2}
	\tabcolsep=5 pt	
	\begin{center}
		\caption{Average running times with respect to different methods over six benchmark datasets.}\label{time}
		\begin{tabular}{llccccc}
			\toprule[1pt]
			Method & COIL20 & COIL100 & Isolet&20Newsgroups & PIE&YaleB\\
			\hline			
			LDA&18.95s&57.13s& 12.70s&$>$24h&12.18s& 5.94s \\
			RSLDA&10.61s&47.95s&22.43s&$>$24h&27.39s&6.81s\\
			aPAC &14.94s&54.11s&2.286s&$>$24h&7.95s&5.49s\\
			GMSS &20.31s&57.90s&11.97s&$>$24h&13.14s&6.38s\\
			HMSS &21.24s&60.30s&12.48s&$>$24h&13.25s&6.37s\\
			STRDA&29.47s&143.22s&30.94s&$>$24h&126.38s&106.26s\\
			MMDA&25.17s&90.94s&18.29s&$>$24h&71.26s&40.01s\\
			\hline
			SWRLDA&0.63s&2.11s&0.36s&51.25s&0.65s&0.37s\\	
			\bottomrule[1pt]				
		\end{tabular}	
	\end{center}
\end{table}
All the comparison methods need to solve a time-consuming eigen-decomposition problem in each iteration with time complexity $O(d^3)$. By contrast, the proposed SWRLDA naturally avoids the eigen-decomposition and only needs to conduct SVD with time complexity $O(dm^2+m^3)$. Due to the fact that the original feature dimensionality $d$ is much larger than the reduced dimensionality $m$, the proposed SWRLDA is guaranteed to be much more efficient than the comparison methods.
For fair comparison, the average running times of each method over each dataset with $c-1$ reduced dimension are demonstrated in \Cref{time}. From \Cref{time}, we can make the following two observations:
\begin{itemize}
\item For 20Newsgroups with a very large number of samples and features, all the comparison methods become infeasible due to the intractable computation complexity ($>24$ h). This is because they need to produce and store large matrices that cannot be fit into the memory when both sample number and feature dimension are large.
\item The proposed SWRLDA only takes several seconds to finish the run over the large-scale 20Newsgroups dataset, which presents its high efficiency for handling large-scale dimensionality reduction problems. For the other smaller datasets, the calculation time of SWRLDA is also noticeably shorter (15 times) than other algorithms.

\end{itemize}
\subsubsection{The Evaluation of Robustness}
\begin{figure*}[t] 	
	\centering 	
	\subfigure[]{\label{7a}		
		\includegraphics[width=0.23\linewidth]{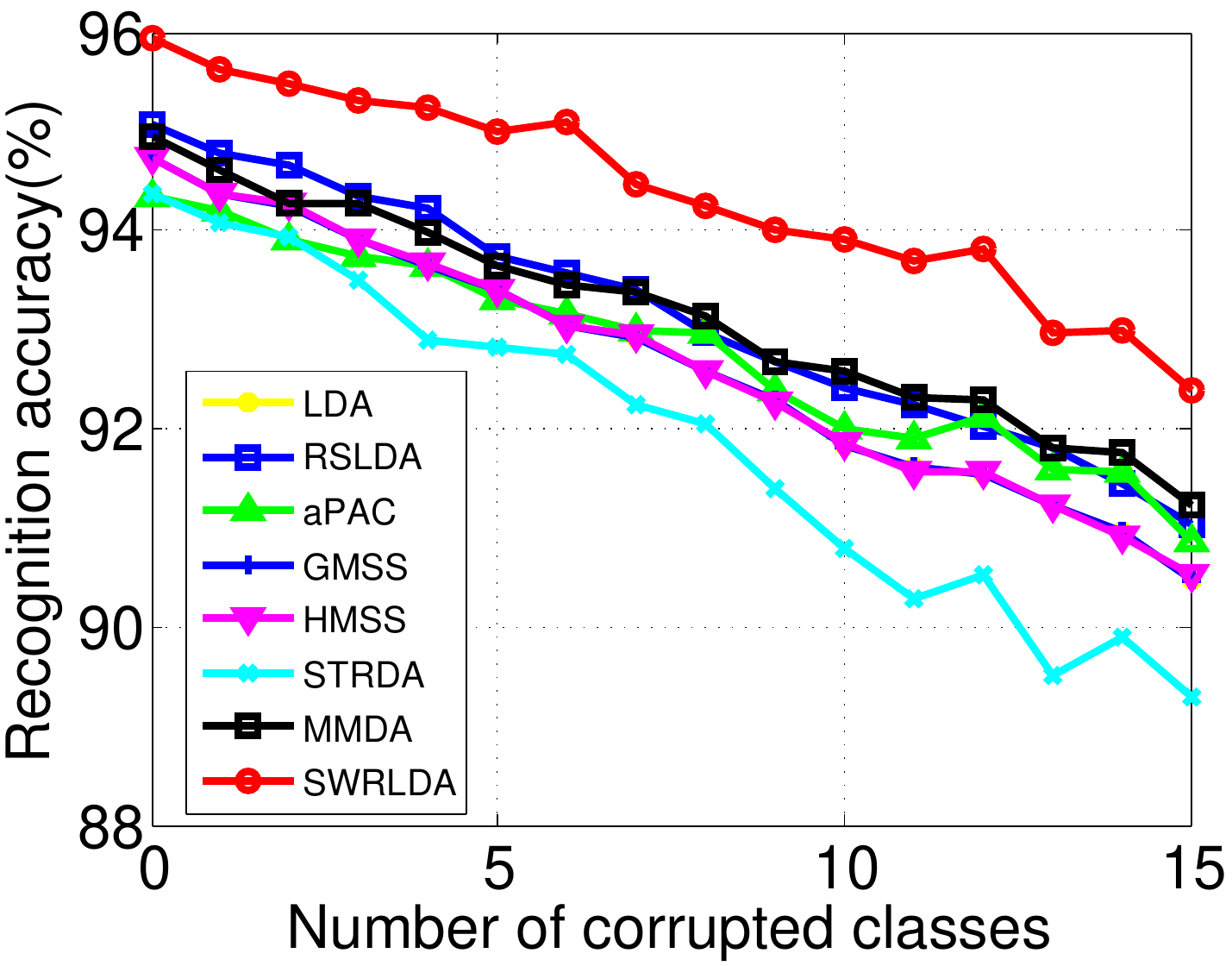}}
	\subfigure[]{\label{7b}				
		\includegraphics[width=0.23\linewidth]{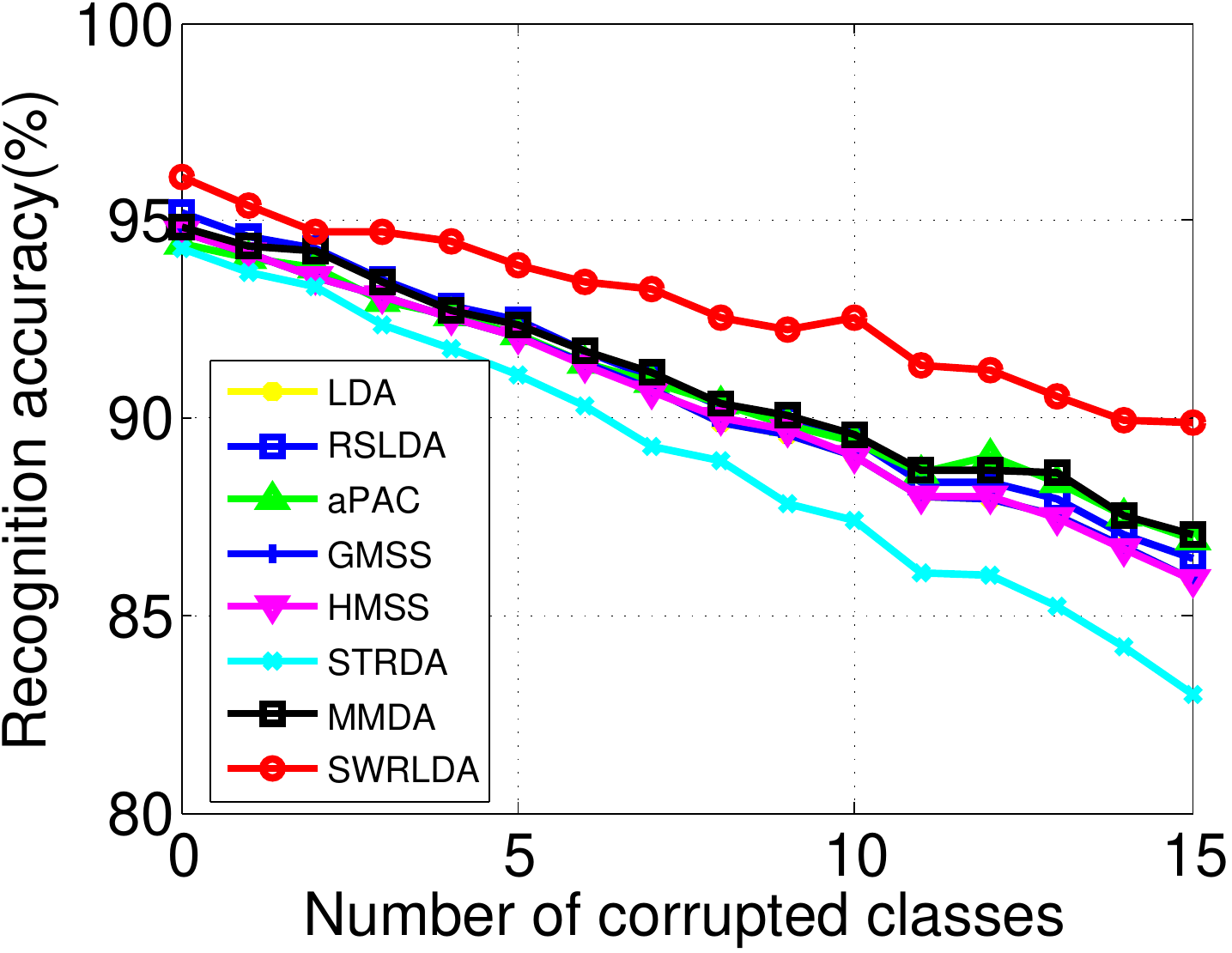}}
	\subfigure[]{\label{7c}		
		\includegraphics[width=0.23\linewidth]{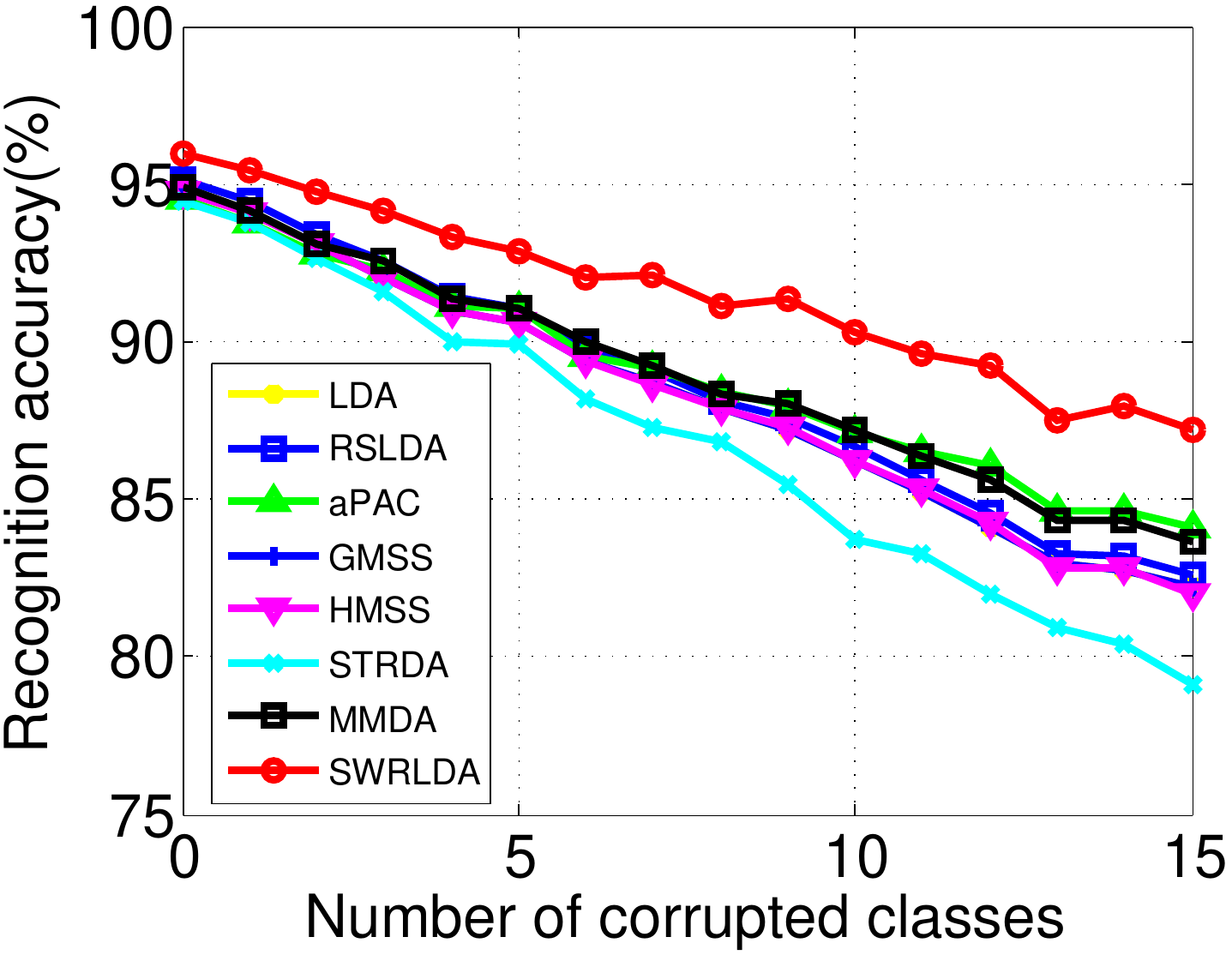}}		
	\subfigure[]{\label{7d}		
		\includegraphics[width=0.23\linewidth]{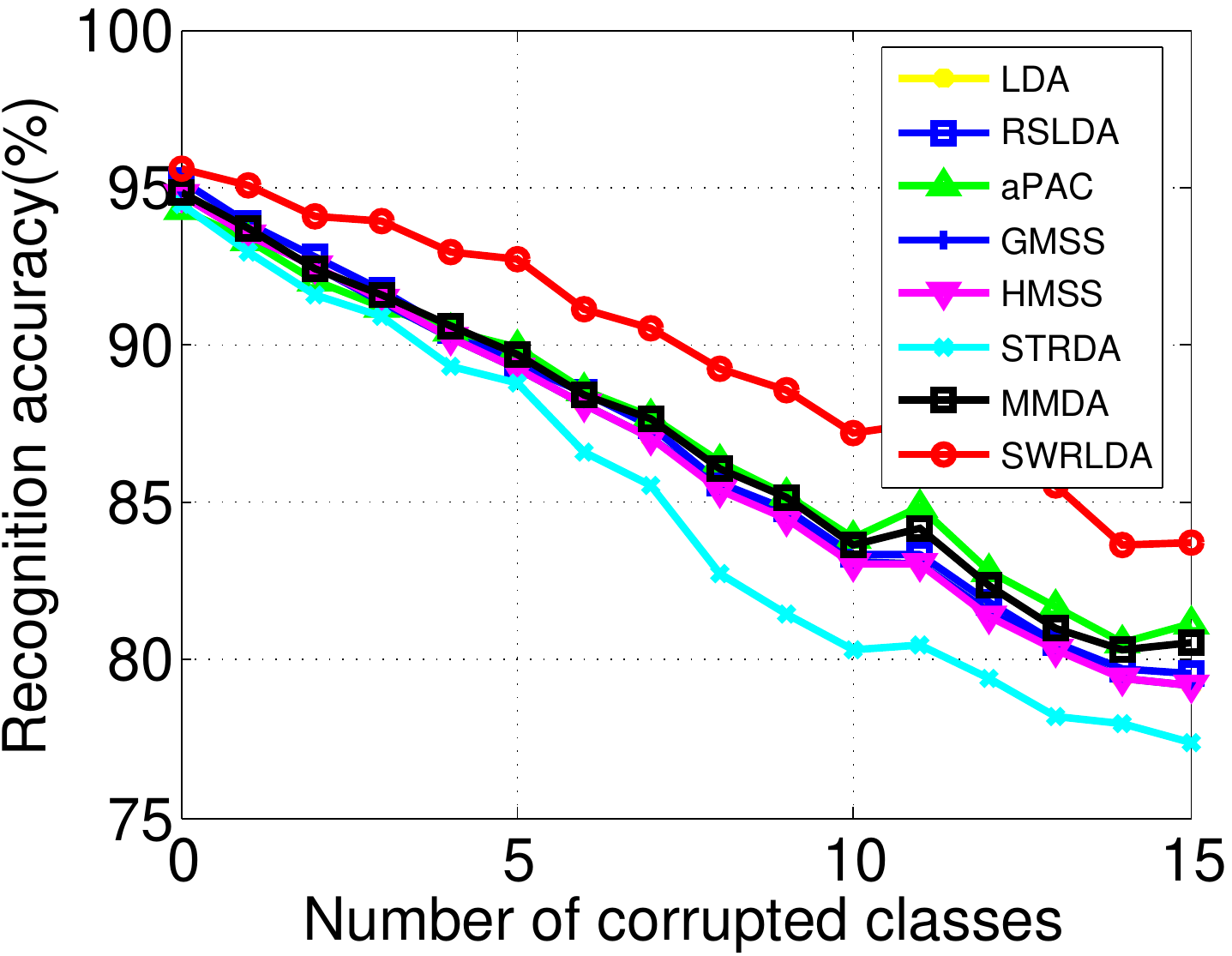}}
	\subfigure[]{\label{7e}		
		\includegraphics[width=0.23\linewidth]{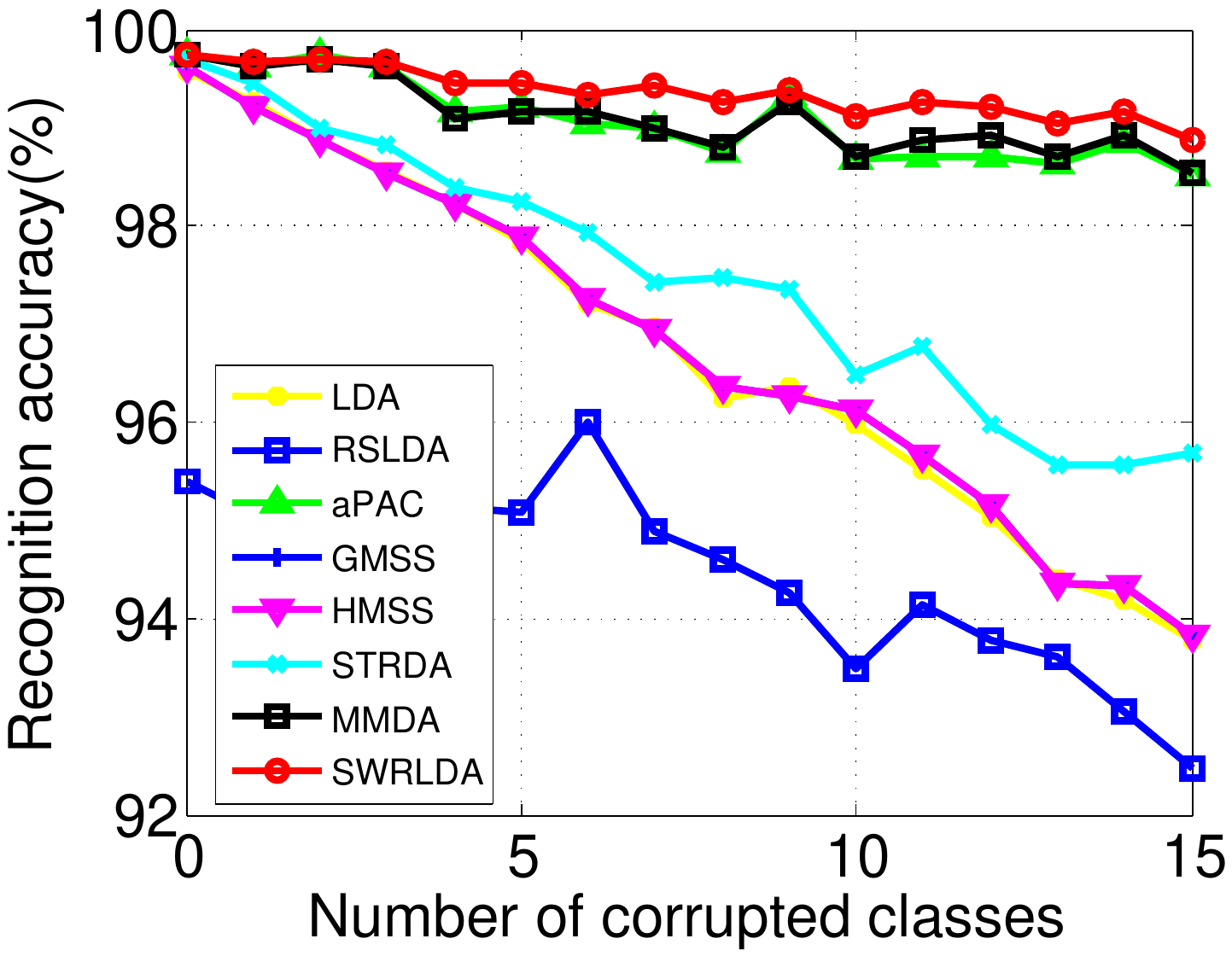}}
	\subfigure[]{\label{7f}				
		\includegraphics[width=0.23\linewidth]{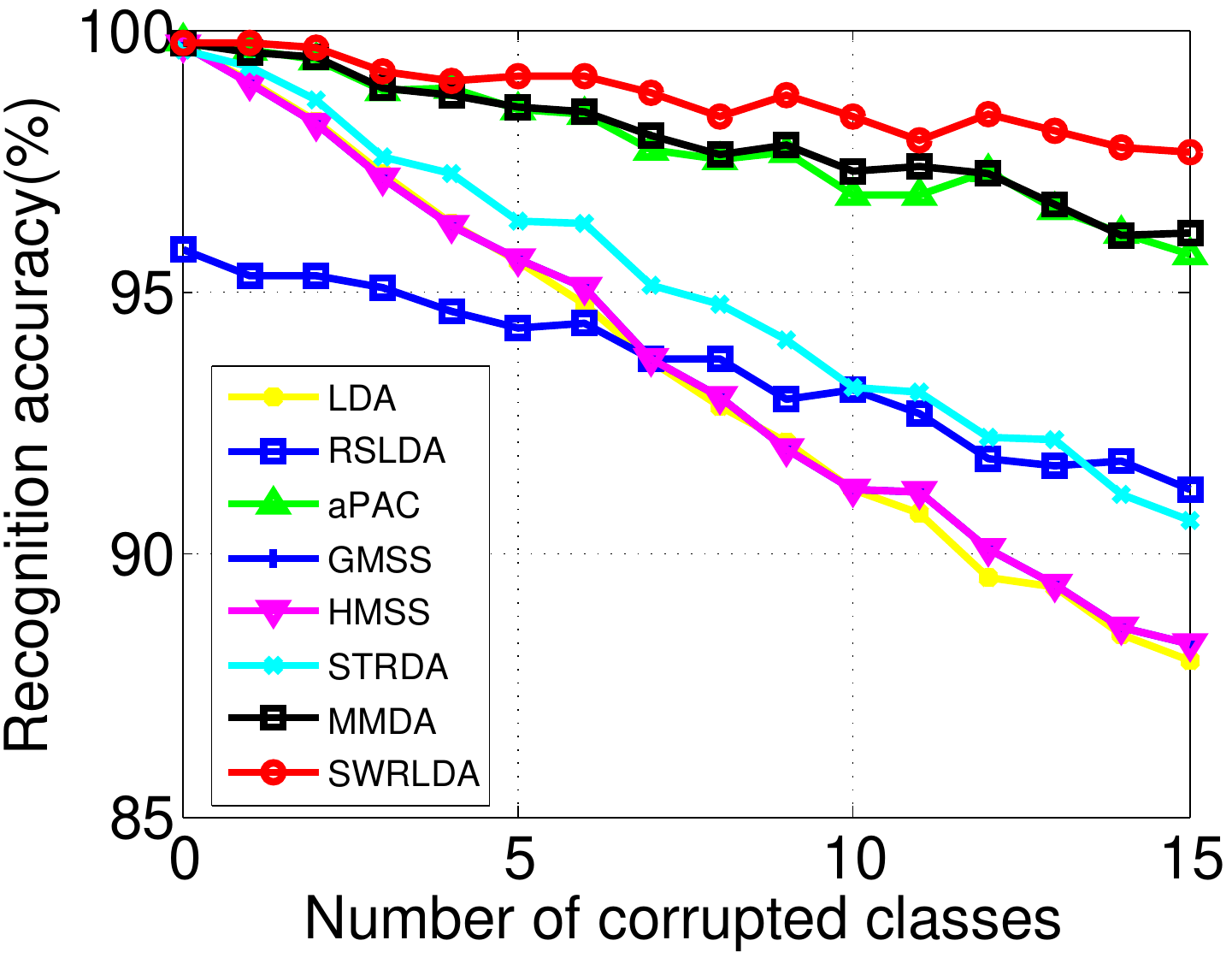}}
	\subfigure[]{\label{7g}		
		\includegraphics[width=0.23\linewidth]{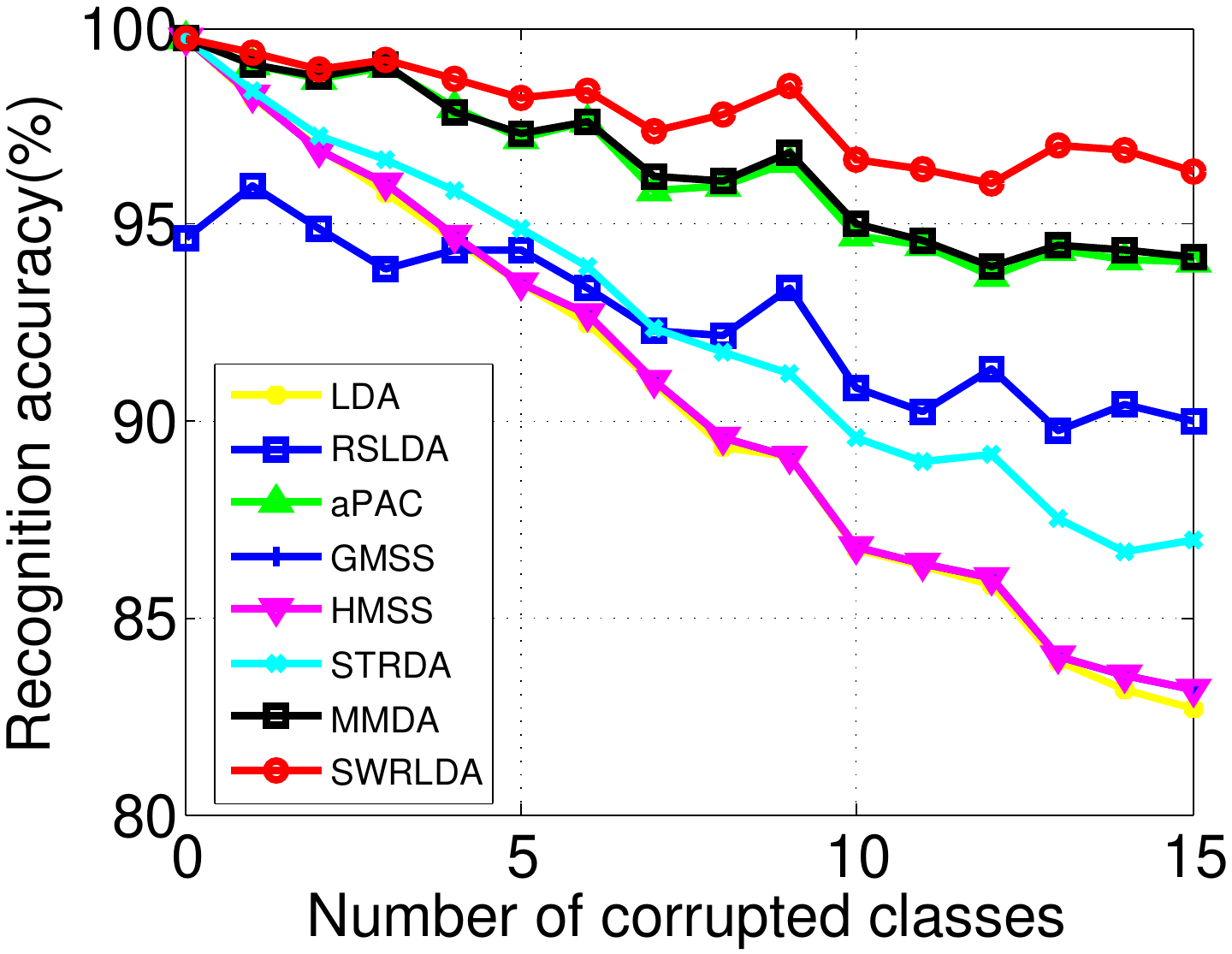}}		
	\subfigure[]{\label{7h}		
		\includegraphics[width=0.23\linewidth]{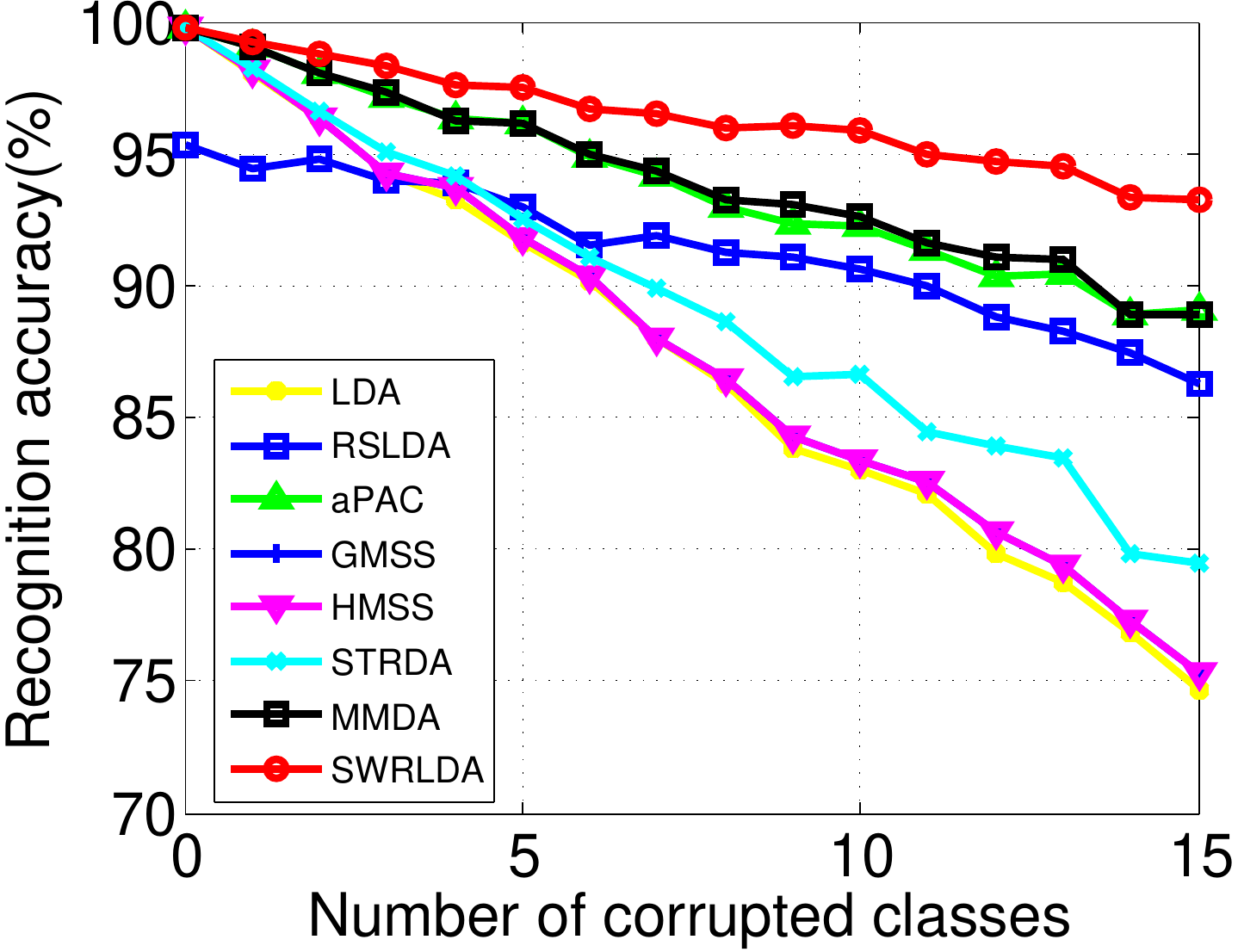}}
	\caption{Recognition accuracy with (a) 20\% (b) 40\% (c) 60\% (d) 80\% corrupted samples in each corrupted class over CMU PIE and (e) 20\% (f) 40\% (g) 60\% (h) 80\% corrupted samples in each corrupted class over YaleB.}	
	\label{PIE}	
\end{figure*}
To evaluate the robustness of the proposed SWRLDA, we further design a face recognition experiment over CMU PIE and YaleB with pixel corruption. Specifically, a certain proportion of samples are randomly selected from a class to be corrupted with ``salt and pepper'' noise and the corruption degree of each image is $30$\%, which can be regarded as an attempt to simulate an edge class to some extent. With $20$, $40$, $60$ and $80$ percentage of images corrupted in each class, the recognition accuracy with respect to varying number of corrupted classes ranging from $0$ to $15$ over CMU PIE and YaleB is demonstrated in Figure \ref{PIE}. Note that SWRLDA only achieves slightly better performance than the comparisons without corrupted class. As the number of corrupted classes increases, all the comparisons have varying degrees of performance degradation, while the superiority of SWRLDA becomes more obvious. The performance of competitors decreases dramatically while the performance of our method keeps more stable, which demonstrates the strong robustness of SWRLDA.

\subsubsection{Convergence Analysis}
In this part, we test the convergence speed of the proposed SWRLDA on all the employed datasets in Fig. \ref{converge}.
Specifically, we use $50$ times different random initializations of $\mathbf{W}$ to calculate objective function value in each iteration until converge. For each employed dataset, the mean and standard deviation of the objective function value in each iteration are illustrated in Fig. \ref{converge}.
Notably, the objective function value of our method converges in $4$ iterations on all the datasets, and thus the convergence speed is very fast. Moreover, the objective function value will converge to a stable value with different random initializations, which indicates that the solution obtained by our iterative algorithm is a global optimum on these data sets in practice.
\begin{figure*}[t] 	
	\centering 	
	\subfigure[COIL20]{\label{4a}		
		\includegraphics[width=0.3\linewidth]{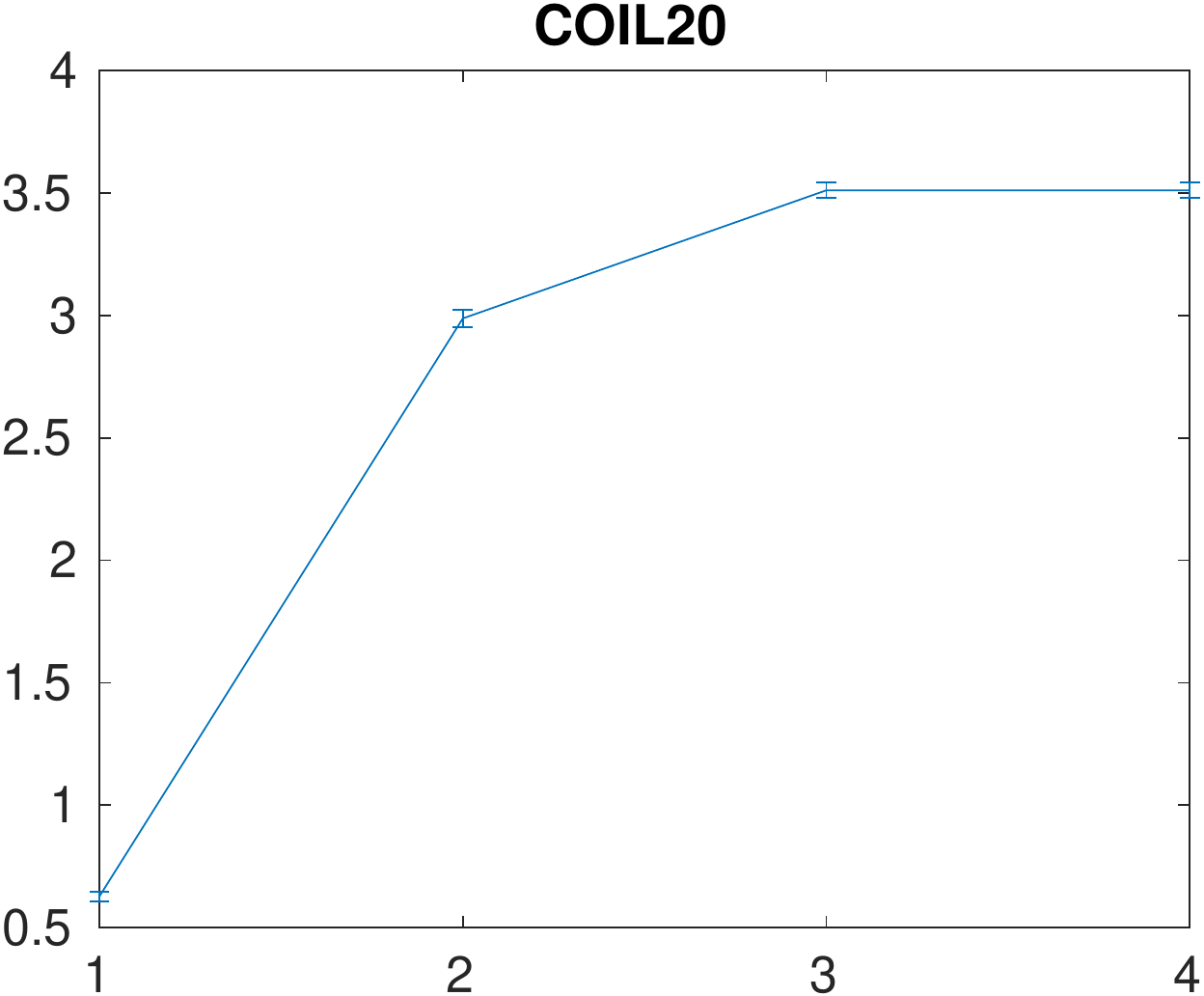}}
	\subfigure[COIL100]{\label{4b}				
		\includegraphics[width=0.3\linewidth]{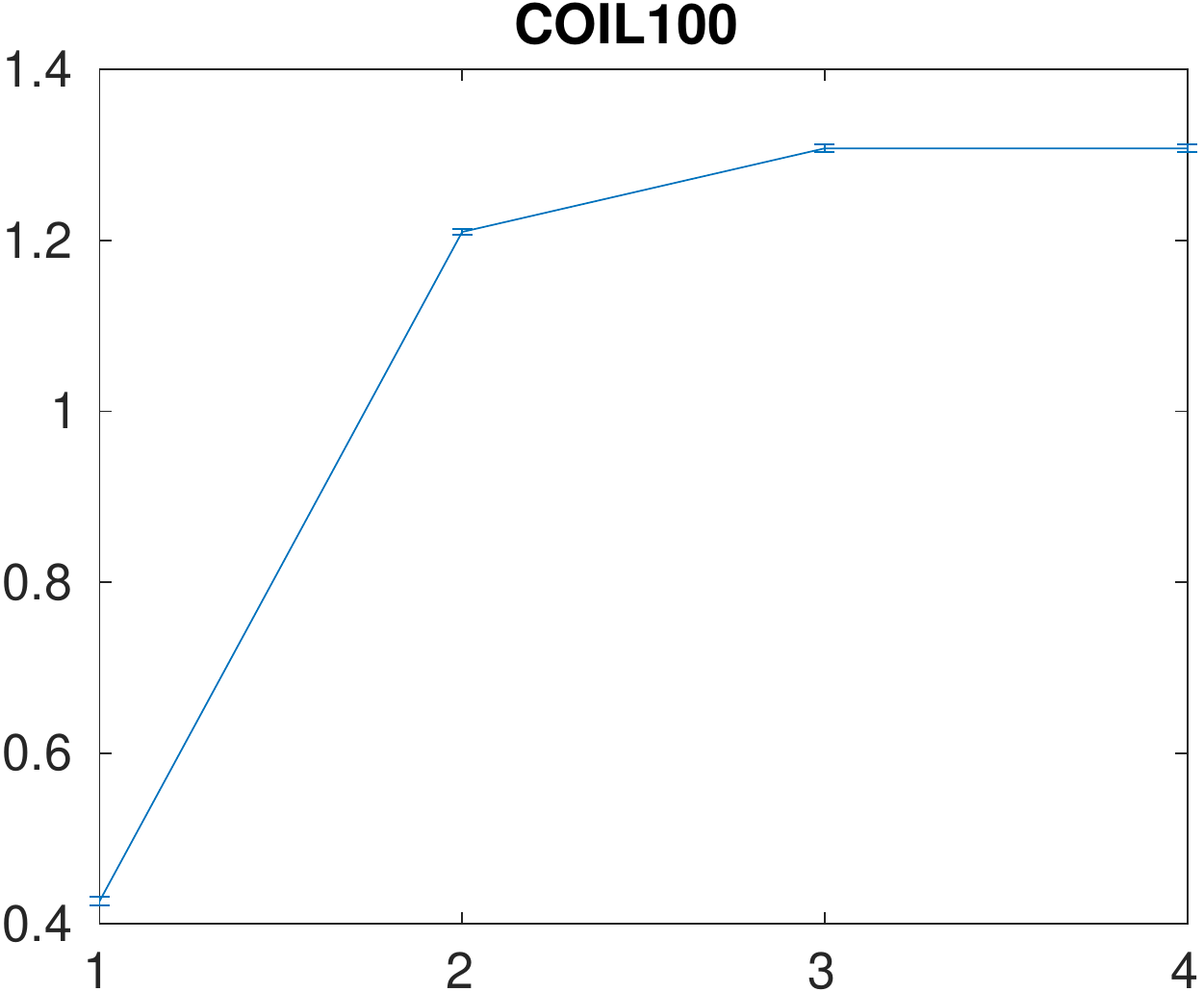}}
	\subfigure[Isolet]{\label{4c}		
		\includegraphics[width=0.3\linewidth]{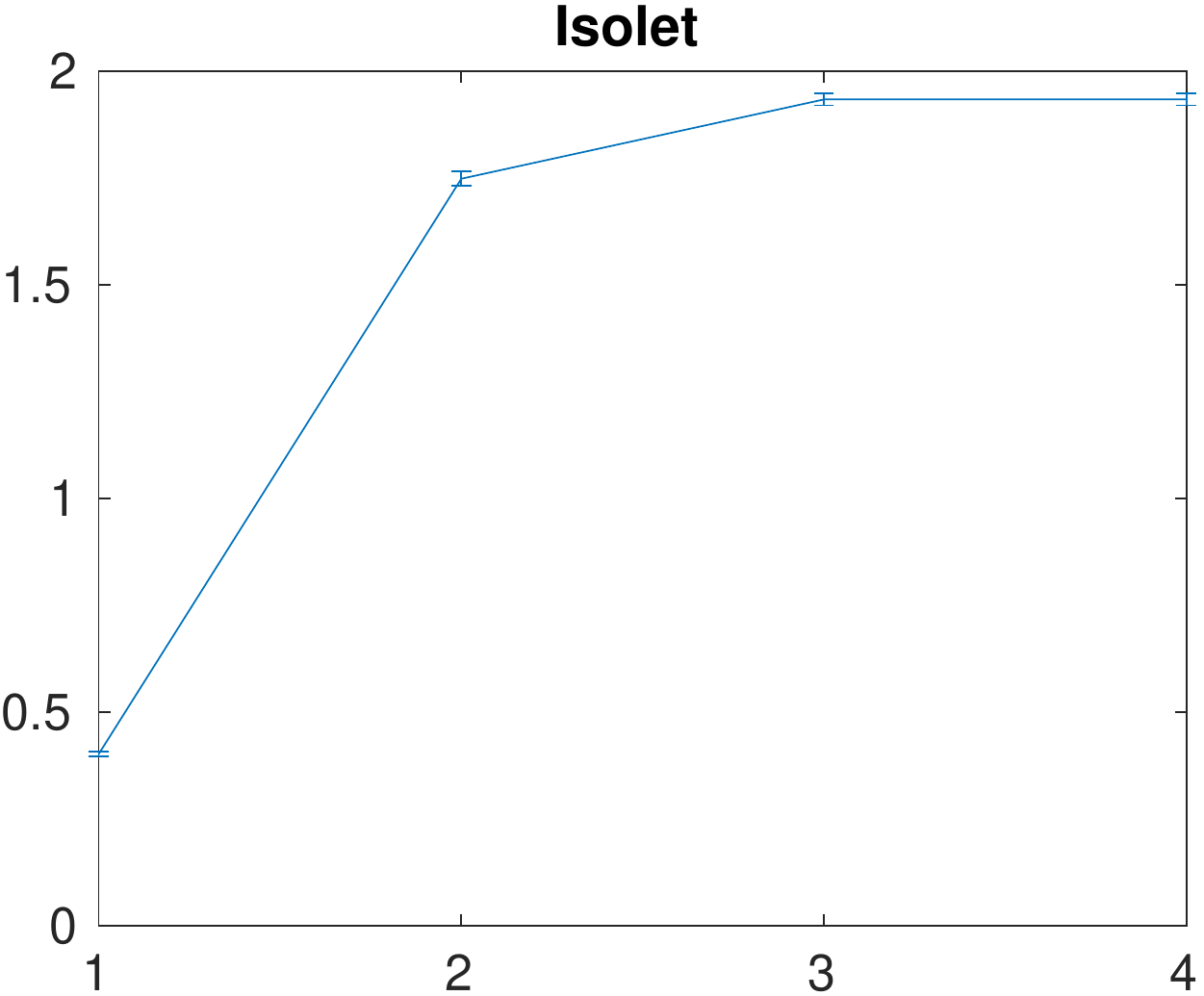}}
	\vfill
	\subfigure[20Newsgroups]{\label{4d}		
		\includegraphics[width=0.3\linewidth]{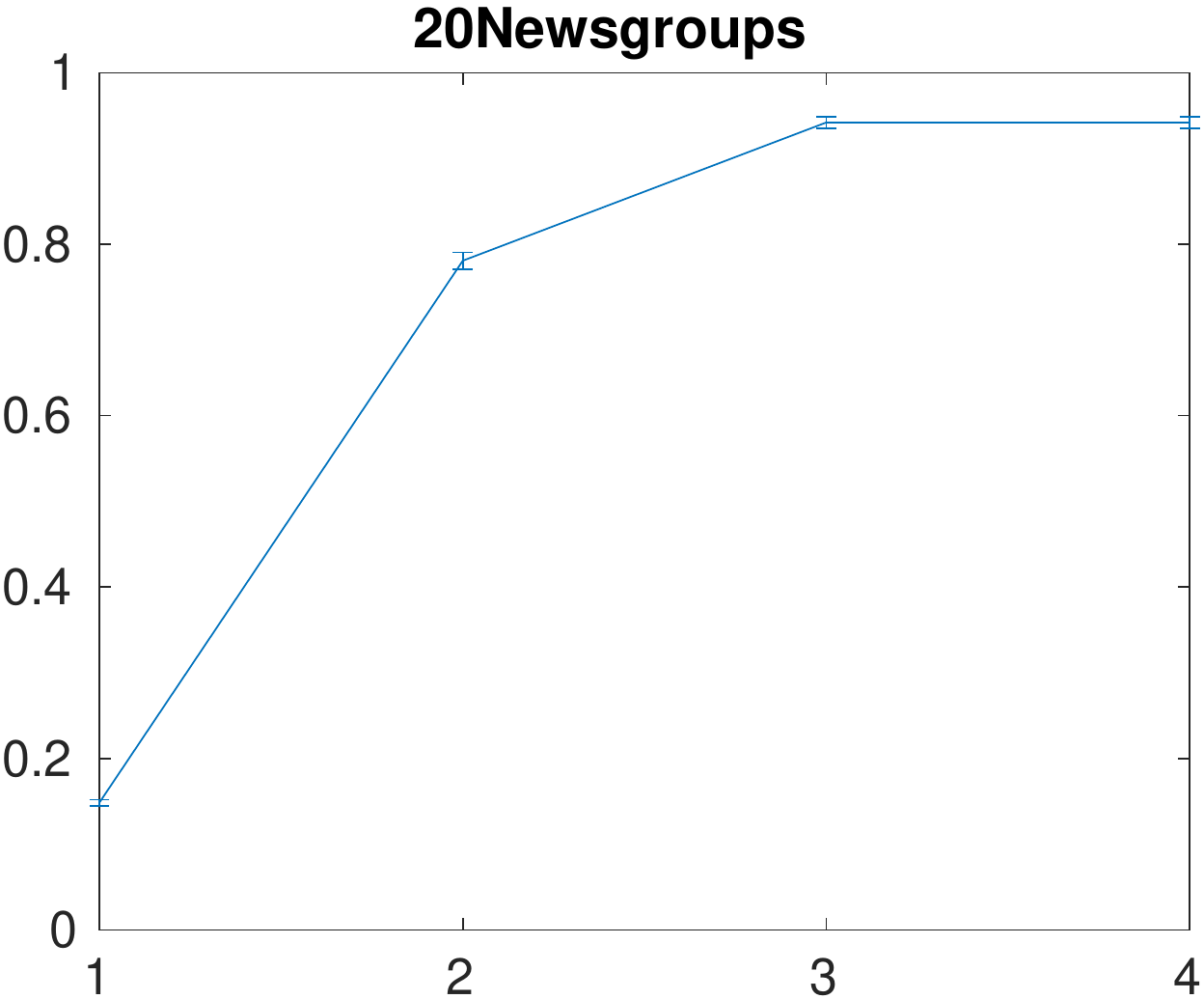}}
	\subfigure[CMU PIE]{\label{4e}				
		\includegraphics[width=0.3\linewidth]{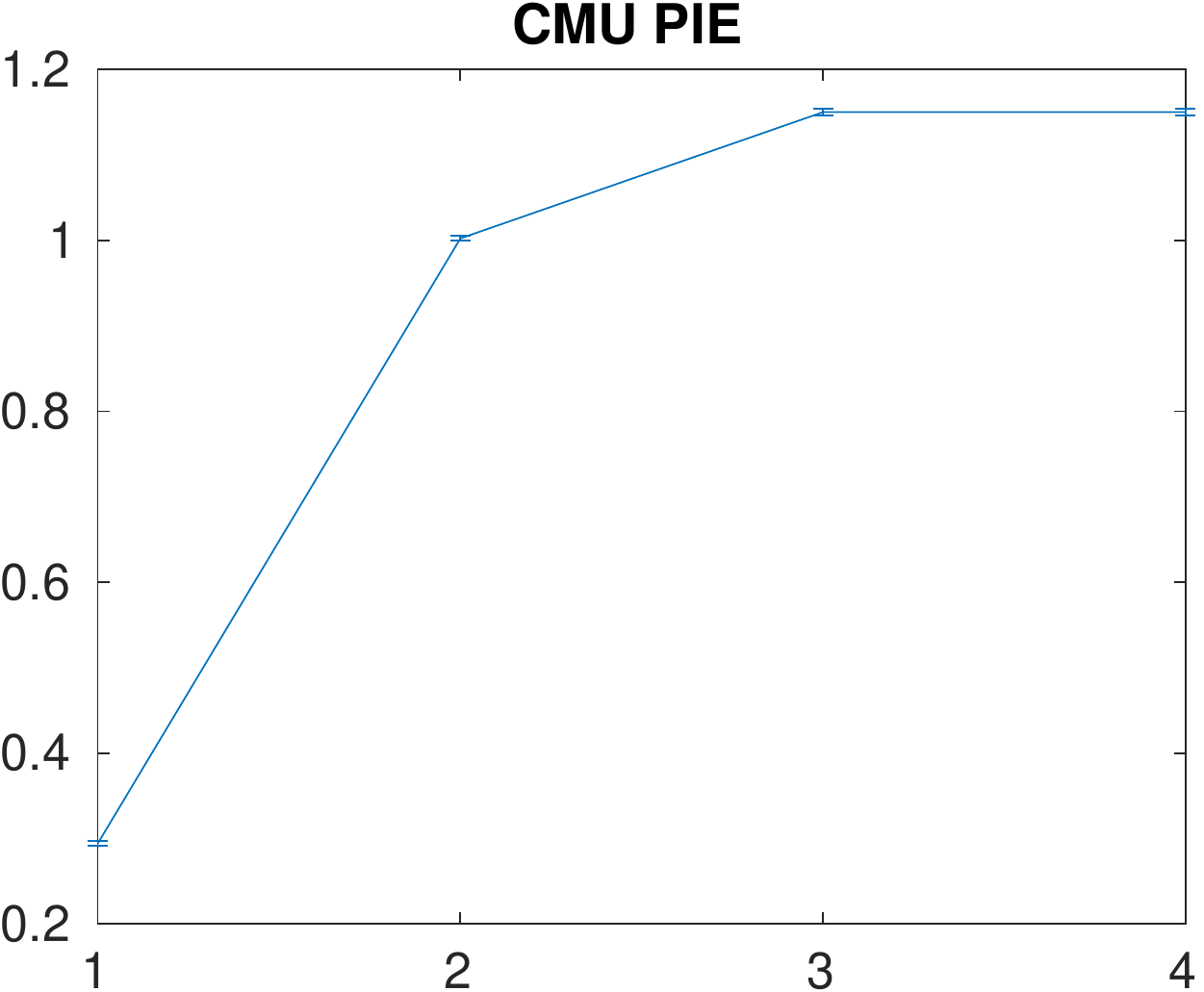}}
	\subfigure[YaleB]{\label{4f}		
		\includegraphics[width=0.3\linewidth]{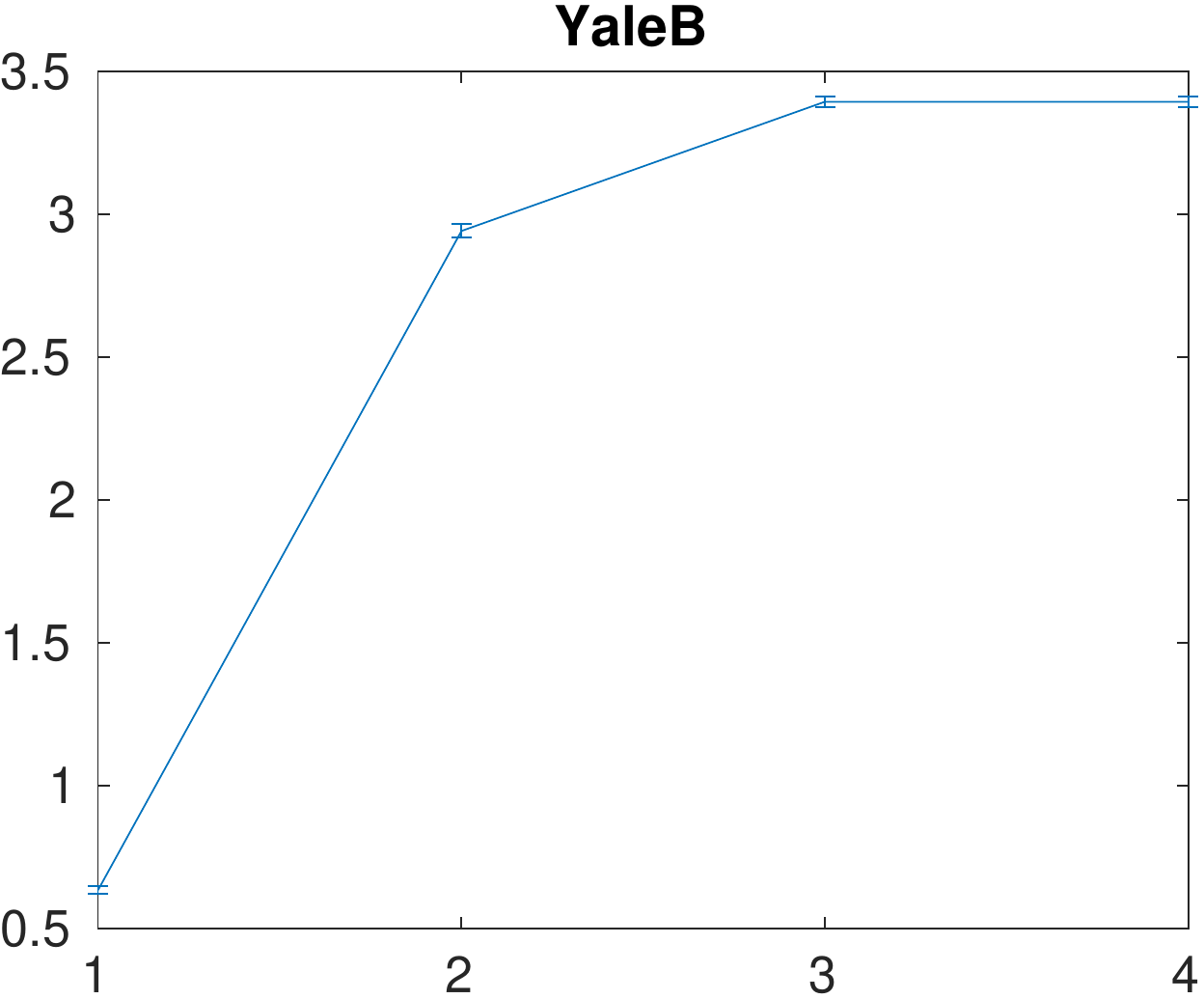}}	
	\caption{The standard deviation and mean of objective function value in the $50$ rounds of the experiments.}	
	\label{converge}	
\end{figure*}

Experimental results on both synthetic and real-world datasets demonstrate that the proposed method achieves comparable or even better performance than other methods. We can attribute this improvement to the combination of avoiding optimal mean strategy and self-weighted between-class distance criterion based on $\ell_{2,1}$-norm.
Benefiting from the efficient re-weighted optimization algorithm, SWRLDA enjoys a very fast convergence speed and is able to obtain the global optimal solution after optimization. 
Particularly, with a reduced computation complexity superior to most other methods, SWRLDA is more efficient for dealing with large-scale dimension reduction problems.

\section{Conclusion}
In this paper, we propose a novel formulation of self-weighted robust LDA for multi-class classification with edge classes, termed as SWRLDA. Considering the fact that the total mean of data is easily dominated by edge classes, the between-class scatter matrix is equivalently redefined as the difference between every two class centers to automatically avoid the calculation of optimal mean. Since the $\ell_2$-norm based distance criterion is prone to overemphasize the class pairs with large distances and neglects the small ones, a $\ell_{2,1}$-norm based between-class distance criterion is exploited in SWRLDA to further enhance its robustness to edge classes. By optimizing the objective function of SWRLDA with an effective re-weighted algorithm, self-adaptive weights are learned for each class pair to characterize its contribution to classification without tuning additional parameters. Extensive experiments on two synthetic and six real-world datasets demonstrate demonstrate the effectiveness and superiority of the proposed approach.	

\section*{Acknowledgments}
This work is supported by the program of China Scholarship Council, the National Key Research and Development Program of China (2016YFB1000903), the National Nature Science Foundation of China (No. 61872287, No. 61772411, No. 61702415 and No. 61532015), the Innovative Research Group of the National Natural Science Foundation of China (No. 61721002), the Innovation Research Team of Ministry of Education (IRT\underline{\hspace{0.5em}}17R86), the National Natural Science Foundation of China [grant no. 61922064], in part by the Zhejiang Provincial Natural Science Foundation [grant no. LR17F030001].

\bibliographystyle{unsrt}  
\bibliography{references}  

\begin{thebibliography}{10}

\bibitem{yan2018top}
Caixia Yan, Minnan Luo, Huan Liu, Zhihui Li, and Qinghua Zheng.
\newblock {Top-$k$ Multi-class SVM Using Multiple Features}.
\newblock {\em Information Sciences}, 432:479--494, 2018.

\bibitem{zhao2018multiclass}
Haifeng Zhao, Siqi Wang, and Zheng Wang.
\newblock Multiclass classification and feature selection based on least
  squares regression with large margin.
\newblock {\em Neural Computation}, 30(10):2781--2804, 2018.

\bibitem{holmes2002multiclass}
Geoffrey Holmes, Bernhard Pfahringer, Richard Kirkby, Eibe Frank, and Mark
  Hall.
\newblock {Multiclass Alternating Decision Trees}.
\newblock In {\em Proceedings of the European Conference on Machine Learning},
  2002.

\bibitem{zhang2005extending}
Liang Zhang, Lindsay~B Jack, and Asoke~K Nandi.
\newblock {Extending Genetic Programming for Multi-class Classification by
  Combining $k$-nearest Neighbor}.
\newblock In {\em Proceedings of the International Conference on Acoustics,
  Speech and Signal Processing}, 2005.

\bibitem{farid2014hybrid}
Dewan~Md Farid, Li~Zhang, Chowdhury~Mofizur Rahman, M~Alamgir Hossain, and
  Rebecca Strachan.
\newblock {Hybrid Decision Tree and Na{\"\i}ve Bayes Classifiers for
  Multi-class Classification Tasks}.
\newblock {\em Expert Systems with Applications}, 41(4):1937--1946, 2014.

\bibitem{chang2011libsvm}
Chih-Chung Chang and Chih-Jen Lin.
\newblock Libsvm: A library for support vector machines.
\newblock {\em ACM Transactions on Intelligent Iystems and Technology},
  2(3):27, 2011.

\bibitem{zhao2017unsupervised}
Xiaowei Zhao, Feiping Nie, Sen Wang, Jun Guo, Pengfei Xu, and Xiaojiang Chen.
\newblock Unsupervised 2d dimensionality reduction with adaptive structure
  learning.
\newblock {\em Neural Computation}, 29(5):1352--1374, 2017.

\bibitem{xie2017joint}
Liping Xie, Dacheng Tao, and Haikun Wei.
\newblock Joint structured sparsity regularized multiview dimension reduction
  for video-based facial expression recognition.
\newblock {\em ACM Transactions on Intelligent Systems and Technology},
  8(2):28, 2017.

\bibitem{wen2018robust}
Jie Wen, Xiaozhao Fang, Jinrong Cui, Lunke Fei, Ke~Yan, Yan Chen, and Yong Xu.
\newblock {Robust sparse linear discriminant analysis}.
\newblock {\em IEEE Transactions on Circuits and Systems for Video Technology},
  2018.
\newblock In Press.

\bibitem{Ye2007Least}
Jieping Ye.
\newblock {Least Squares Linear Discriminant Analysis}.
\newblock In {\em Proceedings of the International Conference of Machine
  Learning}, 2007.

\bibitem{xiong2017daehr}
Haoyi Xiong, Jinghe Zhang, Yu~Huang, Kevin Leach, and Laura~E Barnes.
\newblock Daehr: A discriminant analysis framework for electronic health record
  data and an application to early detection of mental health disorders.
\newblock {\em ACM Transactions on Intelligent Systems and Technology},
  8(3):47, 2017.

\bibitem{yuan2017multilinear}
Sen Yuan, Xia Mao, and Lijiang Chen.
\newblock Multilinear spatial discriminant analysis for dimensionality
  reduction.
\newblock {\em IEEE Transactions on Image Processing}, 26(6):2669--2681, 2017.

\bibitem{ChangY17}
Xiaojun Chang and Yi~Yang.
\newblock Semisupervised feature analysis by mining correlations among multiple
  tasks.
\newblock {\em {IEEE} Trans. Neural Networks Learn. Syst.}, 28(10):2294--2305,
  2017.

\bibitem{WangCLSC16}
Sen Wang, Xiaojun Chang, Xue Li, Quan~Z. Sheng, and Weitong Chen.
\newblock Multi-task support vector machines for feature selection with shared
  knowledge discovery.
\newblock {\em Signal Process.}, 120:746--753, 2016.

\bibitem{wang2018generalized}
Hao Wang, Yuanyuan Fan, Baofu Fang, and Shuanglu Dai.
\newblock {Generalized Linear Discriminant Analysis Based on Euclidean Norm for
  Gait Recognition}.
\newblock {\em Journal of Machine Learning and Cybernetics}, 9(4):569--576,
  2018.

\bibitem{zhang2012confused}
Xu-Yao Zhang and Cheng-Lin Liu.
\newblock Confused distance maximization for large category dimensionality
  reduction.
\newblock In {\em Proceedings of the International Conference on Frontiers in
  Handwriting Recognition}, 2012.

\bibitem{HanYZCL18}
Junwei Han, Le~Yang, Dingwen Zhang, Xiaojun Chang, and Xiaodan Liang.
\newblock Reinforcement cutting-agent learning for video object segmentation.
\newblock In {\em {CVPR}}, 2018.

\bibitem{ZhuHCSS17}
Lei Zhu, Zi~Huang, Xiaojun Chang, Jingkuan Song, and Heng~Tao Shen.
\newblock Exploring consistent preferences: Discrete hashing with pair-exemplar
  for scalable landmark search.
\newblock In Qiong Liu, Rainer Lienhart, Haohong Wang, Sheng{-}Wei~"Kuan{-}Ta"
  Chen, Susanne Boll, Yi{-}Ping~Phoebe Chen, Gerald Friedland, Jia Li, and
  Shuicheng Yan, editors, {\em {ACM} {MM}}, 2017.

\bibitem{ChangNMYZ15}
Xiaojun Chang, Feiping Nie, Zhigang Ma, Yi~Yang, and Xiaofang Zhou.
\newblock A convex formulation for spectral shrunk clustering.
\newblock In Blai Bonet and Sven Koenig, editors, {\em {AAAI}}, 2015.

\bibitem{yahaya2016robust}
Sharipah Soaad~Syed Yahaya, Yai-Fung Lim, Hazlina Ali, and Zurni Omar.
\newblock {Robust Linear Discriminant Analysis with Automatic Trimmed Mean}.
\newblock {\em Journal of Telecommunication, Electronic and Computer
  Engineering}, 8(10):1--3, 2016.

\bibitem{ChengCLHGZ17}
De~Cheng, Xiaojun Chang, Li~Liu, Alexander~G. Hauptmann, Yihong Gong, and
  Nanning Zheng.
\newblock Discriminative dictionary learning with ranking metric embedded for
  person re-identification.
\newblock In Carles Sierra, editor, {\em {IJCAI}}, 2017.

\bibitem{zhang2020top}
Xiaoqin Zhang, Mingyu Fan, Di~Wang, Peng Zhou, and Dacheng Tao.
\newblock Top-k feature selection framework using robust 0-1 integer
  programming.
\newblock {\em IEEE Transactions on Neural Networks and Learning Systems},
  2020.

\bibitem{martin2006robust}
R~Douglas Martin and V{\'\i}ctor~J Yohai.
\newblock {\em {Robust Statistics: Theory and Methods}}.
\newblock Wiley, 2006.

\bibitem{he2000high}
Xuming He and Wing~K Fung.
\newblock {High Breakdown Estimation for Multiple Populations with Applications
  to Discriminant Analysis}.
\newblock {\em Journal of Multivariate Analysis}, 72(2):151--162, 2000.

\bibitem{hubert2004fast}
Mia Hubert and Katrien Van~Driessen.
\newblock {Fast and Robust Discriminant Analysis}.
\newblock {\em Computational Statistics \& Data Analysis}, 45(2):301--320,
  2004.

\bibitem{chork1992integrating}
CY~Chork and Peter~J Rousseeuw.
\newblock {Integrating A High-breakdown Option into Discriminant Analysis in
  Exploration Geochemistry}.
\newblock {\em Journal of Geochemical Exploration}, 43(3):191--203, 1992.

\bibitem{zhang2010worst}
Yu~Zhang and Dit-Yan Yeung.
\newblock Worst-case linear discriminant analysis.
\newblock In {\em Proceedings of the Annual Conference on Neural Information
  Processing Systems}, 2010.

\bibitem{loog2001multiclass}
Marco Loog, RPW Duin, and Reinhold Haeb-Umbach.
\newblock {Multiclass Linear Dimension Reduction by Weighted Pairwise Fisher
  Criteria}.
\newblock {\em IEEE Transactions on Pattern Analysis \& Machine Intelligence},
  (7):762--766, 2001.

\bibitem{tao2009geometric}
Dacheng Tao, Xuelong Li, Xindong Wu, and Stephen~J Maybank.
\newblock {Geometric Mean for Subspace Selection}.
\newblock {\em IEEE Transactions on Pattern Analysis and Machine Intelligence},
  31(2):260--274, 2009.

\bibitem{bian2008harmonic}
Wei Bian and Dacheng Tao.
\newblock {Harmonic Mean for Subspace Selection}.
\newblock In {\em Proceedings of the International Conference on Pattern
  Recognition}, 2008.

\bibitem{LuoCNYHZ18}
Minnan Luo, Xiaojun Chang, Liqiang Nie, Yi~Yang, Alexander~G. Hauptmann, and
  Qinghua Zheng.
\newblock An adaptive semisupervised feature analysis for video semantic
  recognition.
\newblock {\em {IEEE} Trans. Cybern.}, 48(2):648--660, 2018.

\bibitem{LuoNCYHZ18}
Minnan Luo, Feiping Nie, Xiaojun Chang, Yi~Yang, Alexander~G. Hauptmann, and
  Qinghua Zheng.
\newblock Adaptive unsupervised feature selection with structure
  regularization.
\newblock {\em {IEEE} Trans. Neural Networks Learn. Syst.}, 29(4):944--956,
  2018.

\bibitem{ChengGCSHZ18}
De~Cheng, Yihong Gong, Xiaojun Chang, Weiwei Shi, Alexander~G. Hauptmann, and
  Nanning Zheng.
\newblock Deep feature learning via structured graph laplacian embedding for
  person re-identification.
\newblock {\em Pattern Recognit.}, 82:94--104, 2018.

\bibitem{bian2011max}
Wei Bian and Dacheng Tao.
\newblock {Max-min Distance Analysis by Using Sequential SDP Relaxation for
  Dimension Reduction}.
\newblock {\em IEEE Transactions on Pattern Analysis and Machine Intelligence},
  33(5):1037--1050, 2011.

\bibitem{zhao2018new}
Haifeng Zhao, Zheng Wang, and Feiping Nie.
\newblock {A New Formulation of Linear Discriminant Analysis for Robust
  Dimensionality Reduction}.
\newblock {\em IEEE Transactions on Knowledge and Data Engineering}, 2018.

\bibitem{LuoNCYHZ16}
Minnan Luo, Feiping Nie, Xiaojun Chang, Yi~Yang, Alexander~G. Hauptmann, and
  Qinghua Zheng.
\newblock Avoiding optimal mean robust {PCA/2DPCA} with non-greedy l1-norm
  maximization.
\newblock In Subbarao Kambhampati, editor, {\em {IJCAI}}, 2016.

\bibitem{iosifidis2013optimal}
Alexandros Iosifidis, Anastasios Tefas, and Ioannis Pitas.
\newblock {On the Optimal Class Representation in Linear Discriminant
  Analysis}.
\newblock {\em IEEE Transactions on Neural Networks and Learning Systems},
  24(9):1491--1497, 2013.

\bibitem{NieZYCLS17}
Liqiang Nie, Luming Zhang, Yan Yan, Xiaojun Chang, Maofu Liu, and Ling
  Shaoling.
\newblock Multiview physician-specific attributes fusion for health seeking.
\newblock {\em {IEEE} Trans. Cybern.}, 47(11):3680--3691, 2017.

\bibitem{ChangHSLYH18}
Xiaojun Chang, Po{-}Yao Huang, Yi{-}Dong Shen, Xiaodan Liang, Yi~Yang, and
  Alexander~G. Hauptmann.
\newblock {RCAA:} relational context-aware agents for person search.
\newblock In Vittorio Ferrari, Martial Hebert, Cristian Sminchisescu, and Yair
  Weiss, editors, {\em {ECCV}}, 2018.

\bibitem{ye2006null}
Jieping Ye and Tao Xiong.
\newblock Null space versus orthogonal linear discriminant analysis.
\newblock In {\em Proceedings of the 23rd international conference on Machine
  learning}, pages 1073--1080, 2006.

\bibitem{ye2006feature}
Jieping Ye, Ravi Janardan, Qi~Li, and Haesun Park.
\newblock Feature reduction via generalized uncorrelated linear discriminant
  analysis.
\newblock {\em IEEE Transactions on Knowledge and Data Engineering},
  18(10):1312--1322, 2006.

\bibitem{liu2018l_}
Yang Liu, Quanxue Gao, Xinbo Gao, and Ling Shao.
\newblock $l_{2,1}$-norm discriminant manifold learning.
\newblock {\em IEEE Access}, 6:40723--40734, 2018.

\bibitem{zhang2008discriminative}
Tianhao Zhang, Dacheng Tao, and Jie Yang.
\newblock Discriminative locality alignment.
\newblock In {\em European conference on computer vision}, pages 725--738.
  Springer, 2008.

\bibitem{zhou2016manifold}
Yang Zhou and Shiliang Sun.
\newblock Manifold partition discriminant analysis.
\newblock {\em IEEE transactions on cybernetics}, 47(4):830--840, 2016.

\bibitem{peng2016supervised}
Chong Peng, Jie Cheng, and Qiang Cheng.
\newblock A supervised learning model for high-dimensional and large-scale
  data.
\newblock {\em ACM Transactions on Intelligent Systems and Technology (TIST)},
  8(2):1--23, 2016.

\bibitem{peng2019discriminative}
Chong Peng and Qiang Cheng.
\newblock Discriminative ridge machine: A classifier for high-dimensional data
  or imbalanced data.
\newblock {\em arXiv preprint arXiv:1904.07496}, 2019.

\bibitem{rao1948utilization}
C~Radhakrishna Rao.
\newblock {The Utilization of Multiple Measurements in Problems of Biological
  Classification}.
\newblock {\em Journal of the Royal Statistical Society}, 10(2):159--203, 1948.

\bibitem{ChangSWLL14}
Xiaojun Chang, Haoquan Shen, Sen Wang, Jiajun Liu, and Xue Li.
\newblock Semi-supervised feature analysis for multimedia annotation by mining
  label correlation.
\newblock In Vincent~S. Tseng, Tu~Bao Ho, Zhi{-}Hua Zhou, Arbee L.~P. Chen, and
  Hung{-}Yu Kao, editors, {\em Advances in Knowledge Discovery and Data Mining
  - 18th Pacific-Asia Conference, {PAKDD} 2014, Tainan, Taiwan, May 13-16,
  2014. Proceedings, Part {II}}, volume 8444 of {\em Lecture Notes in Computer
  Science}, pages 74--85. Springer, 2014.

\bibitem{cai2008srda}
Deng Cai, Xiaofei He, and Jiawei Han.
\newblock {SRDA: An Efficient Algorithm for Large-scale Discriminant Analysis}.
\newblock {\em IEEE Transactions on Knowledge and Data Engineering},
  20(1):1--12, 2008.

\bibitem{LiNCNZY18}
Zhihui Li, Feiping Nie, Xiaojun Chang, Liqiang Nie, Huaxiang Zhang, and
  Yi~Yang.
\newblock Rank-constrained spectral clustering with flexible embedding.
\newblock {\em {IEEE} Trans. Neural Networks Learn. Syst.}, 29(12):6073--6082,
  2018.

\bibitem{ChangYYX16}
Xiaojun Chang, Yaoliang Yu, Yi~Yang, and Eric~P. Xing.
\newblock They are not equally reliable: Semantic event search using
  differentiated concept classifiers.
\newblock In {\em {CVPR}}, 2016.

\bibitem{XueNWCSY17}
Xiaowei Xue, Feiping Nie, Sen Wang, Xiaojun Chang, Bela Stantic, and Min Yao.
\newblock {AAAI}.
\newblock 2017.

\bibitem{fukunaga2013introduction}
Keinosuke Fukunaga.
\newblock {\em {Introduction to Statistical Pattern Recognition}}.
\newblock 2013.

\bibitem{xu2010dimensionality}
Bo~Xu, Kaizhu Huang, and Cheng-Lin Liu.
\newblock {Dimensionality Reduction by Minimal Distance Maximization}.
\newblock In {\em Proceedings of the International Conference on Pattern
  Recognition}, 2010.

\bibitem{yu2011distance}
Yaoliang Yu, Jiayan Jiang, and Liming Zhang.
\newblock Distance metric learning by minimal distance maximization.
\newblock {\em Pattern Recognition}, 44(3):639--649, 2011.

\bibitem{GongTCY19}
Chen Gong, Dacheng Tao, Xiaojun Chang, and Jian Yang.
\newblock Ensemble teaching for hybrid label propagation.
\newblock {\em {IEEE} Trans. Cybern.}, 49(2):388--402, 2019.

\bibitem{abou2010pareto}
Karim~T Abou-Moustafa, Fernando De~La~Torre, and Frank~P Ferrie.
\newblock {Pareto Discriminant Analysis}.
\newblock In {\em Proceedings of the IEEE Conference on Computer Vision and
  Pattern Recognition}, 2010.

\bibitem{stewart2001eigensystem}
Gilbert~W Stewart.
\newblock {\em Eigensystem}.
\newblock Society for Industrial \& Applied Mathematics, 2001.

\bibitem{li2017beyond}
Zhihui Li, Feiping Nie, Xiaojun Chang, and Yi~Yang.
\newblock Beyond trace ratio: weighted harmonic mean of trace ratios for
  multiclass discriminant analysis.
\newblock {\em IEEE Transactions on Knowledge and Data Engineering},
  29(10):2100--2110, 2017.

\bibitem{nene1996columbia}
Sameer~A Nene, Shree~K Nayar, Hiroshi Murase, et~al.
\newblock {\em {Columbia Object Image Library (Coil 20)}}.
\newblock Technical Report CUCS-006-96, Dept. of Computer Science, Columbia
  Univ., 1996.

\bibitem{nayar1996columbia}
S~Nayar, Sammeer~A Nene, and Hiroshi Murase.
\newblock {\em {Columbia Object Image Library (Coil 100)}}.
\newblock Technical Report CUCS-006-96, Dept. of Computer Science, Columbia
  Univ., 1996.

\bibitem{fanty1991spoken}
Mark Fanty and Ronald Cole.
\newblock {Spoken Letter Recognition}.
\newblock In {\em Proceedings of the Annual Conference on Neural Information
  Processing Systems}, 1991.

\bibitem{lang1995newsweeder}
Ken Lang.
\newblock Newsweeder: Learning to filter netnews.
\newblock In {\em Machine Learning Proceedings 1995}, pages 331--339. Elsevier,
  1995.

\bibitem{blake1998uci}
CL~Blake and CJ~Merz.
\newblock {UCI Repository of Machine Learning Databases. Irvine, CA: University
  of California, Department of Information and Computer Science}, 1998.

\bibitem{debole2005analysis}
Franca Debole and Fabrizio Sebastiani.
\newblock {An Analysis of The Relative Hardness of Reuters-21578 Subsets}.
\newblock {\em Journal of the American Society for Information Science and
  technology}, 56(6):584--596, 2005.

\bibitem{sim2002cmu}
Terence Sim, Simon Baker, and Maan Bsat.
\newblock {The CMU Pose, Illumination, and Expression (PIE) Database}.
\newblock In {\em Proceedings of IEEE International Conference on Automatic
  Face and Gesture Recognition}, 2002.

\bibitem{georghiades2001few}
Athinodoros~S. Georghiades, Peter~N. Belhumeur, and David~J. Kriegman.
\newblock {From Few to Many: Illumination Cone Models for Face Recognition
  under Variable Lighting and Pose}.
\newblock {\em IEEE Transactions on Pattern Analysis and Machine Intelligence},
  23(6):643--660, 2001.

\end{thebibliography}

\end{document}